\documentclass{article}

\usepackage{microtype}
\usepackage{graphicx}
\usepackage{booktabs}
\PassOptionsToPackage{table}{xcolor}

\usepackage{hyperref}

\usepackage[preprint]{icml2026}

\usepackage{amsmath,amssymb,amsfonts}
\usepackage{nicefrac}
\usepackage{tikz}
\usetikzlibrary{arrows.meta, positioning, shapes.geometric, calc, fit, decorations.pathmorphing}
\usepackage{pgfplots}
\pgfplotsset{compat=1.17}

\usepackage{tabularx}
\usepackage{float}
\usepackage{pifont} % for \ding{55} if you keep it
\usepackage{multirow}
\usepackage{algorithm}
\usepackage{algorithmic}
\newcommand{\RETURN}[1]{\STATE \textbf{return} #1}

\usepackage{amsthm}
\theoremstyle{plain}
\newtheorem{theorem}{Theorem}[section]
\newtheorem{proposition}[theorem]{Proposition}

\theoremstyle{definition}

\newcommand{\R}{\mathbb{R}}
\newcommand{\half}{\tfrac{1}{2}}
\newcommand{\T}{\top}

\newcommand{\Ham}{H}
\newcommand{\Jmat}{J}
\newcommand{\Rmat}{R}
\newcommand{\Mmat}{M}
\newcommand{\Dmat}{D}
\newcommand{\grad}{\nabla}

\newcommand{\dt}{\Delta t}
\newcommand{\bbar}{\bar{\beta}}
\newcommand{\Vbar}{\bar{V}}
\newcommand{\Vtilde}{\tilde{V}}

\newcommand{\cmark}{\ding{51}}
\newcommand{\xmark}{\ding{55}}

\newcommand{\best}[1]{{\boldmath$#1$}}
\usepackage{pifont}

\def\isarxiv{1}

\icmltitlerunning{PHAST: Port-Hamiltonian Architecture for Structured Temporal Dynamics}

\begin{document}

% Force single-column layout for arXiv (no page limit)
\onecolumn

\icmltitle{PHAST: Port-Hamiltonian Architecture for Structured Temporal Dynamics Forecasting}

% Author information (visible in preprint)
\begin{icmlauthorlist}
\icmlauthor{Shubham Bhardwaj}{utexas,mihawk}
\icmlauthor{Chandrajit Bajaj}{utexas}
\end{icmlauthorlist}

\icmlaffiliation{utexas}{Department of Computer Science, Oden Institute for Computational Engineering and Sciences, The University of Texas at Austin, Austin, TX, USA}
\icmlaffiliation{mihawk}{Mihawk.ai}

\icmlcorrespondingauthor{Shubham Bhardwaj}{shubham.bhardwaj@utexas.edu}

% Keywords for PDF metadata (not shown in document)
\icmlkeywords{Port-Hamiltonian systems, physics-informed machine learning, dissipative dynamics, structure-preserving neural networks}

\vskip 0.3in

% Print affiliations and copyright notice
\printAffiliationsAndNotice{}

% ---------------------------------------------------------------------------
% Abstract
% ---------------------------------------------------------------------------
\begin{abstract}
Real physical systems are dissipative---a pendulum slows, a circuit loses charge to heat---and
forecasting their dynamics from partial observations is a central challenge in scientific machine learning.
We address the \emph{position-only} (q-only) problem: given only generalized positions~$q_t$ at
discrete times (momenta~$p_t$ latent), learn a structured model that (a)~produces stable long-horizon
forecasts and (b)~recovers physically meaningful parameters when sufficient structure is provided.
The port-Hamiltonian framework makes the conservative--dissipative split explicit via
$\dot{x}=(J-R)\nabla H(x)$, guaranteeing $dH/dt\le 0$ when $R\succeq 0$.
We introduce \textbf{PHAST} (Port-Hamiltonian Architecture for Structured Temporal dynamics),
which decomposes the Hamiltonian into potential~$V(q)$, mass~$M(q)$, and damping~$D(q)$ across
three knowledge regimes (KNOWN, PARTIAL, UNKNOWN), uses efficient low-rank PSD/SPD
parameterizations, and advances dynamics with Strang splitting.
Across thirteen q-only benchmarks spanning mechanical, electrical, molecular, thermal, gravitational, and ecological systems, PHAST achieves the best long-horizon forecasting
among competitive baselines and enables physically meaningful parameter recovery when
the regime provides sufficient anchors.
We show that identification is fundamentally ill-posed without such anchors (gauge freedom),
motivating a two-axis evaluation that separates forecasting stability from identifiability.
\end{abstract}

% ---------------------------------------------------------------------------
% 1. Introduction
% ---------------------------------------------------------------------------
\section{Introduction}
\label{sec:intro}

Forecasting the future states of dynamical systems from observational data is a
central problem in scientific machine learning.
Physical systems span a wide spectrum of complexity:
\emph{simple conservative} systems (e.g., an undamped pendulum or a frictionless
spring) preserve total energy and trace closed orbits in phase space;
\emph{dissipative} systems (e.g., a pendulum slowing due to air drag, an RLC
circuit losing charge to resistive heating) lose energy over time and converge
toward attractors;
and \emph{chaotic} systems (e.g., a double pendulum, turbulent fluid flows)
exhibit sensitive dependence on initial conditions, making long-horizon
prediction fundamentally difficult.
A practical forecasting framework must handle all three regimes---and, crucially,
most real-world systems are \emph{dissipative}: a swinging pendulum eventually
stops, a robot arm loses energy through joint friction, and electrical circuits
dissipate energy as heat.

\paragraph{Problem setting.}
We formalize the forecasting task in the \emph{position-only} (q-only) setting:
we observe only generalized positions~$q_t$ at discrete times; momenta~$p_t$ are
latent and never measured.
The goal is to learn a structured dynamical model from q-only data that
(a)~produces stable long-horizon open-loop forecasts and
(b)~when partial physics is available, recovers physically interpretable parameters
(potential, mass, damping).

\paragraph{Port-Hamiltonian framework.}
The \textbf{port-Hamiltonian} framework~\citep{van2014port} provides a principled
decomposition for dissipative systems:
\begin{equation}
\dot{x} = (J - R)\nabla H(x) + Gu,
\label{eq:ph_dynamics}
\end{equation}
where $x = (q,p)$ denotes generalized coordinates and momenta.
The term $Gu$ represents external actuation: $u$ is a generalized force
(e.g., torque at a joint), and the \emph{port matrix} $G$ selects which
degrees of freedom are actuated.
The conjugate \emph{port output} $y^{\text{port}} = G^\top \nabla H(x)$ is
velocity-like for mechanical systems, and the product
$u^\top y^{\text{port}}$ is the instantaneous power supplied to the system---the
rate at which external work flows in through the port.
Our forecasting benchmarks have no external actuation, so we set
$u \equiv 0$ throughout Sections~1--4; the forced form and its port
structure become central in the Energy--Casimir control study
(Sec.~\ref{sec:casimir_forced_mode}).
In port-Hamiltonian form, dynamics decompose into energy geometry
(via the mass tensor $\Mmat(q)$), environmental structure
(via the potential $V(q)$), and dissipation ($\Dmat(q)$), enabling principled
combinations of known physics and learned components (Sec.~\ref{sec:intro:regimes}).
For mechanical systems, the Hamiltonian takes the form
$H(q,p)=V(q)+T(q,p)$ with kinetic energy $T(q,p)=\tfrac12 p^\top M(q)^{-1}p$,
where $\Mmat(q)\succ 0$ is the generalized mass (inertia) tensor---it induces
a Riemannian metric on configuration space under which free trajectories are
geodesics---and we define the generalized
velocity $v := \nabla_p H(q,p)=\Mmat(q)^{-1}p$.
In our experiments we use a separable constant-mass approximation
$\Mmat(q)\approx\Mmat$ for efficiency (Sec.~\ref{sec:intro:regimes}),
but the architecture supports the general configuration-dependent case
(Appendix~\ref{app:ablations:nonseparable_mass}).
The operator $J=-J^\top$ encodes conservative energy exchange via the canonical
symplectic structure.
Dissipation is introduced through a structured operator
\[
R(q) =
\begin{pmatrix}
0 & 0 \\
0 & D(q)
\end{pmatrix},
\qquad D(q)=D(q)^\top\succeq 0,
\]
so that damping acts only on the momentum dynamics. This corresponds to standard
mechanical dissipation, where friction affects generalized velocities but not
positions.

With the canonical symplectic structure and the block-diagonal dissipation,
the dynamics decompose explicitly as
\begin{align}
\dot q &= \nabla_p H(q,p), \label{eq:qdot}\\
\dot p &= -\nabla_q H(q,p) - D(q)\,\nabla_p H(q,p). \label{eq:pdot}
\end{align}

This structure yields an explicit energy balance:
\begin{equation}
\dot H(x) = \nabla H(x)^\top (J - R)\nabla H(x)
          = - \nabla H(x)^\top R \nabla H(x) \le 0,
\end{equation}
since $\nabla H^\top J \nabla H = 0$ by skew-symmetry. As a result, the system is passive by construction, guaranteeing non-increasing energy regardless of the specific parameterization of $H(x)$.

\paragraph{Oscillatory dynamics.}
In the conservative case ($\Dmat=0$), the dynamics reduce to $\dot x=\Jmat\nabla\Ham(x)$; linearizing around an \emph{elliptic} equilibrium (e.g., a local minimum of the potential with $\Mmat\succ 0$) yields purely imaginary eigenvalues (phase-space rotations), corresponding to oscillatory exchange between kinetic and potential energy.
Adding $\Dmat\succeq 0$ shifts eigenvalues into the closed left half-plane (e.g., $-\sigma \pm i\omega$), yielding damped oscillations; Appendix~\ref{app:math:spectrum} provides a concise linear-algebra statement and proof for the linearized model $\dot x=(\Jmat-\Rmat)Qx$ \citep{mehl2018linear}.
This work develops neural architectures that respect this structure (Appendix~\ref{app:math:strang}).

\paragraph{Motivation.}
We are motivated by three converging observations:
\textbf{(i)} Physics-informed networks like HNNs~\citep{greydanus2019hamiltonian} and LNNs~\citep{cranmer2020lagrangian} show that structure improves generalization, but they assume purely conservative systems ($\Rmat = 0$).
\textbf{(ii)} State-space models like S5~\citep{smith2023simplified}, LinOSS~\citep{linoss2024}, and D-LinOSS~\citep{dlinoss2025} achieve efficiency but typically do not provide \emph{explicit} physical guarantees---e.g., they do not expose an explicit port-Hamiltonian decomposition or a passivity certificate.
\textbf{(iii)} Structure-preserving networks like Volume-Preserving Transformers~\citep{brantner2024volume} enforce geometric constraints such as volume preservation, i.e., their one-step map $\phi_{\dt}: x_t \mapsto x_{t+1}$ has unit Jacobian determinant $\det\!\left(\partial \phi_{\dt}/\partial x_t\right) = 1$, which is naturally aligned with conservative, divergence-free dynamics.
However, they do not provide an explicit dissipation mechanism or a certificate of passivity for \emph{lossy} systems, and strictly volume-preserving maps cannot represent generic dissipative flows that contract phase-space volume.

We propose \textbf{PHAST} (Port-Hamiltonian with Strang Splitting), which properly models \emph{both} the conservative ($\Jmat$) and dissipative ($\Rmat$) components.
The key insight is that the port-Hamiltonian structure provides passivity \emph{regardless} of how the components are parameterized: whether the potential $V(q)$, mass $\Mmat(q)$, and damping $\Dmat(q)$ are structured functions or neural networks, the continuous-time dynamics satisfy $d\Ham/dt \leq 0$ by construction.
PHAST addresses the q-only problem via a three-stage pipeline: a causal velocity
observer reconstructs $\hat{\dot{q}}$ from position history, a canonicalizer maps
to phase state $(q,\hat{p})$, and the port-Hamiltonian core integrates forward
with Strang splitting.
Given a short burn-in window $q_{0:K-1}$, PHAST rolls out open-loop;
Section~\ref{sec:methods} details the components.

% ---------------------------------------------------------------------------
% Knowledge Regimes (promoted from Methods for arXiv)
% ---------------------------------------------------------------------------
\paragraph{Knowledge regimes and the Hamiltonian template.}
\label{sec:intro:regimes}
A central challenge in learning dynamics from data is that trajectories alone do
not uniquely identify physical parameters: multiple combinations of mass,
potential, and damping can generate identical motions (gauge freedom).
PHAST addresses this by allowing problem-specific structure to be imposed on the
Hamiltonian components, breaking gauge freedom when partial physics knowledge is
available.

We formalize this flexibility through three knowledge regimes:
\textsc{Known}, \textsc{Partial}, and \textsc{Unknown}, which differ only in
which components of the Hamiltonian are specified versus learned.
Anchoring either the potential $V(q)$ or the mass matrix $M(q)$ is sufficient to
render the remaining parameters identifiable, while leaving both free leads to
an underdetermined problem.

Table~\ref{tab:structured_hamiltonians} illustrates representative systems
supported by PHAST, spanning classical mechanics, robotics, molecular dynamics,
electrical circuits, thermodynamics, and ecological modeling.
Our experiments span all six domains in Table~\ref{tab:structured_hamiltonians}:
the port-Hamiltonian formulation is not restricted to mechanics---any system whose dynamics admit a storage function and a conservative--dissipative
split can be cast in this form.

\begin{table}[t]
\centering
\scriptsize
\setlength{\tabcolsep}{3pt}
\renewcommand{\arraystretch}{1.2}
\begin{tabularx}{\textwidth}{@{} l l l l X X X c @{}}
\toprule
\textbf{Regime}
& \textbf{Domain}
& \textbf{System}
& \textbf{Forecast $q$}
& \textbf{$V(q)$: landscape}
& \textbf{$M(q)$: geometry}
& \textbf{$D(q)$: dissipation}
& \textbf{Ident.?} \\
\midrule

\multirow{5}{*}{\textsc{Known}}
& Mechanics
& Pendulum
& Angle $\theta(t)$
& $-g\cos\theta$
& Scalar $m$
& Air drag $d(\theta)$
& Yes \\

& Mechanics
& Spring--mass
& Displacement $x(t)$
& $\tfrac12 kx^2$
& Scalar $m$
& Viscous $b\dot x$
& Yes \\

& Molecular
& Lennard--Jones
& Positions $r_i(t)$
& $4\epsilon[(\sigma/r)^{12}{-}(\sigma/r)^6]$
& Diagonal $mI$
& Solvent friction $\gamma I$
& Yes \\

& Electrical
& RLC circuit
& Charge $q_C(t)$
& $q^2/2C$
& Inductance $L$
& Resistance $R$
& Yes \\

& Thermal
& Heat exchange
& Temperature $T(t)$
& $\tfrac12 c\,T^2$
& Time constant $\tau$
& Heat loss $\kappa$
& Yes \\

\midrule

\multirow{4}{*}{\textsc{Partial}}
& Mechanics
& Pendulum (unk.\ $m$)
& Angle $\theta(t)$
& $-g\cos\theta$ (given)
& $mI$ (learn $m$)
& Learned (bounded)
& Yes \\

& Robotics
& Robot arm
& Joint angles $\theta_i(t)$
& Gravity regressor
& CAD inertia template
& Joint friction (learned)
& Yes \\

& Astrophysics
& N-body gravity
& Positions $r_i(t)$
& $-\!\sum_{i<j}\! Gm_im_j/r_{ij}$
& $m_i I$ (learn $m_i$)
& Drag (learned)
& Yes$^\dagger$ \\

& General
& Any partially known
& System-dep.\ $q(t)$
& Template $+\Delta V$
& $m_0 I{+}UU^\top$
& Learned (bounded)
& Partial \\

\midrule

\multirow{2}{*}{\textsc{Unknown}}
& Any
& Black-box dynamics
& System-dep.\ $q(t)$
& Neural network
& Neural SPD net
& Neural PSD net
& No \\

& Ecology
& Predator--prey
& Densities $n_i(t)$
& Neural network
& Neural SPD net
& Neural PSD net
& No \\

\bottomrule
\end{tabularx}
\caption{
\textbf{Hamiltonian decomposition across knowledge regimes.}
Each system is described through the three components of the port-Hamiltonian
template:
the \emph{energy landscape}~$V(q)$ (environmental forces and equilibria),
the \emph{geometry}~$M(q)$ (Riemannian metric governing the cost of motion),
and the \emph{dissipation}~$D(q)$ (energy loss structure).
The regime determines which components are specified versus learned.
Anchoring either $V(q)$ or $M(q)$ is sufficient to break gauge freedom and
render the remaining parameters identifiable;
leaving both unconstrained (\textsc{Unknown}) yields an underdetermined inverse
problem even when forecasting is accurate.
Our experiments cover all six domains; the port-Hamiltonian
formulation applies to any system admitting a storage function and a
conservative--dissipative split.
$^\dagger$Identifiable up to an overall mass scale; fixing the gravitational
constant~$G$ resolves the remaining one-parameter gauge.
}
\label{tab:structured_hamiltonians}
\end{table}

\paragraph{Interpretive structure of the Hamiltonian.}
Appendix~\ref{app:table1_mappings} provides the complete port-Hamiltonian mapping
(state, potential, mass, damping, and identifiability analysis) for every system
in Table~\ref{tab:structured_hamiltonians}, including derivations and schematic
diagrams.
At a high level, the table illustrates that diverse systems---mechanical,
electrical, thermal, and ecological---admit a common decomposition into three
components with distinct physical roles:
\begin{enumerate}
\item \textbf{Potential $V(q)$}: encodes the \emph{environmental energy landscape}---external
  forces, equilibrium structure, and constraints (e.g., gravity for a pendulum,
  capacitive energy for an RLC circuit, entropy potential for a thermal system).
\item \textbf{Mass $M(q)$}: defines the \emph{intrinsic geometry of motion}.
  The kinetic energy $T=\tfrac12\dot q^\top M(q)\dot q$ induces a Riemannian
  metric on configuration space~$\mathcal{Q}$; free (unforced, undamped)
  trajectories are geodesics under this metric.
  $M(q)$ determines how effort translates to motion and encodes inertial
  coupling between degrees of freedom.
\item \textbf{Damping $D(q)$}: encodes \emph{dissipative structure}---how and
  where energy is lost, governing contraction toward attractors.
\end{enumerate}
The three knowledge regimes differ only in which of these components are
specified versus learned; the compositional template---and its passivity
guarantee---remains unchanged (Fig.~\ref{fig:phast_regime_spectrum}).

\paragraph{Concrete instantiations.}
To make the mapping from physical system to $(V,\Mmat,\Dmat)$ template explicit,
we walk through three representative entries from
Table~\ref{tab:structured_hamiltonians}
(Appendix~\ref{app:table1_mappings} provides the full derivation for every row):
\begin{itemize}
\item \textbf{Simple pendulum} (KNOWN).
  State $q=\theta$ (angle), momentum $p=m\ell^2\dot\theta$ (angular momentum).
  The potential $V(\theta)=-mg\ell\cos\theta$ encodes gravity;
  mass $\Mmat=m\ell^2$ is a known scalar inertia;
  the only unknown is position-dependent air drag $\Dmat(\theta)$, which PHAST
  learns.
  Because both $V$ and $\Mmat$ are given, the inverse problem is fully
  identifiable.

\item \textbf{Cart-pole} (PARTIAL).
  State $q=(\text{cart position},\,\theta)$, a two-DOF system.
  Gravity gives $V(q)=mg\ell\cos\theta$ (template provided).
  The mass tensor $\Mmat(q)$ is a $2{\times}2$ configuration-dependent
  inertia coupling cart and pole (given from the Lagrangian or CAD model).
  Dissipation $\Dmat(q)$ has separate friction channels for the cart rail and
  pivot joint---learned with bounded strength.

\item \textbf{RLC circuit} (KNOWN, non-mechanical).
  State $q=q_C$ (charge on capacitor), momentum $p=Li$ (flux linkage, where
  $i$ is current and $L$ is inductance).
  Capacitive energy gives $V(q)=q^2/2C$; inductance plays the role of mass
  ($\Mmat=L$); resistance dissipates energy as heat ($\Dmat=R$).
  The same $(V,\Mmat,\Dmat)$ template applies, illustrating that the
  port-Hamiltonian decomposition extends beyond mechanics.
\end{itemize}

\begin{figure}[t]
  \centering
  \begin{tikzpicture}[
      box/.style={rectangle, draw=black, thick, rounded corners=3pt, minimum width=3.6cm, minimum height=2.1cm, align=left, inner sep=5pt, font=\small},
      title/.style={font=\small\bfseries},
      shared/.style={
        rectangle, draw=black!60, thick, rounded corners=3pt,
        fill=black!5, minimum width=12.8cm, minimum height=1.2cm,
        align=left, inner sep=6pt, font=\scriptsize
      },
      arrow/.style={-{Stealth[length=2mm]}, thick, black!50},
      downarrow/.style={-{Stealth[length=2mm]}, semithick, black!50},
  ]
    % Regime boxes with color coding: gray=given, blue=learned, red=dissipation
    \node[box, fill=blue!8] (unknown) at (0,0) {
      \textbf{UNKNOWN}\\[0.2em]
      $V(q)$: \textcolor{blue!70!black}{learned}\\
      $M(q)$: \textcolor{blue!70!black}{learned} (SPD)\\
      $D(q)$: \textcolor{red!70!black}{learned} (PSD)
    };
    \node[box, fill=white] (partial) at (4.8,0) {
      \textbf{PARTIAL}\\[0.2em]
      $V(q){=}\bar V{+}\textcolor{blue!70!black}{\varepsilon\tilde V}$\\
      $M(q)$: \textcolor{black!50}{given}\\
      $D(q)$: \textcolor{red!70!black}{learned} (bounded)
    };
    \node[box, fill=black!5] (known) at (9.6,0) {
      \textbf{KNOWN}\\[0.2em]
      $V(q)$: \textcolor{black!50}{given}\\
      $M(q)$: \textcolor{black!50}{given}\\
      $D(q)$: \textcolor{red!70!black}{learned} (PSD)
    };

    % Arrows between regimes
    \draw[arrow] (unknown.east) -- (partial.west);
    \draw[arrow] (partial.east) -- (known.west);

    % Down arrows to shared template
    \draw[downarrow] (unknown.south) -- ++(0,-0.5);
    \draw[downarrow] (partial.south) -- ++(0,-0.5);
    \draw[downarrow] (known.south) -- ++(0,-0.5);

    % Shared template ribbon
    \node[shared, below=1.0cm of partial] (template) {
      \textbf{Shared template (all regimes):}\\
      \textbullet\; Port-Hamiltonian state $x=(q,p)$ with conservative--dissipative split\\
      \textbullet\; Dissipation acts on momentum only ($D(q)\succeq 0$); mass is SPD ($M(q)\succ 0$)\\
      \textbullet\; Structure-preserving time stepping (Strang-split symplectic core)
    };

    % Label
    \node[font=\scriptsize, text=black!60, above=0.15cm of partial]
      {$\longleftarrow$\, less prior knowledge \hspace{2em} more prior knowledge\, $\longrightarrow$};
  \end{tikzpicture}
  \caption{\textbf{PHAST unifies three knowledge regimes under one port-Hamiltonian template.}
  Regimes differ only in what is \textcolor{black!50}{given} vs.\ \textcolor{blue!70!black}{learned} for $(V,M(q))$;
  \textcolor{red!70!black}{dissipation $D(q)$} is always learned.
  All regimes share the same continuous-time structure $\dot x=(J{-}R)\nabla H$ and a structure-preserving
  discrete-time transition map $\Phi_{\Delta t}$ (Strang splitting with a symplectic core), ensuring definiteness
  constraints by construction.}
  \label{fig:phast_regime_spectrum}
\end{figure}
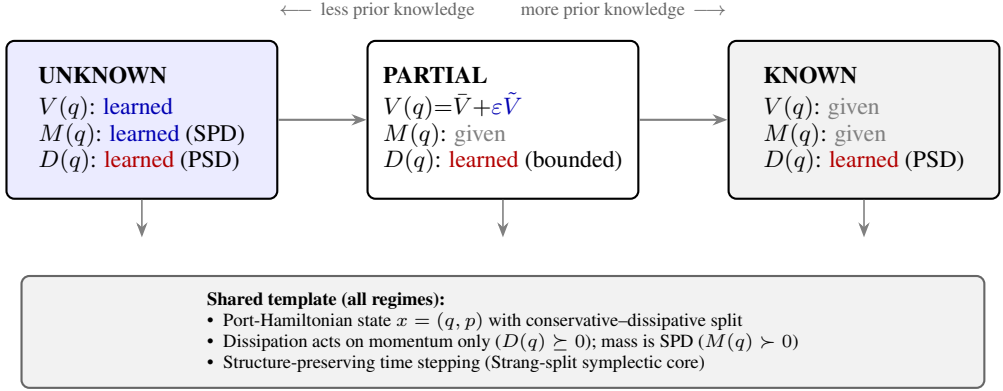

% fig_gallery_overview.tex — Benchmark suite overview (native LaTeX)
% \input{fig_gallery_overview} from experiments.tex (inside \ifdefined\isarxiv)

% Paper-consistent colors (from presentation palette)
\providecolor{phastblue}{RGB}{41,98,163}
\providecolor{dissipred}{RGB}{180,50,50}
\providecolor{floworange}{RGB}{230,120,30}

\begin{figure}[t]
\centering

% ───────────── Row 1: System schematics ─────────────
\begin{tikzpicture}[scale=0.82, every node/.style={font=\small}, >=Stealth]

% ======== (A) Windy Pendulum ========
\begin{scope}[shift={(0,0)}]
    % Panel label (above box)
    \node[font=\small\bfseries, anchor=south] at (-0.1,1.65) {(A) Windy Pendulum};

    % Wind environment box
    \draw[thick, dashed, rounded corners=6pt, fill=cyan!4]
        (-2.5,-2.5) rectangle (2.2,1.5);

    % Wind arrows (wavy)
    \foreach \y in {-1.4, -0.4, 0.6} {
        \draw[->, thick, cyan!55, decorate,
              decoration={snake, amplitude=1pt, segment length=4pt}]
            (-2.3,\y) -- (-1.5,\y);
    }
    \node[font=\tiny, cyan!70!black, rotate=90] at (-2.7,-0.4) {wind};

    % Pivot mount
    \fill[black!15] (-0.4,0.18) rectangle (0.4,0);
    \draw[thick] (-0.4,0) -- (0.4,0);
    \fill (0,0) circle (1.5pt);

    % Vertical reference
    \draw[dashed, black!35] (0,0) -- (0,-2.2);

    % Rod + mass
    \draw[very thick, phastblue!70!black, line cap=round] (0,0) -- (230:2.2);
    \fill[phastblue!60] (230:2.2) circle (5pt);
    \node[font=\small, below left=1pt] at (230:2.2) {$m$};

    % Length
    \coordinate (PendMid) at ($(0,0)!0.5!(230:2.2)$);
    \draw[|-, thin, black!60] ($(0,0)+(0.22,-0.08)$) -- ($(PendMid)+(0.27,0.05)$);
    \node[right, black!60, font=\scriptsize] at ($(PendMid)+(0.22,0)$) {$\ell$};

    % Angle arc
    \draw[->, thick] (270:0.7) arc (270:230:0.7);
    \node at (248:1.0) {$\theta$};

    % Damping indicator (snake arc)
    \draw[<-, thick, dissipred, decorate,
          decoration={snake, amplitude=0.5mm, segment length=2.5mm}]
        ($(0,0)+(250:1.5)$) arc (250:288:1.5);
    \node[dissipred, font=\scriptsize] at (0.85,-1.65) {$d(\theta)$};

    % Gravity
    \draw[->, thick, black!50] (1.6,0.6) -- (1.6,-0.2);
    \node[right, black!50, font=\scriptsize] at (1.6,0.2) {$g$};

    % Damping equation
    \node[draw, rounded corners=2pt, fill=white, inner sep=3pt, font=\small]
        at (-0.1,-3.1) {$d(\theta)=d_0+\Delta d\,|\!\sin\theta|$};
\end{scope}

% ======== (B) Windy Double Pendulum ========
\begin{scope}[shift={(7.5,0)}]
    % Panel label (above box)
    \node[font=\small\bfseries, anchor=south] at (-0.1,1.65) {(B) Windy Double Pendulum};

    % Wind environment box
    \draw[thick, dashed, rounded corners=6pt, fill=cyan!4]
        (-2.5,-3.4) rectangle (2.2,1.5);

    % Wind arrows (wavy)
    \foreach \y in {-2.0, -0.8, 0.4} {
        \draw[->, thick, cyan!55, decorate,
              decoration={snake, amplitude=1pt, segment length=4pt}]
            (-2.3,\y) -- (-1.5,\y);
    }
    \node[font=\tiny, cyan!70!black, rotate=90] at (-2.7,-0.8) {wind};

    % Pivot mount
    \fill[black!15] (-0.4,0.18) rectangle (0.4,0);
    \draw[thick] (-0.4,0) -- (0.4,0);
    \fill (0,0) circle (1.5pt);

    % Vertical reference
    \draw[dashed, black!35] (0,0) -- (0,-3.1);

    % Link 1
    \coordinate (E1) at (240:1.7);
    \draw[very thick, phastblue!70!black, line cap=round] (0,0) -- (E1);
    \fill (E1) circle (1.5pt);

    % CoM 1
    \coordinate (C1) at ($(0,0)!0.55!(E1)$);
    \fill[phastblue!60] (C1) circle (5pt);
    \node[font=\small, left=3pt] at (C1) {$m_1$};

    % Angle 1
    \draw[->, thick] (270:0.6) arc (270:240:0.6);
    \node at (253:0.85) {$\theta_1$};

    % Damping at joint 1 (snake arc)
    \draw[<-, thick, dissipred, decorate,
          decoration={snake, amplitude=0.4mm, segment length=2mm}]
        ($(0,0)+(255:1.15)$) arc (255:282:1.15);
    \node[dissipred, font=\scriptsize] at (0.6,-1.2) {$d_1(\theta_1)$};

    % Link 2
    \coordinate (E2) at ($(E1)+(255:1.5)$);
    \draw[very thick, phastblue!50!black, line cap=round] (E1) -- (E2);

    % CoM 2
    \coordinate (C2) at ($(E1)!0.55!(E2)$);
    \fill[phastblue!45] (C2) circle (5pt);
    \node[font=\small, left=3pt] at (C2) {$m_2$};

    % Angle 2 (relative, at joint E1)
    \draw[->, thick] ($(E1)+(270:0.5)$) arc (270:255:0.5);
    \node at ($(E1)+(260:0.75)$) {$\theta_2$};

    % Damping at joint 2 (snake arc)
    \draw[<-, thick, dissipred, decorate,
          decoration={snake, amplitude=0.4mm, segment length=2mm}]
        ($(E1)+(262:0.95)$) arc (262:285:0.95);
    \node[dissipred, font=\scriptsize] at ($(E1)+(0.55,-1.0)$) {$d_2(\theta_2)$};

    % Gravity
    \draw[->, thick, black!50] (1.6,0.6) -- (1.6,-0.2);
    \node[right, black!50, font=\scriptsize] at (1.6,0.2) {$g$};

    % M(q) annotation
    \node[draw, rounded corners=2pt, fill=phastblue!8, inner sep=2pt,
          font=\scriptsize, phastblue!70!black]
        at (1.5,-2.2) {$M(q)$};

    % Damping equation
    \node[draw, rounded corners=2pt, fill=white, inner sep=3pt, font=\small]
        at (-0.1,-4.0) {$d_i(\theta_i)=b_i+\Delta b_i\,|\!\sin\theta_i|$};
\end{scope}

% ======== (C) Windy Cart-Pole ========
\begin{scope}[shift={(15,0)}]
    % Panel label (above box)
    \node[font=\small\bfseries, anchor=south] at (0.1,1.65) {(C) Windy Cart-Pole};

    % Wind environment box
    \draw[thick, dashed, rounded corners=6pt, fill=cyan!4]
        (-2.8,-2.5) rectangle (3.0,1.5);

    % Wind arrows (wavy)
    \foreach \y in {-0.7, 0.3, 1.1} {
        \draw[->, thick, cyan!55, decorate,
              decoration={snake, amplitude=1pt, segment length=4pt}]
            (-2.6,\y) -- (-1.8,\y);
    }
    \node[font=\tiny, cyan!70!black, rotate=90] at (-3.0,0.2) {wind};

    % Ground/rail
    \fill[black!10] (-2.2,-1.9) rectangle (2.5,-1.7);
    \draw[thick, black!50] (-2.2,-1.7) -- (2.5,-1.7);

    % Cart
    \draw[very thick, fill=black!15, rounded corners=2pt]
        (-0.7,-1.67) rectangle (0.7,-1.0);
    \fill[black] (-0.4,-1.7) circle (0.08);
    \fill[black] (0.4,-1.7) circle (0.08);
    \node[font=\small] at (0,-1.33) {$M$};

    % Cart position arrow
    \draw[<->, thick, black!60] (-1.8,-2.1) -- (0,-2.1);
    \node[below, black!60, font=\scriptsize] at (-0.9,-2.1) {$x$};

    % Pole pivot
    \fill[black] (0,-1.0) circle (0.05);

    % Vertical reference
    \draw[dashed, black!35] (0,-1.0) -- (0,1.3);

    % Pole
    \coordinate (PoleEnd) at ($(0,-1.0)+(70:2.2)$);
    \draw[very thick, phastblue!70!black, line cap=round] (0,-1.0) -- (PoleEnd);

    % Pole mass
    \fill[phastblue!60] (PoleEnd) circle (5pt);
    \node[above right, font=\small] at ($(PoleEnd)+(0.1,0)$) {$m$};

    % Length
    \coordinate (PoleMid) at ($(0,-1.0)!0.5!(PoleEnd)$);
    \draw[|-, thin, black!60] ($(0,-1.0)+(0.2,-0.1)$) -- ($(PoleMid)+(0.25,0.05)$);
    \node[right, black!60, font=\scriptsize] at ($(PoleMid)+(0.2,0)$) {$\ell$};

    % Angle arc
    \draw[->, thick] (0,0.15) arc (90:70:1.15);
    \node[font=\small] at (0.4,0.5) {$\theta$};

    % === DUAL DAMPING ===

    % Damping 1: Angular wind damping on pole (dissipred)
    \draw[<-, thick, dissipred, decorate,
          decoration={snake, amplitude=0.5mm, segment length=2.5mm}]
        (0.55,-0.7) arc (10:65:0.55);
    \node[dissipred, font=\scriptsize] at (1.35,-0.1) {$d(\theta)$};

    % Damping 2: Cart friction (floworange)
    \draw[<-, thick, floworange, decorate,
          decoration={snake, amplitude=0.5mm, segment length=3mm}]
        (-0.7,-1.33) -- (-1.5,-1.33);
    \node[floworange, font=\scriptsize, above] at (-1.45,-1.18) {$d_{\mathrm{c}}$};

    % Gravity
    \draw[->, thick, black!50] (2.3,0.8) -- (2.3,0.0);
    \node[right, black!50, font=\scriptsize] at (2.3,0.4) {$g$};

    % M(q) annotation
    \node[draw, rounded corners=2pt, fill=phastblue!8, inner sep=2pt,
          font=\scriptsize, phastblue!70!black]
        at (2.2,-0.8) {$M(q)$};

    % Damping equation (two-color)
    \node[draw, rounded corners=2pt, fill=white, inner sep=3pt, font=\small]
        at (0.1,-3.1) {$D(q){=}\mathrm{diag}(%
            \textcolor{floworange}{d_{\mathrm{c}}},\;%
            \textcolor{dissipred}{d_0{+}\Delta d\,|\!\sin\theta|})$};
\end{scope}

\end{tikzpicture}

\medskip

% ───────────── Row 2: Quantitative comparison ─────────────
\textbf{(D) Quantitative comparison}\par\smallskip
{\small
\setlength{\tabcolsep}{4pt}
\renewcommand{\arraystretch}{1.1}
\begin{tabular}{@{}l  rr  rr  rr@{}}
    \toprule
    & \multicolumn{2}{c}{\textbf{W.\ Pendulum}}
      & \multicolumn{2}{c}{\textbf{W.\ Double Pend.}}
      & \multicolumn{2}{c}{\textbf{W.\ Cart-Pole}} \\
    \cmidrule(lr){2-3} \cmidrule(lr){4-5} \cmidrule(lr){6-7}
    \textbf{Model}
      & MSE$_{100}$ & $R^2_D$
      & MSE$_{100}$ & $R^2_D$
      & MSE$_{100}$ & $R^2_D$ \\
    \midrule
    \multicolumn{7}{l}{\textit{PHAST (ours)}} \\
    \;\; KNOWN   & 0.064 & \best{1.00}  & 0.532 & \best{1.00}  & 0.033 & $-$1.10 \\
    \;\; PARTIAL & \best{0.056} & 0.18  & \best{0.526} & $-$1.09  & \best{0.014} & \best{-0.24} \\
    \;\; UNKNOWN & 0.737 & $\ll\!0$  & 0.866 & $\ll\!0$  & 0.243 & $\ll\!0$ \\
    \midrule
    S5     & 0.954 & \multicolumn{1}{c}{---} & 2.910 & \multicolumn{1}{c}{---} & 5.145 & \multicolumn{1}{c}{---} \\
    LinOSS & 1.536 & \multicolumn{1}{c}{---} & 3.339 & \multicolumn{1}{c}{---} & 4.167 & \multicolumn{1}{c}{---} \\
    VPT    & 3.227 & \multicolumn{1}{c}{---} & 2.464 & \multicolumn{1}{c}{---} & 2.103 & \multicolumn{1}{c}{---} \\
    \bottomrule
\end{tabular}
}

\caption{\textbf{Benchmark suite overview.}
(A--C)~Schematics of the three representative dissipative systems with position-dependent damping:
Windy Pendulum, Windy Double Pendulum (configuration-dependent $M(q)$),
and Windy Cart-Pole ($\mathbb{R}\times\mathbb{S}^1$, per-DOF damping, $M(q)$).
The Cart-Pole features two distinct damping mechanisms:
constant viscous cart friction (\textcolor{floworange}{$d_{\mathrm{c}}$})
and position-dependent angular wind damping (\textcolor{dissipred}{$d(\theta)$}).
(D)~Wrapped-angle rollout MSE at $H{=}100$ (lower$=$ better) and damping identifiability $R^2_D$
(higher$=$ better; bold$=$ best per column).
PHAST (KNOWN) achieves near-perfect identifiability ($R^2_D \approx 1$) on Pendulum and Double Pendulum;
PHAST (PARTIAL) achieves the best forecasting across all three systems.
All models use 3k--9k parameters (exact counts vary by state dimension; see Appendix).
Baselines do not expose explicit damping fields (---).
Appendix~\ref{app:environments} describes the full suite of thirteen benchmarks,
including conservative variants, the harmonic oscillator, and five non-mechanical systems (Tables~\ref{tab:qonly_rollout_rlc}--\ref{tab:qonly_rollout_predprey}).}
\label{fig:gallery_overview}
\end{figure}

A critical but often overlooked issue emerges in practice: \textbf{forecasting accuracy and parameter recovery are distinct objectives}.
A model may achieve good rollouts by using $\Dmat(q)$ as a ``stabilizer'' that absorbs errors from $V$ or $\Mmat$, rather than learning the true physical damping.
More broadly, without additional structure the inverse problem can be non-identifiable (gauge freedom), even in conservative systems; Appendix~\ref{app:ablations:gauge_freedom} provides a controlled illustration on the double pendulum.
We therefore separate forecasting and identifiability in both the model design and the evaluation (Sec.~\ref{sec:experiments:eval}), and empirically observe that stronger physical structure improves identifiability (Table~\ref{tab:windy_identifiability}).

Our contributions are:
\begin{enumerate}
    \item \textbf{Unified architecture}: A single port-Hamiltonian template spanning three knowledge regimes (KNOWN, PARTIAL, UNKNOWN) by swapping component instantiations for $V$, $\Mmat$, $\Dmat$.
    \item \textbf{Householder-style parameterizations}: Low-rank outer-product expansions that guarantee $\Dmat(q) \succeq 0$ and $\Mmat(q) \succ 0$ by construction with efficient structured primitives: applying $\Dmat(q)$ to a vector costs $O(nr)$, and computing $v=\Mmat(q)^{-1}p$ via Woodbury is $O(nr^2{+}r^3)$, where $n$ is the number of degrees of freedom and $r$ is the number of rank-1 terms (Appendix~\ref{app:ablations:timing}).

    \item \textbf{Damping bounds for identifiability.}
We show that unconstrained dissipation can absorb model mismatch and obscure
physical parameter recovery.
PHAST therefore bounds the effective damping strength, improving identifiability
in grey-box settings while preserving stability (Table~\ref{tab:windy_identifiability}; Appendix~\ref{app:ablations})
    % \item \textbf{Damping bounds for identifiability}: Spectral constraints on the (nonnegative) damping strengths $\{\beta_i(q)\}$ in our low-rank $\Dmat(q)$ parameterization, enforced via $\sum_{i=1}^{r} \beta_i(q) \leq \bar{\beta}$, that mitigate the ``error-sink'' failure mode where damping absorbs model mismatch, improving damping recovery in grey-box settings (Table~\ref{tab:windy_identifiability}; Appendix~\ref{app:ablations}).
    \item \textbf{Structure-preserving integration}: Strang splitting that preserves the continuous-time dissipation structure to $O(\dt^2)$.
    \item \textbf{Two-axis evaluation}: Explicit separation of forecasting metrics (rollout MSE) from identifiability metrics (damping $R^2$), revealing trade-offs invisible to single-axis evaluation.
\end{enumerate}

Across thirteen q-only benchmarks---including non-mechanical systems (RLC circuit, Lennard--Jones cluster, heat exchange, N-body gravity, predator--prey)---PHAST achieves the best long-horizon open-loop stability among competitive baselines (Tables~\ref{tab:qonly_rollout_pendulum_h100}--\ref{tab:qonly_rollout_predprey}); on Windy Pendulum, PHAST (KNOWN) recovers the true position-dependent damping with $R^2=0.996$ (Table~\ref{tab:windy_identifiability}).
These results suggest that, on these q-only benchmarks, \textbf{structure is a stronger driver of long-horizon stability than additional capacity}.
Code to reproduce our experiments will be released.

% ---------------------------------------------------------------------------
% 2. Related Work (included from related.tex)
% ---------------------------------------------------------------------------
% related.tex — Related work section (shared between ICML and arXiv builds)

\section{Related Work}
\label{sec:related}

Learning dynamics from data spans energy-based methods, structure-preserving networks, and efficient sequence models; we position PHAST relative to each.

\paragraph{Energy-based learning.}
Hamiltonian Neural Networks (HNNs)~\citep{greydanus2019hamiltonian} learn $\Ham_\theta(x)$ and set $\dot{x} = \Jmat \nabla \Ham_\theta$, while Lagrangian Neural Networks (LNNs)~\citep{cranmer2020lagrangian} learn $L_\theta(q,\dot q)$ and derive Euler--Lagrange dynamics.
Both are conservative by construction ($\Rmat=0$), and LNNs additionally require Hessian computations.

\paragraph{Dissipative extensions.}
Dissipative SymODEN~\citep{zhong2020dissipative} and DHNN~\citep{sosanya2022dissipative} extend HNN/LNN with damping terms.
However, dissipation is often unconstrained and can entangle with errors in $V$ or $\Mmat$, producing reasonable rollouts but poor physical recovery.
PHAST enforces $\Dmat(q)\succeq 0$ by construction, optionally bounds its spectrum, and evaluates forecasting vs.\ identifiability separately.

\paragraph{Structure-preserving networks.}
SympNets~\citep{jin2020sympnets} and Volume-Preserving Transformers~\citep{brantner2024volume} enforce symplectic/volume-preserving updates aligned with conservative Hamiltonian flows.
For dissipative systems, strict volume preservation cannot represent generic phase-space contraction and provides no passivity guarantee.

\paragraph{Port-Hamiltonian learning.}
Prior work on learning port-Hamiltonian systems~\citep{desai2021port,eidnes2023port} largely assumes known structure, while concurrent work~\citep{anonymous2026hamiltonian} learns generalized Hamiltonian dynamics from phase-space data with probabilistic constraints.
PHAST targets partially known/unknown structure and q-only learning while keeping the $\Jmat$--$\Rmat$ decomposition explicit.

\paragraph{Hamiltonian methods for optimal control.}
The PHAST architecture builds upon a sustained research trajectory investigating the integration of geometric, Hamiltonian, and stochastic structures into learning-based dynamics and control.
Early work addressed the fundamental brittleness of Physics-Informed Neural Networks (PINNs) in the presence of noise, proposing Gaussian Process smoothing to recover robust learning when physical priors fail \cite{bajaj2023recipes}.
Complementing this, the \textit{Motion Code} framework established methods for disentangling complex, noisy time-series data into underlying stochastic processes \cite{bajaj2024motion}.
To bridge deep learning with optimal control, the \textit{NeuralPMP} framework was introduced, leveraging the Pontryagin Maximum Principle to learn reduced Hamiltonian dynamics for continuous control tasks \cite{bajaj2021physics}, a methodology subsequently adapted to navigate the complex energy landscapes of molecular dynamics \cite{bajaj2024reinforcement}.
This paper's focus on Hamiltonian structures was generalized to learn diverse dynamical classes—including conservative, dissipative, and port-Hamiltonian systems—from noisy trajectories by enforcing stability constraints \cite{mclennan2025learning}.
Simultaneously, the theoretical underpinnings of Reinforcement Learning were reformulated through a differential dual perspective in \textit{Differential Policy Optimization} (dfPO), enabling pointwise policy refinement \cite{nguyen2025differential}, which was further extended to stochastic settings using rough path theory \cite{nguyen2025stochastic}.
Most recently, these geometric principles were applied to simultaneous navigation and mapping (SNAM), demonstrating how learned Hamiltonian energy landscapes can guide robotic navigation in unknown environments \cite{ellendula2025grl}.
Collectively, these papers motivate PHAST's unified approach to modeling dissipative systems across varying regimes of prior knowledge.

\paragraph{Efficient sequence models.}
State-space models such as S4/S5 and Mamba~\citep{gu2022efficiently,smith2023simplified,gu2023mamba} provide linear-time recurrence; oscillatory variants like LinOSS and D-LinOSS~\citep{linoss2024,dlinoss2025} learn stable oscillation and dissipation time scales.
Unlike PHAST, these models typically do not expose explicit energies or configuration-dependent damping fields, and they provide no port-Hamiltonian passivity guarantee.

\paragraph{The gap.}
Table~\ref{tab:related_comparison} summarizes the landscape and the key properties we target.

\begin{table}[t]
    \centering
    \caption{\textbf{Comparison of dynamics learning approaches.} PHAST combines dissipative modeling ($\Rmat \neq 0$), passivity-aware structure, spectral control, and efficient primitives for long-horizon rollouts. \emph{Passivity} indicates a by-construction storage-function inequality.}
    \label{tab:related_comparison}
    \scriptsize
    \setlength{\tabcolsep}{4pt}
    \begin{tabular}{@{}lcccc@{}}
        \toprule
        \textbf{Method} & \textbf{$\Rmat \neq 0$} & \textbf{Passivity} & \textbf{Spectral} & \textbf{Efficient} \\
        \midrule
        HNN & \xmark & \cmark & \xmark & \cmark \\
        LNN & \xmark & \cmark & \xmark & \xmark \\
        DHNN / Diss.\ SymODEN & \cmark & \xmark & \xmark & \cmark \\
        SympNets / VPT & \xmark & \xmark & \xmark & \cmark \\
        SSMs (S5, LinOSS, D-LinOSS) & --- & \xmark & \xmark & \cmark \\
        \textbf{PHAST (ours)} & \cmark & \cmark & \cmark & \cmark \\
        \bottomrule
    \end{tabular}
\end{table}

% ---------------------------------------------------------------------------
% 3. Method
% ---------------------------------------------------------------------------
% PHAST Methods (paper-facing).
%
% This file is intended to be included from a main paper (e.g., via \input{methods}).
% It assumes the main preamble loads:
%   amsmath, amssymb, amsfonts, booktabs, tikz, pgfplots (optional), xcolor (optional).
% If you want a quick standalone build, see `methods_standalone.tex`.

\section{Method: PHAST for Dissipative Dynamics}
\label{sec:methods}

% ---------------------------------------------------------------------------
% Minimal safe macros (no-ops if already defined in the main preamble).
% ---------------------------------------------------------------------------
\providecommand{\R}{\mathbb{R}}
\providecommand{\half}{\tfrac{1}{2}}
\providecommand{\T}{\top}

\providecommand{\Ham}{H}
\providecommand{\Jmat}{J}
\providecommand{\Rmat}{R}
\providecommand{\Mmat}{M}
\providecommand{\Dmat}{D}
\providecommand{\grad}{\nabla}

This section describes how PHAST enforces the structural decomposition introduced above
(Table~\ref{tab:structured_hamiltonians}, Fig.~\ref{fig:phast_regime_spectrum})
and preserves it under time discretization.
We begin with the observation model, data contract, and passivity guarantees (Sec.~\ref{sec:methods:qonly}),
followed by the parameterizations that enforce definiteness (Sec.~\ref{sec:methods:householder}),
the structure-preserving integrator (Sec.~\ref{sec:methods:integration}),
the training objective (Sec.~\ref{sec:methods:training}),
and port-based control (Sec.~\ref{sec:casimir_forced_mode}).

% ---------------------------------------------------------------------------
% 3.1 Problem Formulation (q-only)
% ---------------------------------------------------------------------------
\subsection{Problem Formulation: Partial Observability (q-only)}
\label{sec:methods:qonly}

As described in Section~\ref{sec:intro}, PHAST operates in the q-only setting
where only generalized positions~$q_t$ are observed and momenta~$p_t$ are latent.
We now detail the three pipeline stages---velocity observer, canonicalizer, and
port-Hamiltonian core---that map a short position history to long-horizon
phase-space rollouts.

Formally, we observe
\begin{equation}
y_t = q_t, \qquad x_t = (q_t, p_t),
\end{equation}
where $p_t$ is the (generally unobserved) conjugate momentum in the underlying
phase state.

\paragraph{Velocity observer.}
We first compute a simple finite-difference estimate
\begin{equation}
\dot q_t^{\mathrm{fd}} :=
\begin{cases}
0, & t = 0, \\
(q_t - q_{t-1})/\dt, & t \ge 1,
\end{cases}
\label{eq:qonly_fd}
\end{equation}
and then apply a causal Temporal Convolutional Network (TCN) observer
$o_\phi$ that predicts an additive correction:
\begin{equation}
\hat{\dot q}_t = \dot q_t^{\mathrm{fd}} + \delta_t,
\qquad
\delta_{0:t} = o_\phi([q_{0:t}, \dot q^{\mathrm{fd}}_{0:t}]).
\label{eq:qonly_observer}
\end{equation}

\paragraph{Canonicalization.}
The estimated velocity is mapped to a phase state via a canonicalizer
$C_\psi : (q_t, \hat{\dot q}_t) \mapsto (q_t, \hat p_t)$.
For mechanical systems, the physical relation is $p = M(q)\dot q$.
When $M$ is known, we use the mass-consistent mapping
$\hat p = M(q)\hat{\dot q}$.
In q-only benchmarks with constant mass, we use the identity canonicalizer
$\hat p = \hat{\dot q}$ and treat $\hat p$ as a velocity-like latent variable
that yields an approximately Markov phase state.

\paragraph{Rollout modes.}
Given a burn-in window $q_{0:K-1}$, PHAST supports two rollout modes.
In \emph{takeover} mode, an initial phase state is inferred at the end of the
context and then integrated open-loop:
\begin{align}
\hat x_{K-1} &= C_\psi(q_{K-1}, \hat{\dot q}_{K-1}), \nonumber\\
\hat x_{t+1} &= \Phi_{\dt}(\hat x_t), \qquad \hat q_t = \Pi_q(\hat x_t).
\label{eq:qonly_takeover}
\end{align}
In \emph{autoregressive} (self-conditioned) mode, the observer is re-applied to
the generated position prefix at each step.
For cross-model comparisons, we report autoregressive rollouts; takeover rollouts
are additionally used as a diagnostic of the learned dynamics.

When $\Mmat(q)$ is available from physics (KNOWN regime), we use the mass-consistent canonicalizer $\hat p=\Mmat(q)\hat{\dot q}$; combined with a nonseparable Hamiltonian integrator, this yields substantial gains over the constant-mass approximation (e.g., $\sim 3\times$ lower $H{=}100$ rollout error on Windy Cart-Pole; Appendix~\ref{app:ablations:nonseparable_mass}, Table~\ref{tab:cartpole_nonseparable_mass}).

\paragraph{Unforced dynamics ($u\equiv0$).}
\label{sec:methods:ph}
All forecasting and filtering experiments in this work consider the autonomous
system $\dot{x} = (\Jmat - \Rmat)\nabla \Ham(x)$
(Eq.~\eqref{eq:ph_dynamics} with $u{=}0$),
which admits the energy balance
\begin{equation}
\frac{d\Ham}{dt} = -\nabla \Ham(x)^\top \Rmat \nabla \Ham(x) \le 0.
\label{eq:passivity}
\end{equation}
Equation~\eqref{eq:passivity} is the passivity certificate: the learned
Hamiltonian $\Ham$ serves as a storage function that is non-increasing along
trajectories, ensuring bounded energy and stable long-horizon rollouts.

\paragraph{Forced dynamics and ports ($u\neq0$).}
The input $u$ enters through a port matrix $G$, with conjugate output
$y^{\text{port}} = G^\top \nabla \Ham(x)$.
The supplied power is $u^\top y^{\text{port}}$, yielding the balance
$\tfrac{d\Ham}{dt} = -\nabla \Ham^\top \Rmat \nabla \Ham + u^\top y^{\text{port}}$,
which reduces to Eq.~\eqref{eq:passivity} when $u=0$.
Section~\ref{sec:casimir_forced_mode} develops the forced form into a
closed-loop Energy--Casimir controller.

% ---------------------------------------------------------------------------
% 3.3 Householder-Style Low-Rank Parameterizations
% ---------------------------------------------------------------------------
\subsection{Householder-Style Low-Rank Parameterizations}
\label{sec:methods:householder}

To enforce definiteness constraints \emph{and} keep inference fast, PHAST uses low-rank outer-product expansions (``Householder-style'' additive updates; not Householder reflections).

\paragraph{Damping (all regimes).}
We parameterize the damping matrix as
\begin{align}
  \Dmat(q) &= d_0 I + \sum_{i=1}^{r} \beta_i(q)\,k_i(q)k_i(q)^\T, \notag \\
  &\quad d_0\ge 0,\ \beta_i(q)\ge 0,\ \|k_i(q)\|=1\ \text{for } i=1,\ldots,r,
  \label{eq:householder_damping}
\end{align}
where $n$ is the number of degrees of freedom (so $q\in\R^n$) and $r$ is the number of rank-1 terms.
This guarantees $\Dmat(q)\succeq 0$ and admits $O(nr)$ matrix-vector products.

\paragraph{Bounding damping strength (identifiability + stability).}
Optionally, we bound the total strength
\begin{equation}
  \sum_{i=1}^{r}\beta_i(q) \le \bar\beta
  \quad\Rightarrow\quad
  \lambda_{\max}(\Dmat(q)) \le d_0 + \bar\beta,
  \label{eq:damping_bound}
\end{equation}
which (i) prevents $\Dmat$ from acting as an ``error sink'' when $V$ or $\Mmat$ are misspecified, and (ii) improves numerical stability of the explicit damping half-step (Algorithm~\ref{alg:phast_step}, lines 2--3 and 10--11) when $\Delta t\,\lambda_{\max}(\Dmat)$ is large.
In our grey-box Windy benchmarks, we set $\bar\beta$ to a physics-calibrated magnitude (e.g., the known damping variation $\Delta d$) in the PARTIAL regime (Appendix~\ref{app:training}).
We enforce Eq.~\eqref{eq:damping_bound} by reparameterizing the strengths as $\beta_i(q)=\beta_{\max}\,\sigma(\cdot)$ with $\beta_{\max}=\bar\beta/r$, so that each term is nonnegative and uniformly bounded.
Appendix Fig.~\ref{fig:householder_geom} provides a schematic intuition for Eq.~\eqref{eq:damping_bound}.

\paragraph{Mass.}
PHAST supports a general configuration-dependent mass $\Mmat(q)\succ 0$.
In our main experiments we use a constant-mass approximation $\Mmat(q)\approx\Mmat$
for efficiency, enabling a fast leapfrog integrator core;
the nonseparable $\Mmat(q)$ variant uses an implicit-midpoint integrator and is
studied in Appendix~\ref{app:ablations:nonseparable_mass}.
When learning $\Mmat$ (UNKNOWN regime), we use an SPD diagonal-plus-low-rank parameterization and compute $\Mmat^{-1}p$ via Woodbury; see Appendix Eq.~\ref{eq:householder_mass} and Alg.~\ref{alg:velocity}.

% ---------------------------------------------------------------------------
% 3.4 Structure-Preserving Integration
% ---------------------------------------------------------------------------
\subsection{Structure-Preserving Integration}
\label{sec:methods:integration}

For dissipative systems we integrate with Strang splitting \cite{strang1968construction,hairer2006geometric}:
\begin{equation}
  \Phi_{\Delta t}=\Phi_D^{\Delta t/2}\circ\Phi_H^{\Delta t}\circ\Phi_D^{\Delta t/2},
  \label{eq:strang}
\end{equation}
where $\Phi_H^{\Delta t}$ is a conservative (symplectic) update and $\Phi_D^{\Delta t/2}$ is a half-step of damping.
In our main experiments we use the leapfrog (St\"ormer--Verlet) conservative core; we additionally report ablations with implicit midpoint and Cayley--Neumann in settings where additional stability is helpful.

\paragraph{Stacking and timestep.}
The PHAST transition $\Phi_{\dt}$ can optionally compose $L$ integration substeps to produce a single predicted transition $x_t \mapsto x_{t+1}$.
All substeps reuse the same learned energy components ($V$, $\Mmat$, $\Dmat$) and differ only in the intermediate states they act on (and, if enabled, per-substep timesteps).
In our experiments we set $L{=}1$ for interpretability and efficiency, and initialize the internal timestep to the environment timestep $\dt$; unless otherwise noted, this timestep is learnable (Appendix~\ref{app:ablations} reports $L$ and timestep ablations).

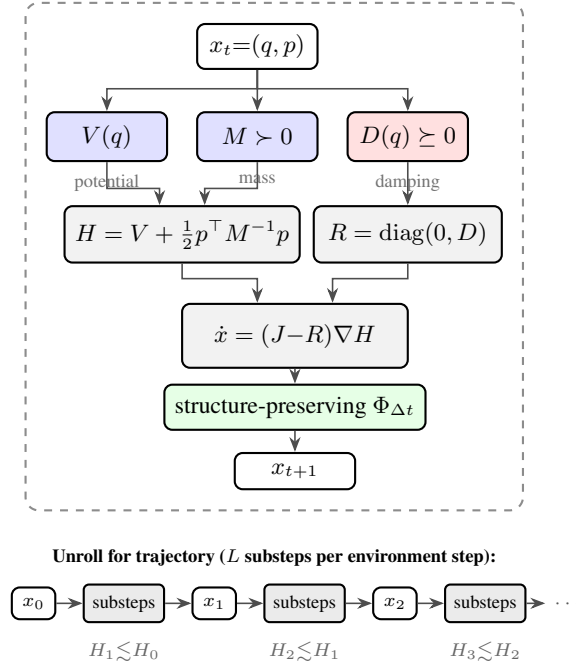
\begin{figure}[t]
  \centering
  \begin{tikzpicture}[
      box/.style={rectangle, draw=black, thick, rounded corners=3pt, align=center, inner sep=4pt, font=\small},
      comp/.style={box, minimum width=1.6cm, minimum height=0.65cm},
      assembly/.style={box, fill=black!5, minimum width=2.5cm, minimum height=0.75cm},
      dynamics/.style={box, fill=black!5, minimum width=3.0cm, minimum height=0.85cm},
      integ/.style={box, fill=green!10, minimum width=2.4cm, minimum height=0.6cm},
      cell/.style={draw=black!45, dashed, thick, rounded corners=5pt, inner sep=8pt},
      arrow/.style={-{Stealth[length=2mm]}, semithick, black!70},
      label/.style={font=\scriptsize, text=black!55},
      unrollbox/.style={rectangle, draw=black, thick, rounded corners=2pt, fill=black!8, minimum width=0.95cm, minimum height=0.5cm, font=\scriptsize},
  ]
    % ===== PHAST core transition cell =====
    \begin{scope}[local bounding box=phastcore]
      % Input state
      \node[box, fill=white, minimum width=1.6cm] (state) at (0,0) {$x_t{=}(q,p)$};

      % Component blocks - horizontal row
      \node[comp, fill=blue!12] (V) at (-2.0,-1.2) {$V(q)$};
      \node[comp, fill=blue!12] (M) at (0,-1.2) {$M\succ 0$};
      \node[comp, fill=red!12]  (D) at (2.0,-1.2) {$D(q)\succeq 0$};

      % Labels BELOW components
      \node[label, below=0.04cm of V] {potential};
      \node[label, below=0.04cm of M] {mass};
      \node[label, below=0.04cm of D] {damping};

      % Assembly row - H and R
      \node[assembly] (H) at (-1.0,-2.5) {$H = V + \tfrac12 p^\top M^{-1}p$};
      \node[assembly] (R) at (2.0,-2.5) {$R = \mathrm{diag}(0,D)$};

      % Port-Hamiltonian dynamics (no extra inequality inside the box)
      \node[dynamics] (ph) at (0.5,-3.85) {$\dot x = (J{-}R)\nabla H$};

      % Integrator (make meaning explicit)
      \node[integ] (phi) at (0.5,-4.8) {structure-preserving $\Phi_{\Delta t}$};

      % Output state
      \node[box, fill=white, minimum width=1.6cm] (next) at (0.5,-5.65) {$x_{t+1}$};

      % Arrows from state to components
      \draw[arrow] (state.south) -- ++(0,-0.25) -| (V.north);
      \draw[arrow] (state.south) -- ++(0,-0.25) -- (M.north);
      \draw[arrow] (state.south) -- ++(0,-0.25) -| (D.north);

      % Arrows from V and M to H
      \draw[arrow] (V.south) -- ++(0,-0.35) -| ([xshift=-0.3cm]H.north);
      \draw[arrow] (M.south) -- ++(0,-0.35) -| ([xshift=0.3cm]H.north);

      % Arrow from D to R
      \draw[arrow] (D.south) -- ++(0,-0.35) -- (R.north);

      % Arrows from H and R to PH dynamics
      \draw[arrow] (H.south) -- ++(0,-0.2) -| ([xshift=-0.5cm]ph.north);
      \draw[arrow] (R.south) -- ++(0,-0.2) -| ([xshift=0.5cm]ph.north);

      \draw[arrow] (ph) -- (phi);
      \draw[arrow] (phi) -- (next);
    \end{scope}

    % Dashed cell border with label
    \node[cell, fit=(phastcore), label={[font=\scriptsize\bfseries]above:PHAST core transition}] (cellbox) {};

    % ===== Unroll strip below =====
    \node[font=\scriptsize, below=0.38cm of cellbox.south] (unrolllabel)
      {\textbf{Unroll for trajectory (}$L$\textbf{ substeps per environment step):}};

    \node[box, fill=white, minimum width=0.55cm, font=\scriptsize, below=0.15cm of unrolllabel, xshift=-3.2cm] (x0) {$x_0$};
    \node[unrollbox, right=0.35cm of x0] (pk1) {substeps};
    \node[box, fill=white, minimum width=0.55cm, font=\scriptsize, right=0.35cm of pk1] (x1) {$x_1$};
    \node[unrollbox, right=0.35cm of x1] (pk2) {substeps};
    \node[box, fill=white, minimum width=0.55cm, font=\scriptsize, right=0.35cm of pk2] (x2) {$x_2$};
    \node[unrollbox, right=0.35cm of x2] (pk3) {substeps};
    \node[font=\scriptsize, right=0.25cm of pk3, text=black!65] (dots) {$\cdots$};

    \draw[arrow] (x0) -- (pk1);
    \draw[arrow] (pk1) -- (x1);
    \draw[arrow] (x1) -- (pk2);
    \draw[arrow] (pk2) -- (x2);
    \draw[arrow] (x2) -- (pk3);
    \draw[arrow] (pk3) -- (dots);

    % Energy inequalities (schematic)
    \node[font=\scriptsize, text=black!60, below=0.12cm of pk1] {$H_1{\lesssim}H_0$};
    \node[font=\scriptsize, text=black!60, below=0.12cm of pk2] {$H_2{\lesssim}H_1$};
    \node[font=\scriptsize, text=black!60, below=0.12cm of pk3] {$H_3{\lesssim}H_2$};

  \end{tikzpicture}
  \caption{\textbf{PHAST core transition computation graph.}
  Components $V$, $M$, and $D$ assemble into Hamiltonian $H$ and dissipation $R$; in the unforced case ($u=0$),
  the port-Hamiltonian form implies $\tfrac{dH}{dt}\le 0$ in continuous time.
  PHAST advances the phase state using a structure-preserving discrete-time map $\Phi_{\Delta t}$ (Strang splitting)
  and optionally composes $L$ integration substeps per environment step (Eq.~\eqref{eq:phast_substeps}).
  The $H_{k+1}\lesssim H_k$ annotations are schematic; discrete-time energy monotonicity depends on timestep and integrator
  (Appendix~\ref{app:math:strang}).}
  \label{fig:phast_step}
\end{figure}

\begin{algorithm}[t]
\caption{PHAST step $\Phi_{\dt}(q,p)\to(q^+,p^+)$ (Strang splitting with leapfrog core; constant mass)}
\label{alg:phast_step}
\begin{algorithmic}[1]
\REQUIRE $(q,p)\in\R^{2n}$, time step $\dt$
\STATE \textbf{Half dissipation step}
\STATE $v \leftarrow \Mmat^{-1}p$
\STATE $p \leftarrow p - \frac{\dt}{2}\,\Dmat(q)\,v$
\STATE \textbf{Conservative step (leapfrog / St\"ormer--Verlet)}
\STATE $p_{1/2} \leftarrow p - \frac{\dt}{2}\,\grad V(q)$
\STATE $v_{1/2} \leftarrow \Mmat^{-1}p_{1/2}$
\STATE $q^+ \leftarrow q + \dt\,v_{1/2}$
\STATE $p^+ \leftarrow p_{1/2} - \frac{\dt}{2}\,\grad V(q^+)$
\STATE \textbf{Half dissipation step}
\STATE $v^+ \leftarrow \Mmat^{-1}p^+$
\STATE $p^+ \leftarrow p^+ - \frac{\dt}{2}\,\Dmat(q^+)\,v^+$
\RETURN $(q^+,p^+)$
\end{algorithmic}
\end{algorithm}

\paragraph{Discrete-time caveat.}
Continuous-time passivity in Eq.~\eqref{eq:passivity} does not automatically imply monotone energy decay at finite step size.
With the explicit damping half-step, discrete energy monotonicity can fail if the system is stiff; bounding Eq.~\eqref{eq:damping_bound} (or using an implicit damping update) mitigates this in practice.

% ---------------------------------------------------------------------------
% 3.5 Training Objective
% ---------------------------------------------------------------------------
\subsection{Training Objective}
\label{sec:methods:training}

PHAST supports end-to-end training with a weighted multi-term objective:
\begin{equation}
  \mathcal{L}
  = \lambda_{\text{data}}\mathcal{L}_{\text{data}}
  + \lambda_{\text{pass}}\mathcal{L}_{\text{pass}}
  + \lambda_{\text{energy}}\mathcal{L}_{\text{energy}}
  + \lambda_{\text{roll}}\mathcal{L}_{\text{roll}}.
  \label{eq:loss}
\end{equation}

\paragraph{Loss breakdown.}
Let $y_{0:T-1}$ denote the observed sequence (in q-only, $y_t=q_t$), where $T$ is the sequence length.
Let $\mathrm{err}(\hat y, y)$ denote the appropriate per-step error on the manifold (Appendix~\ref{app:metrics}).
For full-state training (when $x=(q,p)$ is available), we use teacher-forced one-step predictions $\hat x_{t+1}=\Phi_{\dt}(x_t)$ and define $H_t=\Ham(x_t)$ and $\hat H_{t+1}=\Ham(\hat x_{t+1})$:
\begin{align}
  \mathcal{L}_{\text{data}}
  &= \frac{1}{T-1}\sum_{t=0}^{T-2} \mathrm{err}(\hat y_{t+1}, y_{t+1}), \\
  \mathcal{L}_{\text{pass}}
  &= \frac{1}{T-1}\sum_{t=0}^{T-2} \max(0, \hat H_{t+1} - H_t), \\
	  \mathcal{L}_{\text{energy}}
	  &= \frac{1}{T-1}\sum_{t=0}^{T-2}\left|\frac{\hat H_{t+1}-H_t}{\dt} + v_t^\T \Dmat(q_t) v_t\right|, \\
	  &\qquad v_t := \Mmat^{-1}p_t, \\
	  \mathcal{L}_{\text{roll}}
	  &= \mathbb{E}_{t_0}\left[\frac{1}{H_{\mathrm{roll}}}\sum_{h=1}^{H_{\mathrm{roll}}}\mathrm{err}(\tilde y_{t_0+h}, y_{t_0+h})\right].
\end{align}
Here $\tilde x_0=x_{t_0}$, $\tilde x_{h+1}=\Phi_{\dt}(\tilde x_h)$, and $t_0$ is sampled uniformly from valid start indices.
In the q-only setting, energy diagnostics can be formed using finite-difference velocity estimates as in Appendix~\ref{app:metrics}.

We always train with $\mathcal{L}_{\text{data}}$ for next-step prediction; $\mathcal{L}_{\text{pass}}$ penalizes discrete energy increases and $\mathcal{L}_{\text{energy}}$ matches a dissipation budget (when labels are available) or a label-free residual under Strang splitting.
Optionally, $\mathcal{L}_{\text{roll}}$ trains the model in the same open-loop mode used at evaluation, reducing teacher-forcing train--test mismatch.
Unless otherwise noted, our main benchmark tables use $\lambda_{\text{pass}}=\lambda_{\text{energy}}=\lambda_{\text{roll}}=0$ and report long-horizon open-loop stability and energy/damping diagnostics at evaluation time.

The complete training pseudocode is given in Algorithm~\ref{alg:phast_train}.

% ---------------------------------------------------------------------------
% 3.6 Energy-Casimir Control
% ---------------------------------------------------------------------------
\subsection{Energy-Casimir Control (forced dynamics mode)}
\label{sec:casimir_forced_mode}

\paragraph{Why dynamic control.}
Static (memoryless) feedback can inject damping but cannot reshape the energy
landscape to stabilize a non-natural equilibrium $q^\star$.
A dynamic controller that \emph{stores} energy is needed: a port-Hamiltonian (pH)
controller exchanges energy with the plant through power ports, adding a shaped
potential that moves the minimum of the closed-loop energy to $q^\star$
(Fig.~\ref{fig:casimir_control_structure}).

\paragraph{Plant and controller as two pH systems.}
Let the plant be a controlled pH system
\begin{align}
    \dot x &= (J_p(x) - R_p(x)) \nabla H_p(x) + G_p(x)\,u_p, \\
    y_p &= G_p(x)^\top \nabla H_p(x),
\end{align}
with energy balance
$\dot H_p = -(\nabla H_p)^\top R_p(\nabla H_p) + y_p^\top u_p$.
Here $u_p$ is the control input (torque-like) and $y_p$ is the conjugate port
output (velocity-like); their product $y_p^\top u_p$ is the power supplied to
the plant.
The controller has a virtual state $\xi$ with shaped storage
\begin{equation}
    H_c(\xi) := \tfrac{k_c}{2}\,(\xi - q^\star)^2,
\end{equation}
acting as a virtual spring centered at the target $q^\star$.
Its port variables are
\begin{equation}
    \dot\xi = u_c, \qquad y_c = \nabla H_c(\xi) = k_c(\xi - q^\star).
\end{equation}

\paragraph{Power-preserving interconnection.}
We couple plant and controller through the port assignment
\begin{equation}
    u_p = -y_c + v, \qquad u_c = y_p,
    \label{eq:forced_interconnect}
\end{equation}
where $v$ is an auxiliary channel (set to zero for now).
When $v{=}0$ the cross-power cancels:
$y_p^\top u_p + y_c^\top u_c = -y_p^\top y_c + y_c^\top y_p = 0$,
so the coupling transfers energy between plant and controller but creates none.

\paragraph{Casimir invariant.}
Define $\mathcal{C}(q,\xi) := q - \xi$.
Because $u_c = y_p$ and, for collocated mechanical systems, $y_p = \dot q$,
we obtain $\dot{\mathcal{C}} = \dot q - \dot\xi = y_p - u_c = 0$.
If $\xi(0)=q(0)$, then $\xi(t) = q(t)$ for all $t$, so $H_c(\xi)$ acts as a
shaped potential \emph{on the plant configuration}: the controller's virtual
spring tracks the physical position forever.
$\mathcal{C}$ is a \emph{Casimir function}---a quantity conserved by the
interconnection structure (i.e.\ $\nabla\mathcal{C}\in\ker J_{\mathrm{cl}}^\top$),
independently of the energy functions.

\paragraph{Closed-loop energy decrease.}
The closed-loop storage $H_{\mathrm{cl}} := H_p + H_c$ satisfies
\begin{equation}
    \dot H_{\mathrm{cl}}
    = -(\nabla H_p)^\top R_p(\nabla H_p) + y_p^\top v.
\end{equation}
With damping injection $v = -d_{\mathrm{inj}}\,y_p$ ($d_{\mathrm{inj}}>0$), this yields
$\dot H_{\mathrm{cl}} \le -d_{\mathrm{inj}}\|y_p\|^2 \le 0$,
guaranteeing asymptotic convergence to the minimum of $H_{\mathrm{cl}}$ at $q^\star$.

\paragraph{Forced dynamics mode.}
In the experiments of Appendix~\ref{app:casimir} we additionally inject an
exogenous signal $v_{\mathrm{ext}}$ (for excitation or tracking), setting
\begin{equation}
    v(t) = v_{\mathrm{ext}}(t) - d_{\mathrm{inj}}\,\hat y(t),
    \label{eq:v_forced}
\end{equation}
where $\hat y$ is the (possibly estimated) port output.
Substituting \eqref{eq:forced_interconnect}--\eqref{eq:v_forced} into the power
balance gives
\begin{equation}
    y_p^\top u_p + y_c^\top u_c
    = y_p^\top v_{\mathrm{ext}} - d_{\mathrm{inj}}\, y_p^\top \hat y.
    \label{eq:forced_power_balance}
\end{equation}
In the ideal case $\hat y \equiv y_p$, the closed-loop storage satisfies the
forced passivity bound
\begin{equation}
    \dot H_{\mathrm{cl}}
    \le
    y_p^\top v_{\mathrm{ext}} - d_{\mathrm{inj}}\|y_p\|^2,
    \label{eq:forced_passivity_ideal}
\end{equation}
i.e.\ the closed loop is passive from $v_{\mathrm{ext}}$ to $y_p$ with strict
dissipation.
The Casimir invariance is compatible with forced mode: $v_{\mathrm{ext}}$ enters
only through $u_p$ and does not alter the constraint $u_c = y_p$.

\paragraph{Q-only limitation.}
In q-only settings $y_p$ is not observed directly and must be replaced by an
estimate $\hat y \approx y_p$.
The dissipative term then becomes $-d_{\mathrm{inj}}\,y_p^\top\hat y$, which can
lose definiteness when $\hat y \neq y_p$: estimation error may inject or
dissipate energy.
This is the central partial-observability bottleneck for port-based control.

\paragraph{Discrete-time implementation.}
Let $q_t^{\mathrm{meas}}$ denote q-only measurements and $\hat y_t$ the online
port estimate.
The controller is implemented as
\begin{align}
    u_t
    &=
    -y_c(\xi_t) + v_{\mathrm{ext},t} - d_{\mathrm{inj}}\,\hat y_t,
    \label{eq:forced_u_disc}
    \\
    \xi_{t+1}
    &=
    \xi_t + \Delta t\Big(\hat y_t + k_\xi(q_t^{\mathrm{meas}}-\xi_t)\Big),
    \label{eq:forced_xi_disc}
\end{align}
where $k_\xi \ge 0$ is an optional predictor--corrector gain that corrects
$\xi$ drift under partial observability: when $\hat y = y_p$ and $k_\xi{=}0$,
the update reduces to a forward-Euler discretization of $\dot\xi = y_p$.
All variants share the same $(k_c, d_{\mathrm{inj}}, k_\xi)$ and forcing signal
$v_{\mathrm{ext},t}$; they differ only in how $\hat y_t$ is produced online.

\paragraph{Port-estimation variants.}
We compare the same family of estimators as in the unforced setting
(oracle/full-state, finite differences, fixed-lag MAP smoothing, learned
observer, etc.) and evaluate stability and control-effort degradation as a
function of port-estimation error (Appendix~\ref{app:casimir}).

\section{Experiments}
\label{sec:experiments}

We evaluate PHAST on q-only systems spanning conservative and dissipative dynamics across six physical domains.
Our evaluation separates two regimes with distinct goals:
\emph{open-loop forecasting} (autonomous rollouts) to assess long-horizon stability and physical identifiability,
and \emph{closed-loop control} (feedback) to assess whether the learned port-Hamiltonian structure remains useful under stabilization
and to isolate the role of online port estimation under partial observability.

\providecommand{\best}[1]{\ensuremath{\boldsymbol{#1}}}

% ---------------------------------------------------------------------------
% 4.1 Open-loop forecasting: setup + protocol + main results
% ---------------------------------------------------------------------------
\subsection{Open-loop forecasting (q-only)}
\label{sec:experiments:openloop}

\subsubsection{Experimental setup}
\label{sec:experiments:setup}

\paragraph{Environments.}
We evaluate on thirteen q-only benchmarks spanning conservative and dissipative dynamics across six physical domains.
\emph{Mechanical systems} (eight benchmarks):
single pendulum (conservative, constant damping, and position-dependent ``windy'' damping),
a windy Cart-Pole on $\mathbb{R}\times\mathbb{S}^1$,
harmonic oscillator (conservative and damped),
and double pendulum (conservative and damped).
\emph{Non-mechanical systems} (five benchmarks covering all three knowledge regimes in Table~\ref{tab:structured_hamiltonians}):
series RLC circuit (KNOWN), Lennard--Jones 3-particle cluster (KNOWN), coupled heat exchange (KNOWN), 3-body gravitational system (PARTIAL), and Lotka--Volterra predator--prey (UNKNOWN).
In all environments only configurations $q_t$ are observed; the simulator evolves trajectories from random initial conditions $(q_0,p_0)$.
All models are trained as next-step predictors that map a short history of observed configurations to $\hat q_{t+1}$;
PHAST internally uses the same history to infer a latent phase state via the FD+TCN observer (Sec.~\ref{sec:methods:qonly}),
while sequence baselines must implicitly learn any required state estimation.
All forecasting benchmarks are unforced ($u=0$ in Eq.~\eqref{eq:ph_dynamics}).
For windy settings we use a position-dependent diagonal damping of the form
$d(q)=d_0+\Delta d\,|\sin q|$ (e.g., $d_0{=}0.3$, $\Delta d{=}0.5$), applied to the relevant angular coordinate.
Some benchmarks (Cart-Pole and Double Pendulum) have configuration-dependent inertia $M(q)$ in the simulator; unless otherwise stated,
we use a separable constant-mass approximation for efficiency (Sec.~\ref{sec:intro:regimes})
and study the nonseparable $M(q)$ variant in Appendix~\ref{app:ablations:nonseparable_mass}.
Appendix~\ref{app:environments} provides detailed environment descriptions and Hamiltonians.

\paragraph{Data protocol.}
Trajectories have length $T{=}200$.
We use $\dt{=}0.05$ for single-pendulum environments, $\dt{=}0.02$ for Cart-Pole and Oscillator, and $\dt{=}0.01$ for Double Pendulum.
For PHAST, we initialize the internal integrator step size to the environment timestep $\dt$ and use a single integration substep per observation step ($L{=}1$);
unless otherwise stated, the internal timestep is learnable.
Dataset sizes are $N_{\text{train}}/N_{\text{val}}/N_{\text{test}} = 1000/200/200$.
Unless otherwise noted, tables report mean $\pm$ std over 5 random model seeds using a fixed dataset shared across models,
and all runs are executed on CPU for reproducibility.
All models are trained for 50 epochs; complete hyperparameters are provided in Appendix~\ref{app:training}.

\paragraph{Training and model selection.}
All methods use the same optimizer and schedule: AdamW (weight decay $10^{-5}$) with a cosine learning-rate schedule.
We select the model parameters with lowest validation MSE (evaluated every 10 epochs) and report test metrics for that selection.

\paragraph{Baselines.}
We compare against GRU~\citep{cho2014learning}, S5~\citep{smith2023simplified}, LinOSS~\citep{linoss2024}, D-LinOSS~\citep{dlinoss2025},
a Transformer~\citep{vaswani2017attention}, and a Volume-Preserving Transformer (VPT)~\citep{brantner2024volume}.
We report trainable parameter counts for all methods; parameter matching is treated as an experimental control rather than an assumption.
All baselines are trained and evaluated as causal seq2seq predictors that map a q-only history to the next configuration ($q_{0:t}\mapsto \hat q_{t+1}$);
they do not use an explicit observer/canonicalizer pipeline.

% ---------------------------------------------------------------------------
% Open-loop protocol: two axes
% ---------------------------------------------------------------------------
\subsubsection{Evaluation protocol: two axes}
\label{sec:experiments:eval}

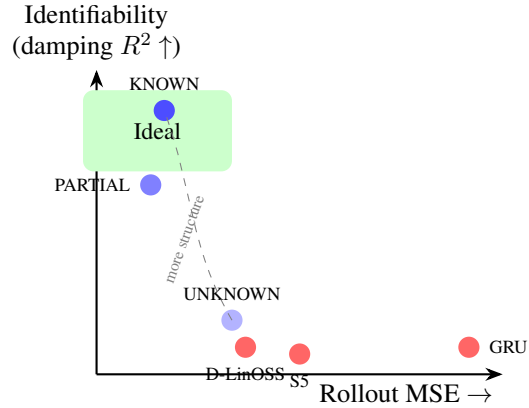
\begin{figure}[t]
    \centering
    \begin{tikzpicture}[scale=0.9]
        \draw[-{Stealth}, thick] (0,0) -- (6.0,0) node[below left] {Rollout MSE $\rightarrow$};
        \draw[-{Stealth}, thick] (0,0) -- (0,4.5) node[above, align=center] {Identifiability\\(damping $R^2$ $\uparrow$)};
        \fill[green!20, rounded corners] (-0.2,3.0) rectangle (2.0,4.2);
        \node[font=\small] at (0.9,3.6) {Ideal};
        \node[circle, fill=blue!50, minimum size=8pt, inner sep=0pt, label={[font=\scriptsize]left:PARTIAL}] at (0.8,2.8) {};
        \node[circle, fill=blue!70, minimum size=8pt, inner sep=0pt, label={[font=\scriptsize]above:KNOWN}] at (1.0,3.9) {};
        \node[circle, fill=blue!30, minimum size=8pt, inner sep=0pt, label={[font=\scriptsize]above:UNKNOWN}] at (2.0,0.8) {};
        \node[circle, fill=red!60, minimum size=8pt, inner sep=0pt, label={[font=\scriptsize]below:D-LinOSS}] at (2.2,0.4) {};
        \node[circle, fill=red!60, minimum size=8pt, inner sep=0pt, label={[font=\scriptsize]below:S5}] at (3.0,0.3) {};
        \node[circle, fill=red!60, minimum size=8pt, inner sep=0pt, label={[font=\scriptsize]right:GRU}] at (5.5,0.4) {};
        \draw[dashed, gray] (2.0,0.8) to[out=120,in=290] (1.0,3.9);
        \node[font=\tiny, gray, rotate=70] at (1.3,2.0) {more structure};
    \end{tikzpicture}
    \caption{\textbf{Two-axis evaluation (open-loop, conceptual).} Forecasting accuracy (low rollout MSE) and physical identifiability (high damping $R^2$) are distinct objectives; model/regularizer choices induce trade-offs.}
    \label{fig:two_axis}
\end{figure}

Forecasting accuracy and physical parameter recovery are \emph{distinct} objectives.
A model may achieve good rollouts by using $\Dmat(q)$ as a ``stabilizer'' rather than learning true physical damping.
Conversely, recovering true $\Dmat(q)$ does not guarantee good open-loop rollouts if $V(q)$ or $\Mmat$ have errors.
We therefore evaluate open-loop performance along two axes (Fig.~\ref{fig:two_axis}):

\paragraph{Axis 1: Forecasting stability.}
For angular coordinates we report one-step wrapped-angle MSE and long-horizon open-loop rollout error at horizon $H{=}100$ with burn-in context $K{=}10$.
For Euclidean q-only environments we report standard MSE and rollout MSE at $H{=}100$.
For mixed manifolds (Cart-Pole) we report a mixed-manifold rollout MSE at $H{=}100$ that averages translation MSE and wrapped-angle MSE.
All models are evaluated in open-loop (autoregressive) mode; PHAST additionally supports takeover rollouts (burn-in then integrate; Appendix~\ref{app:arch:qonly})
as a diagnostic of the learned dynamics.

\paragraph{Axis 2: Identifiability.}
When a model exposes an explicit damping field $\Dmat(q)$, we report damping recovery via $R^2$ and MAE.
We also report an energy-consistency diagnostic on open-loop rollouts (energy-budget residual at $H{=}100$),
computed from finite-difference velocity estimates and the benchmark's simulator damping law $D_{\mathrm{env}}(\cdot)$ for expected dissipation (Appendix~\ref{app:metrics}).
Precise metric definitions (including rollout protocol and the discrete energy-budget residual) are provided in Appendix~\ref{app:metrics}.

\paragraph{Regime-specific expectations.}
In KNOWN/PARTIAL regimes, physics-calibrated damping bounds (e.g., $\sum_i\beta_i(q)\le \bbar \approx \Delta d$) can improve identifiability without destabilizing rollouts
(\ifdefined\isarxiv Table~\ref{tab:partial_damping_cap_windy}\else Appendix Table~\ref{tab:partial_damping_cap_windy}\fi; Appendix Table~\ref{tab:windy_cap_sensitivity_quick}).
In UNKNOWN, damping bounds can reveal a forecasting--identifiability trade-off unless additional anchors break gauge freedoms (Appendix Table~\ref{tab:unknown_damping_cap_windy}).

\begin{center}
\fbox{\begin{minipage}{0.96\linewidth}
\small
\textbf{Open-loop error sources \& diagnostics.}
We report rollout error (stability), damping $R^2$ (identifiability), energy-budget residual, and discrete-time passivity violations.
Appendix~\ref{app:metrics:debugging} provides the q-only computation graph and diagnostic interpretations
(autoregressive vs.\ takeover, step-size/substepping, damping bounds, and time-scale ambiguity).
\end{minipage}}
\end{center}

% ---------------------------------------------------------------------------
% Open-loop main results
% ---------------------------------------------------------------------------
\subsubsection{Main open-loop results}
\label{sec:experiments:openloop:results}

\paragraph{Main qualitative results.}
We present qualitative results across three benchmarks that probe complementary aspects of dissipative dynamics learning (Fig.~\ref{fig:gallery_rollouts_phase}).
The \emph{Windy Pendulum} (1~DOF, $\mathbb{S}^1$) is the simplest: constant mass, a single position-dependent damping coefficient $d(\theta)$, and regular oscillatory trajectories.
It tests whether a model can learn basic dissipative Hamiltonian dynamics from positions alone.
The \emph{Windy Double Pendulum} (2~DOF, $\mathbb{T}^2$) escalates difficulty sharply: the dynamics are chaotic, the mass matrix $M(q)$ couples the two links through $\cos(\theta_1{-}\theta_2)$, and each joint has independent wind damping.
Small prediction errors grow exponentially, making this the hardest forecasting challenge in our suite.
The \emph{Windy Cart-Pole} (2~DOF, $\mathbb{R}{\times}\mathbb{S}^1$) represents a different axis of difficulty:
a mixed manifold (translation $\times$ rotation), configuration-dependent inertia $M(q)$,
and two \emph{qualitatively different} damping mechanisms --- constant viscous cart friction and position-dependent angular wind damping on the pole.
Full forecasting tables across all thirteen q-only benchmarks are reported in Appendix~\ref{app:tables:qonly_rollout}.
\ifdefined\isarxiv
Table~\ref{tab:qonly_suite_summary_h100} provides a compact summary across the full benchmark suite
(see also Fig.~\ref{fig:gallery_overview} in Sec.~\ref{sec:intro} for a visual overview).
\fi
Because the Double Pendulum and Cart-Pole have nonseparable Hamiltonians ($M(q) \neq \text{const}$),
we use the implicit-midpoint integrator for all PHAST curves in the multi-environment figures.
Appendix~\ref{app:ablations:nonseparable_mass} further discusses the nonseparable mass variant.

\begin{figure}[t]
    \centering
    \includegraphics[width=0.98\textwidth]{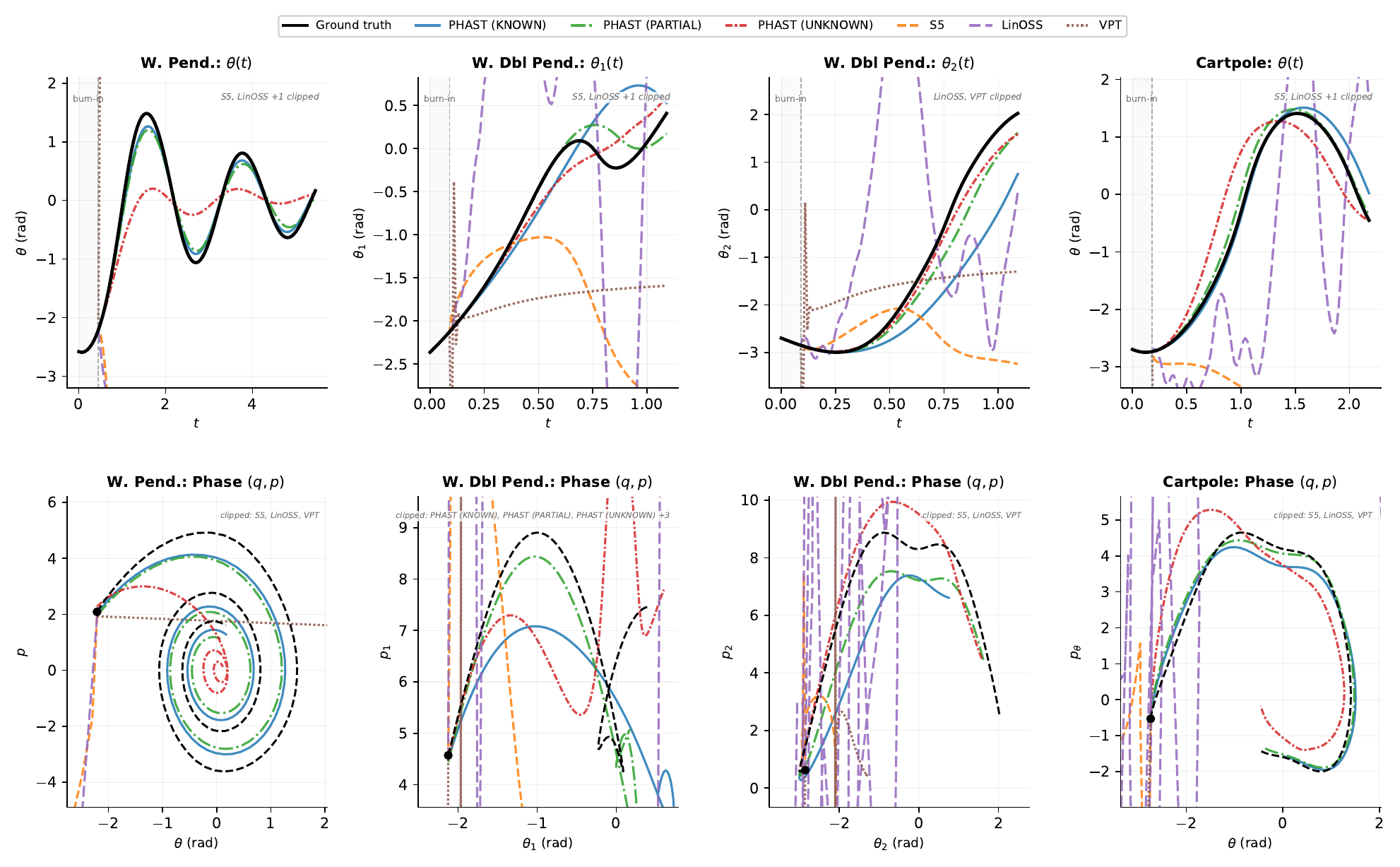}
    \caption{\textbf{Open-loop rollouts and phase-space portraits across three dissipative benchmarks (q-only).}
    \emph{Top row}: a single test trajectory is teacher-forced through a short burn-in window (grey region, vertical dashed line), then predicted open-loop for $H{=}100$ steps.
    \textbf{Column~1} (Windy Pendulum, $\theta$): the simplest system --- a single angle with position-dependent damping.
    PHAST (KNOWN/PARTIAL) tracks the decaying oscillation almost exactly; all three baselines diverge within a few periods.
    \textbf{Columns~2--3} (Windy Double Pendulum, $\theta_1$ and $\theta_2$): a chaotic 2-DOF system with coupled configuration-dependent inertia $M(q)$.
    This is the hardest benchmark: small errors grow exponentially, yet PHAST maintains trajectory coherence on both joints far longer than baselines.
    Showing both angles reveals that the model captures inter-joint coupling, not just marginal statistics.
    \textbf{Column~4} (Windy Cart-Pole, $\theta$): a mixed-manifold system ($\mathbb{R}{\times}\mathbb{S}^1$) with two qualitatively different damping mechanisms (constant cart friction $+$ angular wind damping).
    PHAST (PARTIAL) achieves the tightest tracking here.
    \emph{Bottom row}: canonical phase-space portraits $(\theta, p)$ with $p = M(q)\dot{q}$ during the open-loop segment.
    Momentum is latent in the q-only setting:
    for PHAST, $\dot{q}$ comes from the learned FD+TCN observer (Sec.~\ref{sec:methods});
    for baselines, $\dot{q}$ is approximated by finite differences.
    The closed orbits visible for PHAST confirm that the model has learned physically consistent Hamiltonian structure,
    whereas baseline phase portraits collapse or spiral outward.}
    \label{fig:gallery_rollouts_phase}
\end{figure}

\begin{table}[t]
    \centering
    \caption{\textbf{Q-only open-loop forecasting (Cart-Pole).} Mean $\pm$ std of mixed-manifold rollout MSE at horizon $H{=}100$ over 5 seeds.
    The mixed metric averages translation MSE and wrapped-angle MSE (Appendix~\ref{app:metrics}). Lower is better.}
    \label{tab:qonly_rollout_cartpole_h100}
    \small
    \begin{tabular}{@{}lrr@{}}
        \toprule
        \textbf{Model} & \textbf{Params} & \textbf{Cart-Pole (windy)} \\
        \midrule
        \multicolumn{3}{l}{\textit{PHAST (ours)}} \\
        PHAST (KNOWN)   & 3{,}589  & \best{0.063 \pm 0.019} \\
        PHAST (PARTIAL) & 14{,}283 & $0.083 \pm 0.022$ \\
        PHAST (UNKNOWN) & 14{,}290 & $0.109 \pm 0.022$ \\
        \midrule
        S5 (best baseline) & 17{,}218 & $0.431 \pm 0.077$ \\
        \bottomrule
    \end{tabular}
\end{table}

Table~\ref{tab:qonly_rollout_cartpole_h100} shows that PHAST achieves strong long-horizon stability on Windy Cart-Pole.
\ifdefined\isarxiv
Figure~\ref{fig:gallery_damping_energy} shows the learned damping profiles and energy traces across all three systems,
illustrating the forecasting--identifiability trade-off:
\fi
Table~\ref{tab:windy_identifiability} highlights the two-axis nature of the problem:
in the Windy setting, the KNOWN regime recovers a physically meaningful damping field,
PARTIAL becomes identifiable once we impose a physics-calibrated magnitude bound $\bbar$ on $\Dmat(q)$,
and UNKNOWN still exhibits a forecasting--identifiability trade-off without additional anchors (Appendix~\ref{app:metrics}).
Appendix~\ref{app:ablations} makes this trade-off explicit: Table~\ref{tab:partial_damping_cap_windy} compares bounded vs.\ unbounded damping in the PARTIAL regime,
showing that without the bound the model can achieve accurate rollouts while learning a physically meaningless damping field ($R^2_D \ll 0$).
Takeover rollouts (burn-in then integrate) can improve with better identifiability, but remain sensitive to single-shot momentum inference at the context boundary;
we therefore treat takeover as a diagnostic rather than the primary q-only comparison metric.

\ifdefined\isarxiv
% tables_qonly_suite_summary.tex — Compact suite summary (13 q-only benchmarks)

\begin{table}[t]
    \centering
    \caption{\textbf{Suite summary across thirteen q-only benchmarks.} For each benchmark, we report the best PHAST regime and the best baseline (mean $\pm$ std over 5 seeds). The first eight rows use rollout MSE at horizon $H{=}100$; the last five report next-step MSE (Tables~\ref{tab:qonly_rollout_rlc}--\ref{tab:qonly_rollout_predprey}). Gains are per-row ratios and are not directly comparable across environments.}
    \label{tab:qonly_suite_summary_h100}
    \scriptsize
    \setlength{\tabcolsep}{3pt}
    \begin{tabular}{@{}lrrr@{}}
        \toprule
        \textbf{Benchmark} & \textbf{Best PHAST} & \textbf{Best baseline} & \textbf{Gain} \\
        \midrule
        \multicolumn{4}{l}{\textit{Mechanical systems (rollout MSE at $H{=}100$)}} \\
        Pendulum (cons) & \best{0.680 \pm 0.043} (PARTIAL) & $2.320 \pm 0.224$ (Transformer) & $3.4\times$ \\
        Pendulum (damped) & \best{0.017 \pm 0.005} (KNOWN) & $0.450 \pm 0.241$ (D-LinOSS) & $26.5\times$ \\
        Pendulum (windy) & \best{0.092 \pm 0.014} (PARTIAL) & $0.435 \pm 0.239$ (D-LinOSS) & $4.7\times$ \\
        Cart-Pole (windy) & \best{0.063 \pm 0.019} (KNOWN) & $0.431 \pm 0.077$ (S5) & $6.8\times$ \\
        Oscillator (cons) & \best{0.0010 \pm 0.0002} (KNOWN/PARTIAL) & $1.087 \pm 0.299$ (Transformer) & $1.1\times 10^{3}$ \\
        Oscillator (damped) & \best{0.0011 \pm 0.0003} (PARTIAL) & $0.926 \pm 0.254$ (Transformer) & $8.4\times 10^{2}$ \\
        Double pendulum (cons) & \best{0.402 \pm 0.047} (PARTIAL) & $0.618 \pm 0.028$ (S5) & $1.5\times$ \\
        Double pendulum (damped) & \best{0.320 \pm 0.032} (PARTIAL) & $0.630 \pm 0.031$ (S5) & $2.0\times$ \\
        \midrule
        \multicolumn{4}{l}{\textit{Non-mechanical systems (next-step MSE)}} \\
        RLC circuit (KNOWN) & \best{2.63 \times 10^{-5}} (UNKNOWN) & $4.81 \times 10^{-4}$ (Transformer) & $18\times$ \\
        LJ-3 cluster (KNOWN) & \best{4.59 \times 10^{-10}} (PARTIAL) & $2.05 \times 10^{-4}$ (S5) & $4.5\times 10^{5}$ \\
        Heat exchange (KNOWN) & \best{2.42 \times 10^{-6}} (KNOWN) & $4.46 \times 10^{-4}$ (LinOSS) & $1.8\times 10^{2}$ \\
        N-body gravity (PARTIAL) & \best{4.27 \times 10^{-8}} (PARTIAL) & $1.83 \times 10^{-3}$ (Transformer) & $4.3\times 10^{4}$ \\
        Predator--prey (UNKNOWN) & \best{0.0199} (UNKNOWN) & $0.179$ (Transformer) & $9.0\times$ \\
        \bottomrule
    \end{tabular}
\end{table}

% tables_windy_identifiability.tex — Windy Pendulum identifiability + energy consistency table

\begin{table}[t]
    \centering
    \caption{\textbf{Windy Pendulum (q-only) identifiability and energy consistency.} Mean $\pm$ std over 5 seeds (same setup as Table~\ref{tab:qonly_rollout_pendulum_h100}). Baselines do not expose a damping field, so damping recovery metrics are not applicable.}
    \label{tab:windy_identifiability}
    \scriptsize
    \begin{tabular}{@{}lrrrrr@{}}
        \toprule
        \textbf{Model} & \textbf{$\mathrm{WrapMSE}_\theta^{\mathrm{roll}}(H{=}100)$} $\downarrow$ & \textbf{$R^2_D$} $\uparrow$ & \textbf{$\mathrm{MAE}_D$} $\downarrow$ & \textbf{$\mathrm{EbudRes}^{\mathrm{roll}}(H{=}100)$} $\downarrow$ & \textbf{Params} \\
        \midrule
        PHAST (KNOWN) & $0.106 \pm 0.020$ & \best{0.996 \pm 0.004} & \best{0.007 \pm 0.004} & $1.56 \pm 0.02$ & 3{,}364 \\
        PHAST (PARTIAL) & \best{0.092 \pm 0.014} & $0.654 \pm 0.077$ & $0.063 \pm 0.011$ & \best{1.50 \pm 0.04} & 13{,}736 \\
        PHAST (UNKNOWN) & $0.298 \pm 0.048$ & $-96.546 \pm 15.897$ & $1.343 \pm 0.115$ & $2.58 \pm 0.19$ & 13{,}738 \\
        \midrule
        GRU & $1.796 \pm 0.625$ & --- & --- & $12.34 \pm 6.58$ & 37{,}889 \\
        S5 & $0.600 \pm 0.047$ & --- & --- & $11.71 \pm 2.91$ & 17{,}089 \\
        LinOSS & $1.458 \pm 0.324$ & --- & --- & $56.14 \pm 37.41$ & 17{,}089 \\
        D-LinOSS & $0.435 \pm 0.239$ & --- & --- & $6.86 \pm 1.14$ & 33{,}793 \\
        Transformer & $0.824 \pm 0.134$ & --- & --- & $7.71 \pm 0.63$ & 100{,}161 \\
        VPT & $2.218 \pm 0.135$ & --- & --- & $40.48 \pm 8.04$ & 16{,}833 \\
        \bottomrule
    \end{tabular}
\end{table}

% tables_partial_damping_cap_windy.tex — PARTIAL damping cap ablation (Windy Pendulum, q-only)

\begin{table}[t]
    \centering
    \caption{\textbf{PARTIAL regime: effect of bounding total damping strength (Windy Pendulum, q-only).} Mean $\pm$ std over 5 seeds on CPU (same data protocol as Sec.~\ref{sec:experiments:setup}). Without the bound, damping becomes an ``error sink'' and is non-identifiable ($R^2_D\ll 0$) even when rollouts are accurate; bounding $\sum_{i=1}^{r}\beta_i(q)$ recovers a physically meaningful damping field.}
    \label{tab:partial_damping_cap_windy}
    \scriptsize
    \setlength{\tabcolsep}{3pt}
    \begin{tabular}{@{}lrrrr@{}}
        \toprule
        \textbf{Damping cap} &
        \textbf{WrapMSE$_\theta$(100)} $\downarrow$ &
        \textbf{$R^2_D$} $\uparrow$ &
        \textbf{MAE$_D$} $\downarrow$ &
        \textbf{EbudRes(100)} $\downarrow$ \\
        \midrule
        Unbounded &
        \best{0.059 \pm 0.007} &
        $-90.15 \pm 6.14$ &
        $1.326 \pm 0.044$ &
        \best{1.50 \pm 0.03} \\
        $\sum_i\beta_i(q)\le 0.5$ &
        $0.071 \pm 0.004$ &
        \best{0.703 \pm 0.072} &
        \best{0.061 \pm 0.010} &
        $1.58 \pm 0.04$ \\
        \bottomrule
    \end{tabular}
\end{table}

\begin{figure}[t]
    \centering
    \includegraphics[width=0.98\textwidth]{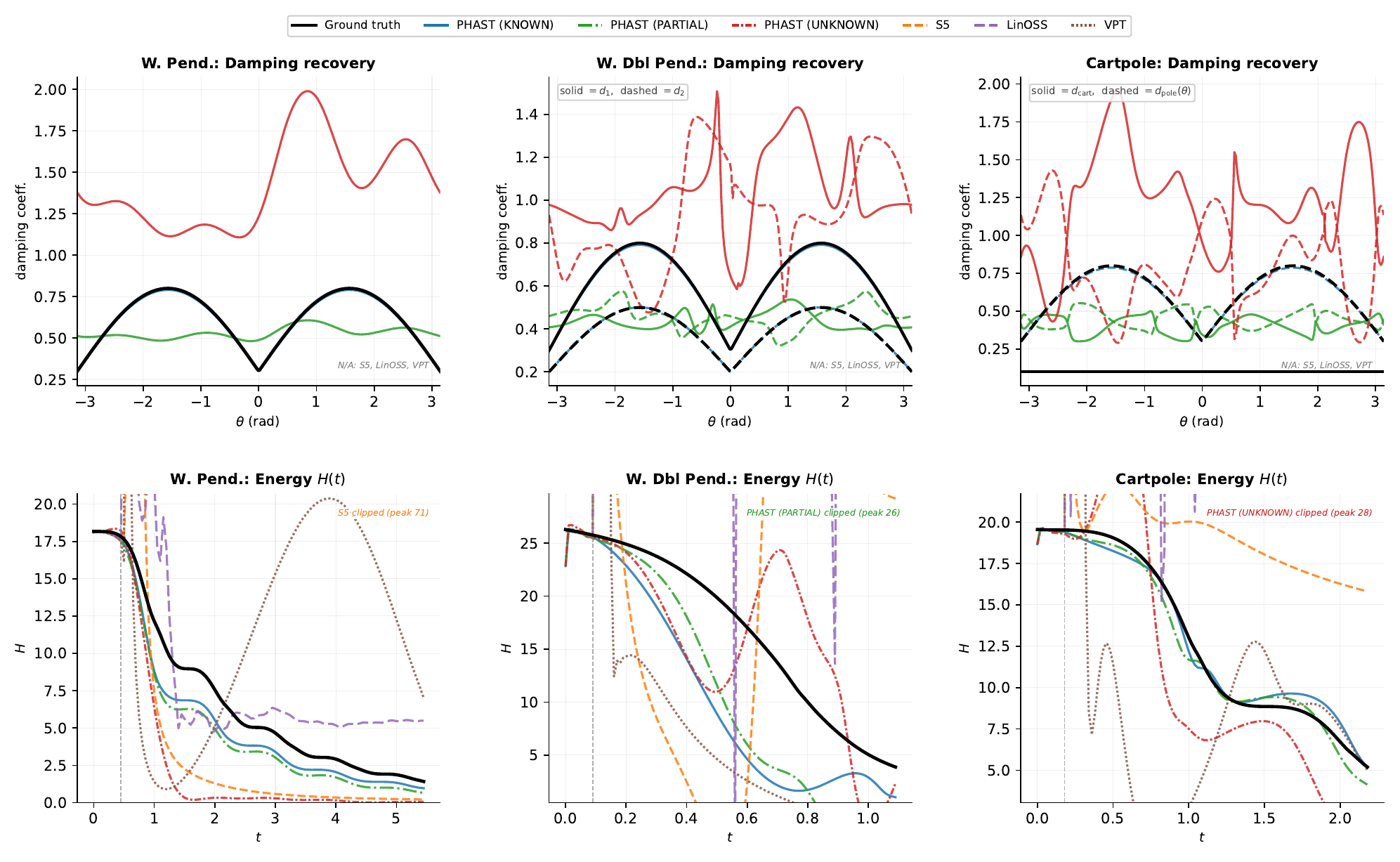}
    \caption{\textbf{Damping identifiability and energy consistency across environments (q-only).}
    \emph{Top row}: learned damping field $d(\theta)$ vs.\ ground truth (black).
    \textbf{Pendulum}: PHAST (KNOWN) recovers the sinusoidal profile $d(\theta){=}d_0{+}\Delta d\,|\!\sin\theta|$ near-exactly ($R^2_D{\approx}1$);
    PARTIAL captures the shape but with a magnitude offset, illustrating the forecasting--identifiability trade-off.
    \textbf{Double Pendulum}: solid and dashed curves show per-joint damping ($d_1$, $d_2$).
    Despite chaotic dynamics and a coupled $M(q)$, KNOWN recovers both joints' damping profiles to high accuracy.
    \textbf{Cart-Pole}: solid and dashed curves distinguish constant viscous cart friction ($d_{\mathrm{c}}$, flat) from angular wind damping on the pole ($d(\theta)$, sinusoidal) --- two qualitatively different dissipation mechanisms that PHAST disentangles from q-only data.
    UNKNOWN shows poor identifiability across all systems ($R^2_D \ll 0$; Table~\ref{tab:qonly_suite_summary_h100}).
    Baselines do not expose explicit damping fields (marked N/A).
    \emph{Bottom row}: total energy $H(t)$ during open-loop rollouts.
    In a dissipative system, energy must monotonically decrease; this is a necessary physical consistency check.
    PHAST (PARTIAL) most closely tracks the ground-truth energy decay across all three systems.
    Baselines either diverge (energy blowup) or collapse to zero, both indicating physically inconsistent dynamics.
    For PHAST, energy is computed using the observer's learned velocity; for baselines, via finite differences.}
    \label{fig:gallery_damping_energy}
\end{figure}
\fi

\paragraph{Additional studies (open-loop).}
Appendix~\ref{app:ablations} reports additional open-loop ablations (integrator choices, observer capacity, and damping constraints, including $\bar\beta$ sweeps) and a full-state comparison.

% ---------------------------------------------------------------------------
% 4.2 Closed-loop control experiments: Energy--Casimir
% ---------------------------------------------------------------------------
\subsection{Closed-loop control (Energy--Casimir stabilization)}
\label{sec:experiments:closedloop}

\paragraph{Demarcation from open-loop.}
Open-loop forecasting evaluates autonomous rollouts ($u=0$) and physical identifiability of learned fields.
Closed-loop evaluation instead probes whether the learned pH structure supports passivity-based stabilization under feedback.
In q-only settings, closed-loop performance depends critically on the quality of the online port/velocity estimate used for damping injection.

\paragraph{Scope (closed-loop).}
Throughout this subsection we keep the controller structure and gains fixed and vary only the port/velocity estimate $\hat y$ used for damping injection and (optionally) for the controller-state update.
Unless stated otherwise, the plant is the \emph{true} pendulum integrated by RK4.

\paragraph{Setup and controller.}
We evaluate stabilization on a single-degree-of-freedom pendulum plant with state $x=(q,p)$ and q-only measurements $q_t^{\mathrm{meas}}=q_t+\epsilon_t$.
The Energy--Casimir-style controller maintains an internal state $\xi$ and uses a port estimate $\hat y_t \approx y_p(t)$ (which reduces to $\hat y_t=\hat{\dot q}_t$ for the pendulum):
\begin{align}
    u_t &= -k_c(\xi_t - q^\star)\;-\;d_{\mathrm{inj}}\,\hat y_t, \\
    \xi_{t+1} &= \xi_t + \dt\Big(\hat y_t + k_\xi\,(q_t^{\mathrm{meas}}-\xi_t)\Big),
\end{align}
where the $k_\xi$ term is a predictor--corrector measurement correction that reduces $\xi$ drift under noisy q-only feedback.
When $k_\xi=0$ and $\hat y=y_p$, the update reduces to a discrete approximation of $\dot\xi=y_p$.

\paragraph{Closed-loop protocol and metrics.}
We evaluate $N$ rollouts from random initial conditions over horizon $T_{\mathrm{ctl}}$ and report:
(i) success rate (fraction of trials that converge to the target),
(ii) final wrapped-angle error,
(iii) control effort $\sum_{t=0}^{T_{\mathrm{ctl}}-1}\|u_t\|_2^2$,
and (iv) port-estimation error $\mathbb{E}_t\|\hat y_t-y_p(t)\|$ when $y_p$ is available in simulation.
Precise thresholds, horizons, and noise settings follow Appendix~\ref{app:casimir}.

\paragraph{Port-estimation variants.}
We compare five Energy--Casimir variants that share the same shaping gain $k_c$, damping-injection gain $d_{\mathrm{inj}}$, and drift-correction gain $k_\xi$,
differing only in how $\hat y_t$ is obtained online:
(i) oracle/full-state,
(ii) finite differences,
(iii) fixed-lag MAP smoothing (label-free),
(iv) FD+TCN observer trained offline with noise augmentation,
and (v) PHAST-trained observer trained only through next-step prediction.
Appendix~\ref{app:casimir} reports additional control ablations (noise mismatch, near-stable regime, and model-based velocity from learned $\hat H$).

\paragraph{Closed-loop results (summary).}
\ifdefined\isarxiv
Sec.~\ref{app:casimir} shows that PHAST models support passivity-based stabilization:
using the learned PHAST Hamiltonian to compute the port output $\hat{y}=\partial\hat{H}/\partial p$ achieves 100\% stabilization success with \emph{lower} control effort than oracle velocities (245 vs.\ 263),
indicating the learned energy landscape is accurate enough for feedback.
Under q-only sensing, the bottleneck shifts to port/velocity estimation quality rather than model accuracy;
noise-aware observers reduce control effort by ${\sim}9\times$ compared to finite differences.
\else
Appendix~\ref{app:casimir} shows that PHAST models support passivity-based stabilization:
using the learned PHAST Hamiltonian to compute the port output $\hat{y}=\partial\hat{H}/\partial p$ achieves 100\% stabilization success with \emph{lower} control effort than oracle velocities (245 vs.\ 263).
Under q-only sensing, the bottleneck shifts to port/velocity estimation quality; noise-aware observers reduce control effort by ${\sim}9\times$ compared to finite differences.
\fi

% \paragraph{Additional studies (open-loop).}
% Appendix~\ref{app:ablations} reports additional open-loop ablations (integrator choices, observer capacity, and damping constraints, including $\bar\beta$ sweeps) and a full-state comparison.

% ---------------------------------------------------------------------------
% Extended arXiv sections (promoted from the appendix)
% ---------------------------------------------------------------------------
% arxiv_extended.tex — Extra sections promoted to the main paper for arXiv
%
% Included from main_arxiv.tex before the conclusion.

% appendix_casimir.tex — Energy-Casimir control study (shared between ICML appendix and arXiv main)

% ===========================================================================
% H. Casimir Control (Separate Task)
% ===========================================================================
\section{Energy-Casimir Control}
\label{app:casimir}

Energy--Casimir control is a passivity-based approach for stabilizing port-Hamiltonian plants by interconnecting the plant with a dynamic controller and shaping a closed-loop storage function (Hamiltonian plus Casimir).
Although PHAST is evaluated primarily as an \emph{open-loop forecaster} in the main paper, port-Hamiltonian models are often used in \emph{feedback} settings.
We include this study to (i) show a minimal closed-loop wiring compatible with the port variables in Sec.~\ref{sec:methods:ph}, and (ii) isolate a practical bottleneck in q-only control: the controller requires the velocity-like port output $y_p=y^{\mathrm{port}}$ (for torque actuation, $y_p=\dot q$), which must be estimated online from noisy position measurements.

\paragraph{Closed-loop scope.}
We evaluate Energy--Casimir stabilization under \emph{q-only} feedback.
Across all control experiments we keep the controller structure and gains
$(k_c,d_{\mathrm{inj}},k_\xi)$ fixed and vary only the online port estimate
$\hat y \approx y_p$ used for damping injection (and, when stated, in the $\xi$ update).
Unless stated otherwise, the plant is the \emph{true} pendulum integrated by RK4.

\paragraph{What is a Casimir function?}
A Casimir function $\mathcal{C}(z)$ is a \emph{structural invariant} of a
port-Hamiltonian system: it satisfies $\nabla\mathcal{C}\in\ker J_{\mathrm{cl}}^\top\cap\ker R_{\mathrm{cl}}$,
so it is conserved by the interconnection structure regardless of the energy
function.
For the mechanical plant--controller pair, $\mathcal{C}(q,\xi)=q-\xi$ acts as
a ``configuration lock'': the interconnection forces the controller state $\xi$
to track the plant position $q$ at all times (Sec.~\ref{sec:casimir_forced_mode}).

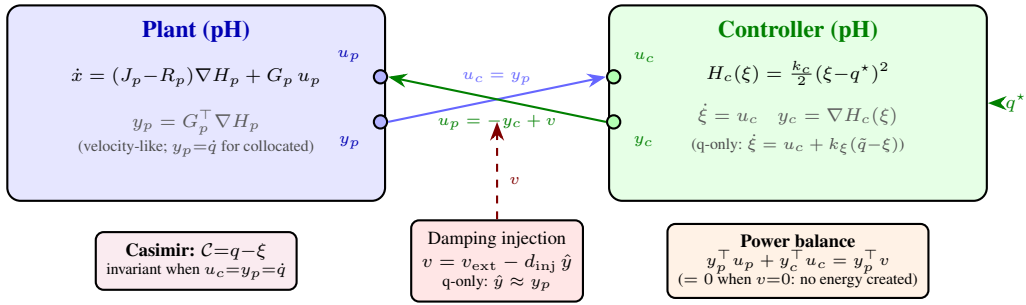
\begin{figure}[H]
\centering
\begin{tikzpicture}[
    phbox/.style={rectangle, draw, thick, rounded corners=5pt, align=center,
                  minimum height=2.6cm, text width=4.6cm, inner sep=6pt},
    port/.style={circle, draw, thick, fill=white, minimum size=5pt, inner sep=0pt},
    arrow/.style={-{Stealth[length=2.5mm]}, thick},
    annot/.style={rectangle, draw, thick, rounded corners=3pt, inner sep=4pt,
                  align=center, font=\scriptsize},
]

% =========================
% PLANT (left)
% =========================
\node[phbox, fill=blue!10] (plant) at (-4.0, 0) {};
\node[font=\small\bfseries, blue!70!black] at (-4.0, 0.95) {Plant (pH)};
\node[font=\scriptsize, align=center] at (-4.0, 0.35)
  {$\dot{x} = (\Jmat_p{-}\Rmat_p)\nabla\Ham_p + G_p\,u_p$};
\node[font=\scriptsize, align=center, gray!70!black] at (-4.0, -0.25)
  {$y_p = G_p^\top\nabla\Ham_p$};
\node[font=\tiny, gray!60!black] at (-4.0, -0.60)
  {(velocity-like; $y_p{=}\dot q$ for collocated)};

% Plant port dots (on box edge)
\node[port, fill=blue!30] (pu) at (-1.55, 0.35) {};
\node[port, fill=blue!30] (py) at (-1.55, -0.25) {};
\node[font=\tiny, blue!70!black, above left=0pt and 2pt of pu] {$u_p$};
\node[font=\tiny, blue!70!black, below left=0pt and 2pt of py] {$y_p$};

% =========================
% CONTROLLER (right)
% =========================
\node[phbox, fill=green!10] (ctrl) at (4.0, 0) {};
\node[font=\small\bfseries, green!50!black] at (4.0, 0.95) {Controller (pH)};
\node[font=\scriptsize, align=center] at (4.0, 0.40)
  {$H_c(\xi) = \tfrac{k_c}{2}(\xi{-}q^\star)^2$};
\node[font=\scriptsize, align=center, gray!70!black] at (4.0, -0.15)
  {$\dot\xi = u_c$\quad $y_c = \nabla H_c(\xi)$};
\node[font=\tiny, gray!60!black] at (4.0, -0.55)
  {(q-only: $\dot\xi = u_c + k_\xi(\tilde q{-}\xi)$)};

% Controller port dots (on box edge)
\node[port, fill=green!30] (cu) at (1.55, 0.35) {};
\node[port, fill=green!30] (cy) at (1.55, -0.25) {};
\node[font=\tiny, green!50!black, above right=0pt and 2pt of cu] {$u_c$};
\node[font=\tiny, green!50!black, below right=0pt and 2pt of cy] {$y_c$};

% Target input
\node[font=\scriptsize, green!50!black] at (6.9, 0) {$q^\star$};
\draw[arrow, green!50!black] (6.7, 0) -- (ctrl.east);

% =========================
% POWER-PRESERVING INTERCONNECTION (between boxes)
% =========================
% Top arrow: y_p --> u_c  (plant output feeds controller input)
\draw[arrow, blue!60, thick]
  (py.east) -- node[above=2pt, font=\tiny, pos=0.5] {$u_c = y_p$} (cu.west);
% Bottom arrow: y_c --> u_p  (controller output feeds plant input)
\draw[arrow, green!50!black, thick]
  (cy.west) -- node[below=2pt, font=\tiny, pos=0.5] {$u_p = -y_c + v$} (pu.east);

% =========================
% DAMPING INJECTION (below center)
% =========================
\node[annot, fill=red!10] (damping) at (0, -2.1)
  {Damping injection\\[1pt]
   $v = v_{\mathrm{ext}} - d_{\mathrm{inj}}\,\hat y$\\[-1pt]
   {\tiny q-only: $\hat y \approx y_p$}};
\draw[arrow, red!50!black, dashed] (damping.north) -- node[right=1pt, font=\tiny, red!50!black, pos=0.4] {$v$} (0, -0.25);

% =========================
% ANNOTATIONS (below)
% =========================
% Casimir annotation
\node[annot, fill=purple!8] (casimir) at (-4.0, -2.1)
  {\textbf{Casimir:} $\mathcal{C}{=}q{-}\xi$\\[-1pt]
   {\tiny invariant when $u_c{=}y_p{=}\dot q$}};

% Power balance annotation
\node[annot, fill=orange!10] (power) at (4.0, -2.1)
  {\textbf{Power balance}\\[-1pt]
   {\tiny $y_p^\top u_p + y_c^\top u_c = y_p^\top v$}\\[-1pt]
   {\tiny ($=0$ when $v{=}0$: no energy created)}};

\end{tikzpicture}
\caption{\textbf{Energy--Casimir control as port-Hamiltonian interconnection.}
The plant (blue, left) and controller (green, right) are pH systems coupled
through power ports: $u_c=y_p$ feeds plant velocity to the controller,
$u_p=-y_c+v$ returns the shaped restoring force plus an auxiliary channel $v$.
When $v{=}0$ the cross-power cancels ($y_p^\top u_p + y_c^\top u_c = 0$),
and the Casimir $\mathcal{C}=q-\xi$ is invariant.
Damping injection $v=-d_{\mathrm{inj}}\hat y$ (red, dashed) ensures
$\dot H_{\mathrm{cl}}\le 0$; in q-only settings $\hat y\approx y_p$ is
estimated, so estimation error can break exact power balance.}
\label{fig:casimir_control_structure}
\end{figure}

Having established the theoretical framework in Sec.~\ref{sec:casimir_forced_mode}, we now evaluate it experimentally.
All experiments use the forced-dynamics mode of the Energy--Casimir controller (Eqs.~\ref{eq:forced_u_disc}--\ref{eq:forced_xi_disc}) on a single-degree-of-freedom pendulum.
For the pendulum with collocated torque actuation, $y_p=\dot q$, so the port estimate reduces to $\hat y=\hat{\dot q}$.

\paragraph{MAP smoother observer (label-free).}
For the MAP smoother baseline, we estimate $\hat{\dot{q}}_t$ by MAP smoothing over a fixed-lag window of length $w$ (with $w\ge 2$) of q-only measurements $q_{t-w+1:t}^{\mathrm{meas}}$.
We unwrap the angles within the window (integrating wrapped increments) and solve the convex quadratic problem
\begin{align}
    q^{\mathrm{smooth}} &\in \arg\min_{q}\;
    \sum_{s=0}^{w-1}\frac{1}{\sigma^2}\,\big\|q_s-q^{\mathrm{meas}}_s\big\|_2^2
    \notag\\
    &\quad+\;
    \sum_{s=0}^{w-3}\frac{1}{\dt^4\sigma_a^2}\,\big\|q_{s+2}-2q_{s+1}+q_s\big\|_2^2,
\end{align}
which corresponds to a Gaussian measurement model and a Gaussian prior on discrete acceleration.
We then take a backward difference on the smoothed sequence:
\begin{align}
    \hat{\dot{q}}_t = \frac{q^{\mathrm{smooth}}_{w-1}-q^{\mathrm{smooth}}_{w-2}}{\dt}.
\end{align}
This baseline requires only $q^{\mathrm{meas}}$ and the assumed noise scales $(\sigma,\sigma_a)$; it does not use $\dot{q}$ supervision.
Unless otherwise stated, we fix $\sigma_a=10.0$ for all MAP results.

\paragraph{Main results.}
Table~\ref{tab:casimir_main} summarizes the overall results under (i) no measurement noise and (ii) moderate q-only measurement noise $\sigma=0.01$.
Under noise, finite differences incur large velocity error and increased control effort; the learned FD+TCN observer reduces velocity error by $\sim$6--7$\times$ while preserving convergence.
The PHAST-trained observer also preserves convergence at $\sigma=0.01$, but its velocity estimate remains noisy (closer to finite differences), which increases damping-injection effort near the target.
A simple MAP smoother provides a competitive label-free baseline, reducing velocity error and final error without requiring $\dot{q}$ supervision.
This emphasizes that closed-loop performance hinges on the quality of the port output estimate $y=\dot{q}$.

\begin{table}[t]
    \centering
    \caption{\textbf{Energy--Casimir control on the true pendulum (q-only sensing).} Mean over $100$ trials (5 regimes $\times$ 20). All methods use the same predictor--corrector $\xi$ update ($k_\xi=5$).}
    \label{tab:casimir_main}
    \small
    \begin{tabular}{@{}rllrrrr@{}}
        \toprule
        $\sigma$ & Method & Velocity estimate & Success $\uparrow$ & Final error $\downarrow$ & Effort $\downarrow$ & $|\hat{\dot{q}}-\dot{q}|$ $\downarrow$ \\
        \midrule
        0.00 & Oracle & oracle & 1.00 & $1.31\times 10^{-4}$ & 262.9 & 0.000 \\
        0.00 & Finite differences & FD & 1.00 & $1.21\times 10^{-4}$ & 268.6 & 0.007 \\
        0.00 & MAP smoother & MAP smoother & 1.00 & $1.16\times 10^{-4}$ & 268.6 & 0.007 \\
        0.00 & FD+TCN observer & FD+TCN & 1.00 & $1.03\times 10^{-3}$ & 270.7 & 0.008 \\
        0.00 & PHAST-trained observer & PHAST observer & 1.00 & $5.71\times 10^{-3}$ & 272.4 & 0.019 \\
        \midrule
        0.01 & Oracle & oracle & 1.00 & $5.49\times 10^{-3}$ & 262.9 & 0.000 \\
        0.01 & Finite differences & FD & 1.00 & $1.67\times 10^{-2}$ & 342.2 & 1.117 \\
        0.01 & MAP smoother & MAP smoother & 1.00 & $1.17\times 10^{-2}$ & 324.8 & 0.133 \\
        0.01 & FD+TCN observer & FD+TCN & 1.00 & $2.39\times 10^{-2}$ & 335.5 & 0.172 \\
        0.01 & PHAST-trained observer & PHAST observer & 1.00 & $1.93\times 10^{-2}$ & 327.7 & 0.979 \\
        \bottomrule
    \end{tabular}
\end{table}

\paragraph{Model-based velocity from a learned PHAST model (full-state).}
To isolate the effect of \emph{model mismatch} from partial observability, we also evaluate a controller that uses a learned PHAST Hamiltonian to form the port output estimate $\hat{y}=\partial \hat{H}/\partial p$ from the \emph{observed} state $(q,p)$.
Table~\ref{tab:casimir_phast_fullstate} shows that this model-based velocity estimate is accurate enough to preserve stability, supporting the use of PHAST-learned dynamics for control when full state is available.

\begin{table}[t]
    \centering
    \caption{\textbf{Energy--Casimir with PHAST-based velocity (q,p observed).} Mean over $100$ trials (5 regimes $\times$ 20). Both methods use the same predictor--corrector $\xi$ update ($k_\xi=5$).}
    \label{tab:casimir_phast_fullstate}
    \small
    \begin{tabular}{@{}rllrrrr@{}}
        \toprule
        $\sigma$ & Method & Velocity estimate & Success $\uparrow$ & Final error $\downarrow$ & Effort $\downarrow$ & $|\hat{\dot{q}}-\dot{q}|$ $\downarrow$ \\
        \midrule
        0.00 & Oracle & oracle & 1.00 & $1.31\times 10^{-4}$ & 262.9 & 0.000 \\
        0.00 & PHAST model-based & $\partial \hat{H}/\partial p$ & 1.00 & $5.46\times 10^{-5}$ & 245.0 & 0.011 \\
        \midrule
        0.01 & Oracle & oracle & 1.00 & $5.49\times 10^{-3}$ & 262.9 & 0.000 \\
        0.01 & PHAST model-based & $\partial \hat{H}/\partial p$ & 1.00 & $5.32\times 10^{-3}$ & 245.1 & 0.011 \\
        \bottomrule
    \end{tabular}
\end{table}

\paragraph{Near-target efficiency under noisy q-only feedback.}
The strongest effect of learning a noise-aware observer is visible near the target, where finite differences can inject high-frequency damping torques.
Table~\ref{tab:casimir_near_stable} reports the near-stable regime ($20$ trials) under $\sigma=0.01$.

\begin{table}[t]
    \centering
    \caption{\textbf{Near-stable regime ($\sigma=0.01$).} Learning a noise-aware q-only observer reduces control effort near the target by $\sim$9$\times$ vs. finite differences.}
    \label{tab:casimir_near_stable}
    \small
\begin{tabular}{@{}lrrr@{}}
        \toprule
        Method & Final error $\downarrow$ & Effort $\downarrow$ & $|\hat{\dot{q}}-\dot{q}|$ $\downarrow$ \\
        \midrule
        Oracle & 0.0041 & 6.3 & 0.000 \\
        Finite differences & 0.0170 & 80.1 & 1.117 \\
        MAP smoother & 0.0133 & 10.3 & 0.095 \\
        FD+TCN observer & 0.0137 & 8.6 & 0.130 \\
        PHAST-trained observer & 0.0188 & 54.9 & 0.911 \\
        \bottomrule
    \end{tabular}
\end{table}

\paragraph{Observer-noise match ablation (limiting factor).}
To isolate the key failure mode under noisy q-only feedback, we train the FD+TCN observer under different noise distributions and evaluate at $\sigma=0.01$.
Table~\ref{tab:casimir_noise_match} shows that \emph{noise-matching} (or a narrow range anchored to deployment noise) is the dominant factor: training the observer at $\sigma=0$ and deploying at $\sigma=0.01$ recovers neither accurate velocities nor low-effort behavior.

\begin{table}[t]
    \centering
    \caption{\textbf{Observer training noise vs. evaluation noise ($\sigma_{\text{eval}}=0.01$), near-stable regime.} Q-only FD+TCN observer used inside the Energy--Casimir controller.}
    \label{tab:casimir_noise_match}
    \small
    \begin{tabular}{@{}lrrr@{}}
        \toprule
        Observer train noise & Final error $\downarrow$ & Effort $\downarrow$ & $|\hat{\dot{q}}-\dot{q}|$ $\downarrow$ \\
        \midrule
        $\sigma_{\text{train}}=0.00$ (mismatch) & 0.0168 & 82.4 & 1.134 \\
        $\sigma_{\text{train}}=0.01$ (matched) & 0.0120 & 8.5 & 0.132 \\
        $\sigma_{\text{train}}\sim\mathcal{U}[0,0.01]$ (anchored range) & 0.0137 & 8.6 & 0.130 \\
        \bottomrule
    \end{tabular}
\end{table}

\paragraph{Takeaway.}
These experiments support a simple conclusion: in the q-only setting, closed-loop Energy--Casimir behavior is primarily limited by the \emph{velocity estimate} used to form the port output $y=\dot{q}$.
Noise-aware observer training is therefore a prerequisite for meaningful closed-loop validation under partial observability.

\paragraph{High-noise stress test ($\sigma=0.05$).}
We additionally stress-test q-only feedback at higher measurement noise ($\sigma=0.05$).
The oracle-velocity controller remains stable at 100\% success, indicating that the controller structure itself is not the limiting factor.
In contrast, q-only finite differences fail completely, while a noise-conditioned FD+TCN observer recovers partial stability.
This setting remains challenging and is best viewed as a supplementary robustness probe rather than a main result.

\begin{table}[t]
    \centering
    \caption{\textbf{High-noise stress test ($\sigma=0.05$).} Mean over $100$ trials. For the FD+TCN controller we use a noise-conditioned observer trained with $\sigma_{\text{train}}\sim\mathcal{U}[0,0.05]$.}
    \label{tab:casimir_sigma005}
    \scriptsize
    \setlength{\tabcolsep}{3pt}
    \begin{tabular}{@{}llrrrr@{}}
        \toprule
        Method & Velocity estimate & Success $\uparrow$ & Final error $\downarrow$ & Effort $\downarrow$ & $|\hat{\dot{q}}-\dot{q}|$ $\downarrow$ \\
        \midrule
        Oracle & oracle & 1.00 & $2.75\times 10^{-2}$ & 263.6 & 0.000 \\
        Finite differences & FD & 0.00 & $8.36\times 10^{-2}$ & 2109.0 & 5.582 \\
        FD+TCN observer & FD+TCN (noise) & 0.26 & $7.70\times 10^{-2}$ & 396.6 & 0.659 \\
        \bottomrule
    \end{tabular}
\end{table}

\begin{figure}[H]
    \centering
    \includegraphics[width=0.95\linewidth]{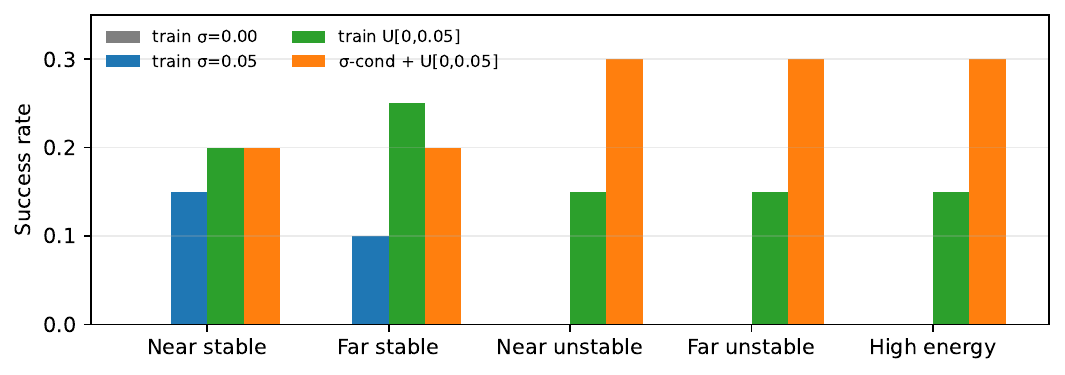}
    \caption{\textbf{Success rate breakdown at $\sigma=0.05$ for q-only FD+TCN control.} Each bar reports success over 20 trials for the given initial-condition regime, varying only the observer training noise distribution.}
    \label{fig:casimir_sigma005_success}
\end{figure}

% ---------------------------------------------------------------------------
% 5. Conclusion
% ---------------------------------------------------------------------------
\section{Conclusion}
\label{sec:conclusion}

We presented PHAST, a port-Hamiltonian framework for learning dissipative dynamics that unifies three knowledge regimes---KNOWN, PARTIAL, and UNKNOWN---within a single architecture.
The key technical contributions are low-rank parameterizations that guarantee $\Dmat(q) \succeq 0$ and $\Mmat(q) \succ 0$ by construction (Sec.~\ref{sec:methods:householder}), and optional damping-strength constraints that can prevent learned dissipation from absorbing model mismatch.

Our main empirical finding is that \textbf{forecasting and identifiability are distinct objectives} that require explicit two-axis evaluation.
Without damping bounds, models can achieve reasonable rollouts while learning physically meaningless parameters ($R^2 < 0$).
On Windy Pendulum (q-only), PHAST (PARTIAL) improves 100-step rollout $\theta$-wrap MSE from 0.435 (best baseline) to 0.092, while PHAST (KNOWN) recovers the true damping with $R^2=0.996$.
This supports the broader takeaway that, on these q-only benchmarks, structural priors can matter more than additional capacity for long-horizon stability.

\paragraph{Limitations.}
Our evaluation is primarily in the q-only setting, so performance depends on inferring a momentum-like latent from a short context window and can degrade under severe sensor noise or partial observability; scaling PHAST to high-DOF robotics and continuum systems remains open.
In PARTIAL/UNKNOWN regimes, physical recovery and discrete-time stability are not guaranteed without calibrated anchors (e.g., damping bounds) and appropriate step sizes (Appendix~\ref{app:ablations}; Appendix~\ref{app:math}).

\paragraph{Future work.}
Promising directions include: (i) extending PHAST to more complex dynamical systems such as multi-body and continuum systems; (ii) learning damping bounds from data rather than requiring physics calibration; (iii) integrating with online estimation and closed-loop planning (e.g., model-predictive control); (iv) extending our Energy--Casimir results (Sec.~\ref{app:casimir}) to stabilization at non-natural equilibria and multi-DOF systems; and (v) test-time adaptation of physics components on the burn-in window, leveraging PHAST's modular regime structure for instance-specific tuning of $(V, \Mmat, \Dmat)$ via meta-learned initializations.

% ---------------------------------------------------------------------------
% Acknowledgements
% ---------------------------------------------------------------------------
\section*{Acknowledgements}

This research was supported in part by the Peter O'Donnell Foundation, the Jim Holland--Backcountry Foundation, and in part by a grant from the Army Research Office accomplished under Cooperative Agreement Number W911NF-19-2-0333.

% ---------------------------------------------------------------------------
% Impact Statement (kept for consistency with ICML style)
% ---------------------------------------------------------------------------
\section*{Impact Statement}

This paper presents work whose goal is to advance the field of Machine Learning by developing physics-informed neural architectures that respect fundamental physical principles such as energy dissipation. Our framework enables more accurate and interpretable modeling of physical systems---including mechanical, electrical, molecular, thermal, gravitational, and ecological dynamics---with potential applications in robotics, simulation, scientific computing, and beyond. There are many potential societal consequences of our work, none which we feel must be specifically highlighted here.

% ---------------------------------------------------------------------------
% References
% ---------------------------------------------------------------------------
\bibliography{references}
\bibliographystyle{icml2026}

% ---------------------------------------------------------------------------
% Appendix (still included; promoted sections are conditionally omitted)
% ---------------------------------------------------------------------------
\newpage
\appendix
\onecolumn
% appendix.tex — Supplementary material for PHAST
% Included from main.tex via \input{appendix} after \appendix

% ===========================================================================
% A. Mathematical Details
% ===========================================================================
\section{Mathematical Details}
\label{app:math}

\paragraph{Organization.}
Appendix~\ref{app:math} provides supplementary mathematical details (passivity, discrete-time energy under splitting, and Woodbury identities).
Appendix~\ref{app:environments} describes the benchmark systems and their Hamiltonians/damping laws.
Appendix~\ref{app:arch} and Appendix~\ref{app:potentials} detail architectural components and parameterizations.
Appendix~\ref{app:training} summarizes losses and hyperparameters.
Appendix~\ref{app:metrics} defines evaluation metrics and q-only diagnostics.
\ifdefined\isarxiv
Sec.~\ref{app:tables:qonly_rollout} and Appendix~\ref{app:ablations} collect additional tables and ablations.
Finally, Appendix~\ref{app:notation} lists notation, and Sec.~\ref{app:casimir} reports a separate closed-loop Energy--Casimir control study.
\else
Appendix~\ref{app:tables:qonly_rollout} and Appendix~\ref{app:ablations} collect additional tables and ablations.
Finally, Appendix~\ref{app:notation} lists notation and Appendix~\ref{app:casimir} reports a separate closed-loop Energy--Casimir control study.
\fi

\subsection{Passivity Proof}
\label{app:math:passivity}

\begin{theorem}[Energy Balance and Passivity]
Consider port-Hamiltonian dynamics with input
\begin{equation}
  \dot{x} = (\Jmat - \Rmat)\grad\Ham(x) + G u,
  \qquad x=(q,p),
\end{equation}
and define the conjugate port output $y^{\mathrm{port}} := G^\T \grad \Ham(x)$ (Sec.~\ref{sec:methods:ph}; not to be confused with the q-only observation $y_t=q_t$).
Then the energy balance is
\begin{equation}
  \frac{d\Ham}{dt} = -(\grad\Ham)^\T \Rmat(\grad\Ham) + {y^{\mathrm{port}}}^\T u
  \le {y^{\mathrm{port}}}^\T u.
\end{equation}
For mechanical systems with $\Rmat=\mathrm{diag}(0,\Dmat(q))$ and $\Dmat(q)\succeq 0$, this reduces to
\begin{equation}
  \frac{d\Ham}{dt} = -v^\T \Dmat(q)\,v + {y^{\mathrm{port}}}^\T u \le {y^{\mathrm{port}}}^\T u,
  \qquad v := \frac{\partial \Ham}{\partial p}\; (= \Mmat^{-1}p\ \text{for our separable Hamiltonian}).
\end{equation}
In particular, the unforced system ($u=0$) is passive with nonincreasing energy.
\end{theorem}

\begin{proof}
Expanding the time derivative of the Hamiltonian:
\begin{align}
    \frac{d\Ham}{dt} &= \grad\Ham^\T \dot{x}
    = \grad\Ham^\T (\Jmat - \Rmat)\grad\Ham + \grad\Ham^\T G u \\
    &= \underbrace{\grad\Ham^\T \Jmat \grad\Ham}_{= 0 \text{ (skew-sym)}} - \grad\Ham^\T \Rmat \grad\Ham \\
    &\quad + \underbrace{(G^\T \grad\Ham)^\T u}_{={y^{\mathrm{port}}}^\T u}.
\end{align}
With the block structure $\Rmat = \mathrm{diag}(0, \Dmat(q))$ and $\grad_p\Ham = v$:
\begin{equation}
    \frac{d\Ham}{dt} = -v^\T \Dmat(q)\,v + {y^{\mathrm{port}}}^\T u \le {y^{\mathrm{port}}}^\T u,
\end{equation}
since $\Dmat(q) \succeq 0$ by construction.
\end{proof}

\subsection{Energy Budget Under Strang Splitting}
\label{app:math:strang}

\begin{proposition}
Let $\Phi_H^{\dt}$ denote the (exact) flow of the conservative subsystem and $\Phi_D^{\dt}$ the (exact) flow of the dissipative subsystem with $q$ held fixed.
Under Strang splitting $\Phi_{\dt} = \Phi_D^{\dt/2} \circ \Phi_H^{\dt} \circ \Phi_D^{\dt/2}$ with sufficiently small $\dt$:
\begin{equation}
    \Ham(q^+, p^+) - \Ham(q, p)
    = -\int_{0}^{\dt} v(t)^\T \Dmat(q(t))\,v(t)\,dt
    = -\dt \cdot v^\T \Dmat(q)\,v + O(\dt^2),
\end{equation}
where $v(t)=\Mmat^{-1}p(t)$ and the $O(\dt^2)$ term reflects variation of $(q(t),v(t))$ over the step.
\end{proposition}

\paragraph{Remark (discrete-time dissipation).}
Our dissipative half-step uses an explicit (Euler) update on the momentum dynamics $\dot p=-\Dmat(q)\,v$:
\begin{equation}
    p \leftarrow p - \frac{\dt}{2}\,\Dmat(q)\,v, \qquad v = \Mmat^{-1}p.
\end{equation}
Since $q$ is fixed in this half-step, only the kinetic energy changes. Writing $h=\dt/2$, a direct expansion gives
\begin{equation}
  \Ham(q,p^+) - \Ham(q,p)
  =
  -h\,v^\T\Dmat(q)\,v
  + \frac{h^2}{2}\,v^\T \Dmat(q)\,\Mmat^{-1}\,\Dmat(q)\,v,
\end{equation}
so discrete-time energy monotonicity is step-size dependent.
In particular, a sufficient condition for $\Ham(q,p^+)\le \Ham(q,p)$ for all $v$ is
\begin{equation}
  s := \frac{\dt}{2}\,\frac{\lambda_{\max}(\Dmat(q))}{\lambda_{\min}(\Mmat)} \le 2,
\end{equation}
which can fail in stiff regimes.
Bounding $\sum_{i=1}^{r} \beta_i(q) \le \bbar$ (where $r$ is the number of rank-1 terms; Eq.~\ref{eq:damping_bound}) controls $\lambda_{\max}(\Dmat)$ and improves stability in practice.
Separately, our conservative map $\Phi_H^{\dt}$ is implemented with symplectic integrators (leapfrog / implicit midpoint), so the true Hamiltonian need not be exactly preserved even when the damping half-step is stable; we therefore treat passivity violations and energy-budget residuals as empirical diagnostics of the net discrete-time behavior.

\subsection{Woodbury Identity for Mass Inverse}
\label{app:math:woodbury}

For the low-rank mass $\Mmat = \Lambda + UU^\T$ where $\Lambda = \mathrm{diag}(d)$ and $U \in \R^{n \times r}$:
\begin{equation}
    \Mmat^{-1} = \Lambda^{-1} - \Lambda^{-1}U\,(I_r + U^\T\Lambda^{-1}U)^{-1}\,U^\T\Lambda^{-1}.
\end{equation}
The $r \times r$ matrix $(I_r + U^\T\Lambda^{-1}U)$ is inverted once; total cost is $O(nr^2)$ instead of $O(n^3)$.
More precisely, the dominant terms are $O(nr^2 + r^3)$ (with $r \ll n$ in all our experiments).

For the log-determinant (needed if extending to configuration-dependent/Riemannian Hamiltonians):
\begin{equation}
    \log|\Mmat| = \log|\Lambda| + \log|I_r + U^\T\Lambda^{-1}U|,
\end{equation}
where the first term is $O(n)$ and the second is $O(r^3)$.

\subsection{Spectral Properties of the Linearized System}
\label{app:math:spectrum}

Consider the linearized port-Hamiltonian dynamics $\dot x = (\Jmat - \Rmat)Q\,x$, where
$Q = Q^\T \succ 0$ is the Hessian of $\Ham$ at an equilibrium and $\Rmat \succeq 0$ is PSD.
Since $\Jmat$ is skew-symmetric, we have
\begin{equation}
  A := (\Jmat - \Rmat)\,Q.
\end{equation}
The eigenvalues of $A$ lie in the closed left half-plane,
i.e., $\mathrm{Re}(\lambda_i) \le 0$ for all $i$.

\begin{proof}
Let $(\lambda, z)$ be an eigenpair of $A$, so $(\Jmat - \Rmat)Qz = \lambda z$.
Define $w = Q^{1/2}z$ (well-defined since $Q \succ 0$), so
$Q^{1/2}(\Jmat - \Rmat)Q^{1/2}w = \lambda w$.
Then
$\mathrm{Re}(\lambda) = \mathrm{Re}(w^*Q^{1/2}(\Jmat - \Rmat)Q^{1/2}w) / (w^*w)$.
Since $w^*Q^{1/2}\Jmat Q^{1/2}w$ is purely imaginary (by skew-symmetry),
$\mathrm{Re}(\lambda) = -w^*Q^{1/2}\Rmat Q^{1/2}w / (w^*w) \le 0$
since $\Rmat \succeq 0$.
\end{proof}

\noindent
\textbf{Conservative case ($\Rmat = 0$):} All eigenvalues are purely imaginary,
corresponding to oscillatory energy exchange (phase-space rotations).

\noindent
\textbf{Dissipative case ($\Rmat \succ 0$):} Eigenvalues shift into the open left
half-plane (e.g., $-\sigma \pm i\omega$ with $\sigma > 0$), yielding damped oscillations.
See \citet{mehl2018linear} for a comprehensive treatment.

% ===========================================================================
% B. Environment Descriptions
% ===========================================================================
\section{Environment Descriptions}
\label{app:environments}

This section provides detailed descriptions of the systems used in our q-only benchmarks, including their state spaces, Hamiltonians, and damping models. We cover four mechanical systems (Secs.~\ref{app:env:pendulum}--\ref{app:env:double_pendulum}) and five non-mechanical systems (Sec.~\ref{app:table1_mappings}).

% ---------------------------------------------------------------------------
% B.1 Single Pendulum
% ---------------------------------------------------------------------------
\subsection{Single Pendulum}
\label{app:env:pendulum}

\begin{figure}[H]
    \centering
    \begin{tikzpicture}[scale=1.0]
        % Pivot/ceiling
        \fill[black!20] (-0.8,0.2) rectangle (0.8,0);
        \draw[thick] (-0.8,0) -- (0.8,0);
        \fill[black] (0,0) circle (0.08);

        % Rod
        \draw[very thick, blue!70!black, line cap=round] (0,0) -- (240:3.0);

        % Mass
        \fill[blue!60] (240:3.0) circle (0.25);
        \node[right] at ($(240:3.0)+(0.3,0)$) {$m, I$};

        % Center of mass marker
        \coordinate (CoM) at (240:2.0);
        \draw[thick, black!50, dashed] (0,0) -- (CoM);
        \fill[red!60] (CoM) circle (0.08);
        \node[left, red!60] at ($(CoM)+(-0.15,0)$) {CoM};

        % Vertical reference (dashed)
        \draw[dashed, black!40] (0,0) -- (0,-3.2);

        % Angle arc
        \draw[->, very thick] (0,-1.0) arc (-90:-120:1.0);
        \node[font=\large] at (-0.7,-1.3) {$\theta$};

        % Length annotations
        \draw[|-, thin, black!60] (0.25,0) -- (0.25,-1.3);
        \node[right, black!60] at (0.28,-0.65) {$\ell_c$};
        \draw[|-, thin, black!60] (0.5,0) -- ($(240:3.0)+(0.6,0.2)$);
        \node[right, black!60] at (0.53,-1.5) {$\ell$};

        % Gravity arrow
        \draw[->, very thick, black!50] (2.0,-0.5) -- (2.0,-1.8);
        \node[right, black!50, font=\large] at (2.0,-1.15) {$g$};

        % Damping torque (curved arrow at pivot)
        \draw[<-, thick, red!70, decorate, decoration={snake, amplitude=0.5mm, segment length=3mm}]
            (0.6,-0.2) arc (-20:-70:0.7);
        \node[red!70, right] at (0.9,-0.6) {$d(\theta)\dot{\theta}$};

        % Coordinate system
        \draw[->, thick] (-2.5,-2.8) -- (-1.8,-2.8) node[right] {$x$};
        \draw[->, thick] (-2.5,-2.8) -- (-2.5,-2.1) node[above] {$y$};
    \end{tikzpicture}
    \caption{\textbf{Single pendulum.} A point mass $m$ at distance $\ell$ from a fixed pivot, with moment of inertia $I$ about the pivot. The angle $\theta$ is measured from the downward vertical. Position-dependent damping $d(\theta)$ acts at the pivot.}
    \label{fig:pendulum_detail}
\end{figure}
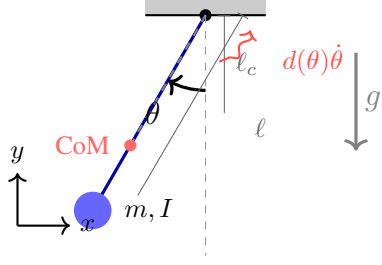

\paragraph{Configuration space.}
The single pendulum has one degree of freedom: the angle $\theta \in \mathbb{S}^1$ measured from the downward vertical (so $\theta = 0$ corresponds to the stable equilibrium with the pendulum hanging straight down).

\paragraph{State space.}
The phase-space state is $x = (\theta, p) \in \mathbb{S}^1 \times \mathbb{R}$, where $p = I\dot{\theta}$ is the angular momentum and $I$ is the moment of inertia about the pivot.

\paragraph{Kinematics.}
The Cartesian position of the center of mass is:
\begin{equation}
    \mathbf{r} = \begin{bmatrix} \ell_c \sin\theta \\ -\ell_c \cos\theta \end{bmatrix},
\end{equation}
where $\ell_c$ is the distance from the pivot to the center of mass.

\paragraph{Kinetic energy.}
For rotation about the fixed pivot:
\begin{equation}
    E_{\mathrm{kin}} = \frac{1}{2} I \dot{\theta}^2 = \frac{p^2}{2I},
\end{equation}
where $I = m\ell_c^2$ for a point mass (or includes rotational inertia for an extended body).

\paragraph{Potential energy.}
Taking the lowest point ($\theta = 0$) as the zero reference:
\begin{equation}
    V(\theta) = m g \ell_c (1 - \cos\theta).
\end{equation}

\paragraph{Hamiltonian.}
The total energy is:
\begin{equation}
    H(\theta, p) = \frac{p^2}{2I} + m g \ell_c (1 - \cos\theta).
    \label{eq:pendulum_hamiltonian}
\end{equation}

\paragraph{Equations of motion.}
The conservative dynamics are:
\begin{align}
    \dot{\theta} &= \frac{\partial H}{\partial p} = \frac{p}{I}, \\
    \dot{p} &= -\frac{\partial H}{\partial \theta} = -m g \ell_c \sin\theta.
\end{align}

\paragraph{Damping model.}
We consider three damping variants:
\begin{itemize}
    \item \textbf{Conservative}: $d(\theta) = 0$.
    \item \textbf{Constant damping}: $d(\theta) = \gamma$ (constant viscous friction, $\gamma = 0.5$).
    \item \textbf{Windy (position-dependent)}: $d(\theta) = d_0 + \Delta d\,|\sin\theta|$, where $d_0 = 0.3$ and $\Delta d = 0.5$.
\end{itemize}

\paragraph{Windy damping: physical interpretation.}
The ``windy'' damping model $d(\theta) = d_0 + \Delta d\,|\sin\theta|$ captures position-dependent air resistance: when the pendulum is horizontal ($\theta = \pm\pi/2$), it presents maximum cross-sectional area to the airflow, yielding maximum damping $d_{\max} = d_0 + \Delta d = 0.8$. When vertical ($\theta = 0$ or $\pi$), the cross-section is minimal, yielding $d_{\min} = d_0 = 0.3$. This is the \textbf{primary benchmark} for testing PHAST's ability to learn configuration-dependent dissipation $D(q)$.

\begin{figure}[H]
    \centering
    \begin{tikzpicture}[scale=0.85]
        % Axes
        \draw[->, thick] (-3.5,0) -- (3.8,0) node[right] {$\theta$};
        \draw[->, thick] (0,-0.3) -- (0,2.5) node[above] {$d(\theta)$};

        % Tick marks on x-axis
        \draw (-3.14,-0.1) -- (-3.14,0.1) node[below=2pt] {\small $-\pi$};
        \draw (-1.57,-0.1) -- (-1.57,0.1) node[below=2pt] {\small $-\frac{\pi}{2}$};
        \draw (0,-0.1) -- (0,0.1);
        \node[below=2pt] at (0,-0.1) {\small $0$};
        \draw (1.57,-0.1) -- (1.57,0.1) node[below=2pt] {\small $\frac{\pi}{2}$};
        \draw (3.14,-0.1) -- (3.14,0.1) node[below=2pt] {\small $\pi$};

        % d_0 and d_max lines
        \draw[dashed, black!50] (-3.5,0.6) -- (3.5,0.6);
        \node[left, black!60, font=\small] at (-3.5,0.6) {$d_0{=}0.3$};
        \draw[dashed, black!50] (-3.5,1.6) -- (3.5,1.6);
        \node[left, black!60, font=\small] at (-3.5,1.6) {$d_0{+}\Delta d{=}0.8$};

        % Plot d(theta) = 0.3 + 0.5*|sin(theta)|, scaled: y = 2*(0.3 + 0.5*|sin(x)|)
        \draw[very thick, blue!70!black, domain=-3.14:3.14, samples=100]
            plot (\x, {2*(0.3 + 0.5*abs(sin(\x r)))});

        % Annotations
        \node[blue!70!black, font=\small, align=center] at (2.2,2.2) {$d(\theta) = d_0 + \Delta d\,|\sin\theta|$};

        % Physical interpretation sketches
        \begin{scope}[shift={(-2.5,1.8)}, scale=0.35]
            \draw[thick, blue!60] (0,0) -- (-0.7,-0.7);
            \fill[blue!50] (-0.7,-0.7) circle (0.15);
            \node[font=\tiny] at (0,-1.3) {min drag};
        \end{scope}
        \begin{scope}[shift={(0,2.3)}, scale=0.35]
            \draw[thick, blue!60] (0,0) -- (1,0);
            \fill[blue!50] (1,0) circle (0.15);
            \node[font=\tiny] at (0.5,-0.7) {max drag};
        \end{scope}
    \end{tikzpicture}
    \caption{\textbf{Windy damping profile.} The position-dependent damping coefficient $d(\theta) = d_0 + \Delta d\,|\sin\theta|$ varies between $d_0 = 0.3$ (vertical) and $d_0 + \Delta d = 0.8$ (horizontal), modeling air resistance that depends on the pendulum's cross-sectional area.}
    \label{fig:windy_damping}
\end{figure}
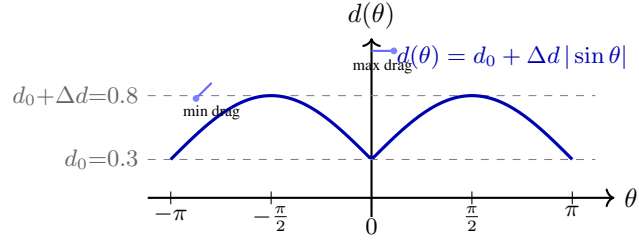

\paragraph{Port-Hamiltonian dynamics with windy damping.}
The dissipative equations of motion are:
\begin{align}
    \dot{\theta} &= \frac{p}{I}, \\
    \dot{p} &= -m g \ell_c \sin\theta - d(\theta) \cdot \frac{p}{I}.
\end{align}
The energy dissipation rate is:
\begin{equation}
    \frac{dH}{dt} = -d(\theta) \cdot v^2 = -\big(d_0 + \Delta d\,|\sin\theta|\big) \cdot \left(\frac{p}{I}\right)^2 \leq 0,
\end{equation}
which is always non-positive (passive) since $d(\theta) \geq d_0 > 0$.

\paragraph{Simulator parameters.}
In our experiments: $m = 1$, $\ell_c = \ell = 1$, $g = 9.81$, $I = m\ell_c^2 = 1$.

% ---------------------------------------------------------------------------
% B.2 Cart-Pole
% ---------------------------------------------------------------------------
\subsection{Cart-Pole}
\label{app:env:cartpole}

\begin{figure}[H]
    \centering
    \begin{tikzpicture}[scale=1.0]
        % Ground/rail
        \fill[black!10] (-2.5,-2.0) rectangle (3.5,-1.8);
        \draw[thick, black!50] (-2.5,-1.8) -- (3.5,-1.8);
        \foreach \x in {-2.2,-1.6,-1.0,-0.4,0.2,0.8,1.4,2.0,2.6,3.2}
            \draw[black!30] (\x,-1.8) -- (\x-0.2,-2.0);

        % Cart
        \draw[very thick, fill=black!15, rounded corners=2pt] (-0.8,-1.75) rectangle (0.8,-0.95);
        \fill[black] (-0.5,-1.8) circle (0.12);
        \fill[black] (0.5,-1.8) circle (0.12);
        \node at (0,-1.35) {$m_c$};

        % Cart position arrow
        \draw[<->, thick, black!60] (-2.0,-2.2) -- (0,-2.2);
        \node[below, black!60] at (-1.0,-2.2) {$x$};
        \draw[dashed, black!40] (0,-2.2) -- (0,-0.95);

        % Pole pivot
        \fill[black] (0,-0.95) circle (0.08);

        % Pole
        \coordinate (PoleEnd) at ($(0,-0.95)+(50:2.8)$);
        \draw[very thick, blue!70!black, line cap=round] (0,-0.95) -- (PoleEnd);

        % Pole CoM
        \coordinate (PoleCoM) at ($(0,-0.95)!0.5!(PoleEnd)$);
        \fill[blue!60] (PoleCoM) circle (0.18);
        \node[right] at ($(PoleCoM)+(0.25,0)$) {$m_p$};

        % Pole end mass (optional, for clarity)
        \fill[blue!40] (PoleEnd) circle (0.1);

        % Vertical reference (downward; stable at theta=0)
        \draw[dashed, black!40] (0,-0.95) -- (0,-1.75);

        % Angle arc
        \draw[->, very thick] (0,-1.75) arc (-90:50:0.8);
        \node[font=\large] at (0.6,0.8) {$\theta$};

        % Length annotation
        \draw[|-, thin, black!60] ($(0,-0.95)+(0.3,0)$) -- ($(PoleCoM)+(0.35,-0.1)$);
        \node[right, black!60] at ($(0,-0.95)!0.5!(PoleCoM)+(0.4,0)$) {$\ell$};

        % Gravity arrow
        \draw[->, very thick, black!50] (2.8,0.5) -- (2.8,-0.8);
        \node[right, black!50, font=\large] at (2.8,-0.15) {$g$};

        % Damping on pole (curved arrow)
        \draw[<-, thick, red!70, decorate, decoration={snake, amplitude=0.5mm, segment length=2.5mm}]
            (0.4,-0.75) arc (0:50:0.5);
        \node[red!70, right, font=\small] at (0.6,-0.5) {$d(\theta)\dot q$};

        % Coordinate frame
        \draw[->, thick] (-2.0,-0.5) -- (-1.3,-0.5) node[right] {$x$};
        \draw[->, thick] (-2.0,-0.5) -- (-2.0,0.2) node[above] {$y$};
    \end{tikzpicture}
    \caption{\textbf{Cart-pole system.} A cart of mass $m_c$ moves on a horizontal rail while a pole of mass $m_p$ (concentrated at distance $\ell$ from the pivot) rotates about the cart. The angle $\theta$ is measured from the downward vertical (hanging configuration: $\theta=0$). Windy damping acts isotropically on $\dot q$ via a scalar coefficient $d(\theta)$.}
    \label{fig:cartpole_detail}
\end{figure}
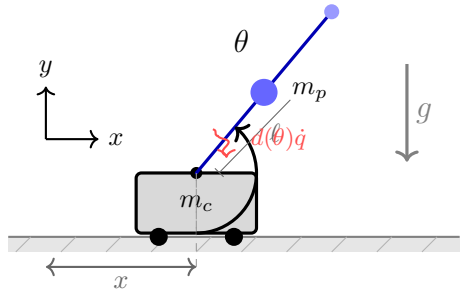

\paragraph{Configuration space.}
The cart-pole has two degrees of freedom: $q = (x, \theta) \in \mathbb{R} \times \mathbb{S}^1$, where $x$ is the cart position and $\theta$ is the pole angle measured from the downward vertical (stable at $\theta=0$).

\paragraph{State space.}
The phase-space state is $x = (q, p) = (x, \theta, p_x, p_\theta) \in \mathbb{R} \times \mathbb{S}^1 \times \mathbb{R}^2$.

\paragraph{Kinematics.}
The positions of the cart and pole center of mass are:
\begin{align}
    \mathbf{r}_c &= \begin{bmatrix} x \\ 0 \end{bmatrix}, &
    \mathbf{r}_p &= \begin{bmatrix} x + \ell \sin\theta \\ -\ell \cos\theta \end{bmatrix}.
\end{align}

\paragraph{Kinetic energy.}
The velocities are $\dot{\mathbf{r}}_c = (\dot{x}, 0)^\T$ and:
\begin{equation}
    \dot{\mathbf{r}}_p = \begin{bmatrix} \dot{x} + \ell \cos\theta \cdot \dot{\theta} \\ \ell \sin\theta \cdot \dot{\theta} \end{bmatrix}.
\end{equation}
The total kinetic energy is:
\begin{align}
    E_{\mathrm{kin}} &= \frac{1}{2} m_c \dot{x}^2 + \frac{1}{2} m_p \|\dot{\mathbf{r}}_p\|^2 \\
      &= \frac{1}{2}(m_c + m_p)\dot{x}^2 + m_p \ell \cos\theta \cdot \dot{x}\dot{\theta} + \frac{1}{2} m_p \ell^2 \dot{\theta}^2,
\end{align}
which can be written as $E_{\mathrm{kin}} = \frac{1}{2}\dot{q}^\T M(q)\dot{q}$ with the \textbf{configuration-dependent mass matrix}:
\begin{equation}
    M(\theta) = \begin{bmatrix}
        m_c + m_p & m_p \ell \cos\theta \\
        m_p \ell \cos\theta & m_p \ell^2
    \end{bmatrix}.
    \label{eq:cartpole_mass}
\end{equation}

\paragraph{Potential energy.}
Taking the hanging configuration ($\theta = 0$) as the zero reference:
\begin{equation}
    V(\theta) = m_p g \ell (1 - \cos\theta).
\end{equation}

\paragraph{Hamiltonian.}
With generalized momenta $p = M(\theta)\dot{q}$:
\begin{equation}
    H(q, p) = \frac{1}{2} p^\T M(\theta)^{-1} p + m_p g \ell (1 - \cos\theta).
    \label{eq:cartpole_hamiltonian}
\end{equation}

\paragraph{Separable approximation.}
In our PHAST experiments, we use a \emph{separable} Hamiltonian with constant $M$ (Sec.~\ref{sec:intro:regimes}), which is an approximation. The true Cart-Pole has configuration-dependent inertia as shown in Eq.~\eqref{eq:cartpole_mass}.

\paragraph{Damping model.}
We use windy (position-dependent) damping on the pole angle:
\begin{equation}
    d(\theta) = d_0 + \Delta d\,|\sin\theta|,
\end{equation}
with $d_0 = 0.3$ and $\Delta d = 0.5$. The damping acts on both the cart and pole velocities as a scalar coefficient:
\begin{equation}
    \dot{p} = -\nabla_q H - d(\theta) \cdot \dot{q}.
\end{equation}
This models increased air resistance when the pole is horizontal. The energy dissipation rate is:
\begin{equation}
    \frac{dH}{dt} = -d(\theta) \cdot \|\dot{q}\|^2 \leq 0.
\end{equation}

\paragraph{Simulator parameters.}
In our experiments: $m_c = 1$, $m_p = 1$, $\ell = 1$, $g = 9.81$, $\dt = 0.02$.

% ---------------------------------------------------------------------------
% B.3 Harmonic Oscillator
% ---------------------------------------------------------------------------
\subsection{Harmonic Oscillator}
\label{app:env:oscillator}

\begin{figure}[H]
    \centering
    \begin{tikzpicture}[scale=1.0]
        % Two independent harmonic oscillators visualization
        % Oscillator 1
        \begin{scope}[shift={(0,0.8)}]
            % Wall
            \fill[black!20] (-0.2,-0.6) rectangle (0,0.6);
            \draw[very thick] (0,-0.6) -- (0,0.6);
            \foreach \y in {-0.4,0,0.4}
                \draw[black!40] (0,\y) -- (-0.2,\y+0.1);

            % Spring
            \draw[thick, decorate, decoration={coil, aspect=0.5, segment length=3mm, amplitude=2.5mm}]
                (0,0) -- (2.2,0);
            \node[above, font=\small] at (1.1,0.3) {$\omega$};

            % Mass
            \draw[very thick, fill=blue!30, rounded corners=2pt] (2.2,-0.4) rectangle (3.2,0.4);
            \node at (2.7,0) {$m_1$};

            % Displacement
            \draw[->, very thick, black!70] (2.7,-0.6) -- (3.5,-0.6);
            \node[below, black!70, font=\small] at (3.1,-0.6) {$q_1$};

            % Damping
            \draw[thick, red!60, decorate, decoration={snake, amplitude=0.6mm, segment length=2mm}]
                (2.7,0.55) -- (2.7,0.9);
            \node[red!60, above, font=\scriptsize] at (2.7,0.9) {$\gamma$};
        \end{scope}

        % Oscillator 2
        \begin{scope}[shift={(0,-0.8)}]
            % Wall
            \fill[black!20] (-0.2,-0.6) rectangle (0,0.6);
            \draw[very thick] (0,-0.6) -- (0,0.6);
            \foreach \y in {-0.4,0,0.4}
                \draw[black!40] (0,\y) -- (-0.2,\y+0.1);

            % Spring
            \draw[thick, decorate, decoration={coil, aspect=0.5, segment length=3mm, amplitude=2.5mm}]
                (0,0) -- (2.2,0);
            \node[above, font=\small] at (1.1,0.3) {$\omega$};

            % Mass
            \draw[very thick, fill=blue!30, rounded corners=2pt] (2.2,-0.4) rectangle (3.2,0.4);
            \node at (2.7,0) {$m_2$};

            % Displacement
            \draw[->, very thick, black!70] (2.7,-0.6) -- (3.5,-0.6);
            \node[below, black!70, font=\small] at (3.1,-0.6) {$q_2$};

            % Damping
            \draw[thick, red!60, decorate, decoration={snake, amplitude=0.6mm, segment length=2mm}]
                (2.7,0.55) -- (2.7,0.9);
            \node[red!60, above, font=\scriptsize] at (2.7,0.9) {$\gamma$};
        \end{scope}

        % Label
        \node[font=\small, align=center] at (5.5,0) {$n=2$ independent\\harmonic oscillators\\with unit mass};
    \end{tikzpicture}
    \caption{\textbf{Harmonic oscillator.} Two independent harmonic oscillators with natural frequency $\omega$ and unit mass. Coordinates $q_1, q_2$ measure displacements from equilibrium. Optional viscous damping $\gamma$ acts on each oscillator.}
    \label{fig:oscillator_detail}
\end{figure}
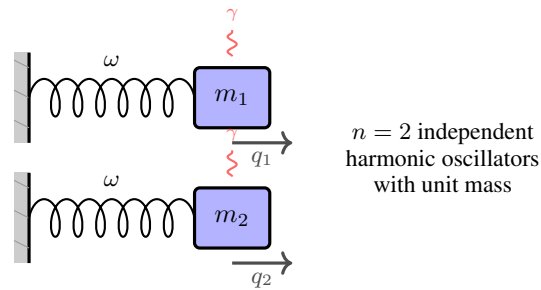

\paragraph{Configuration space.}
The oscillator benchmark uses $n$ degrees of freedom (we use $n=2$): $q = (q_1, \ldots, q_n) \in \mathbb{R}^n$, where $q_i$ is the displacement of oscillator $i$ from equilibrium.

\paragraph{State space.}
The phase-space state is $x = (q, p) \in \mathbb{R}^{2n}$, where $p = \dot{q}$ (unit mass).

\paragraph{Kinetic energy.}
With unit mass ($M = I$):
\begin{equation}
    E_{\mathrm{kin}} = \frac{1}{2} \|p\|^2 = \frac{1}{2} \sum_{i=1}^{n} p_i^2.
\end{equation}

\paragraph{Potential energy.}
Simple harmonic potential with natural frequency $\omega$:
\begin{equation}
    V(q) = \frac{1}{2} \omega^2 \|q\|^2 = \frac{1}{2} \omega^2 \sum_{i=1}^{n} q_i^2.
\end{equation}

\paragraph{Hamiltonian.}
The total energy is:
\begin{equation}
    H(q, p) = \frac{1}{2} \|p\|^2 + \frac{1}{2} \omega^2 \|q\|^2.
    \label{eq:oscillator_hamiltonian}
\end{equation}

\paragraph{Equations of motion.}
\begin{align}
    \dot{q} &= p, \\
    \dot{p} &= -\omega^2 q - \gamma p,
\end{align}
where $\gamma$ is the viscous damping coefficient.

\paragraph{Damping model.}
\begin{itemize}
    \item \textbf{Conservative}: $\gamma = 0$ (energy preserved).
    \item \textbf{Damped}: $\gamma = 0.1$ (constant viscous damping).
\end{itemize}

\paragraph{Simulator parameters.}
In our experiments: $n = 2$ (degrees of freedom), $\omega = 1$ (natural frequency), $\dt = 0.02$.

% ---------------------------------------------------------------------------
% B.4 Double Pendulum
% ---------------------------------------------------------------------------
\subsection{Double Pendulum}
\label{app:env:double_pendulum}

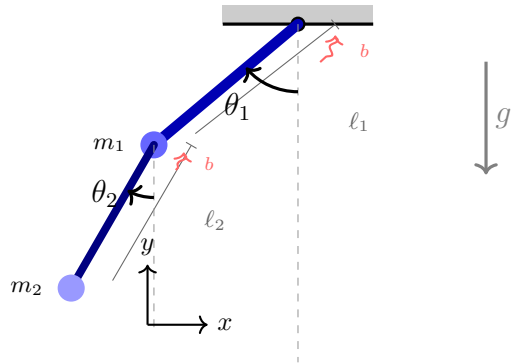
\begin{figure}[H]
    \centering
    \begin{tikzpicture}[scale=1.0]
        % Pivot/ceiling
        \fill[black!20] (-1.0,0.25) rectangle (1.0,0);
        \draw[very thick] (-1.0,0) -- (1.0,0);
        \fill[black] (0,0) circle (0.1);

        % Define angles
        \def\thetaOne{-50}  % theta1 from downward vertical
        \def\thetaTwo{-30}  % theta2 (absolute angle of link 2)

        % Link 1
        \coordinate (J1) at (0,0);
        \coordinate (E1) at ($(\thetaOne-90:2.5)$);
        \draw[line width=4pt, blue!70!black, line cap=round] (J1) -- (E1);
        \fill[black] (E1) circle (0.1);

        % Point mass 1 at end of link 1
        \fill[blue!60] (E1) circle (0.18);
        \node[left, font=\small] at ($(E1)+(-0.25,0)$) {$m_1$};

        % Link 1 length annotation
        \draw[|-, thin, black!60] (0.5,0) -- ($(E1)+(0.55,0.15)$);
        \node[right, black!60, font=\small] at (0.55,-1.3) {$\ell_1$};

        % Link 2
        \coordinate (E2) at ($(E1)+(\thetaTwo-90:2.2)$);
        \draw[line width=3pt, blue!50!black, line cap=round] (E1) -- (E2);

        % Point mass 2 at end of link 2
        \fill[blue!40] (E2) circle (0.18);
        \node[left, font=\small] at ($(E2)+(-0.25,0)$) {$m_2$};

        % Link 2 length annotation
        \draw[|-, thin, black!60] ($(E1)+(0.5,0)$) -- ($(E2)+(0.55,0.1)$);
        \node[right, black!60, font=\small] at ($(E1)+(0.55,-1.0)$) {$\ell_2$};

        % Vertical reference lines
        \draw[dashed, black!40] (0,0) -- (0,-4.5);
        \draw[dashed, black!40] (E1) -- ($(E1)+(0,-2.5)$);

        % Angle theta1 (absolute)
        \draw[->, very thick] (0,-0.9) arc (-90:\thetaOne-90:0.9);
        \node[font=\large] at (-0.8,-1.1) {$\theta_1$};

        % Angle theta2 (absolute)
        \draw[->, very thick] ($(E1)+(0,-0.7)$) arc (-90:\thetaTwo-90:0.7);
        \node[font=\large] at ($(E1)+(-0.65,-0.65)$) {$\theta_2$};

        % Gravity arrow
        \draw[->, very thick, black!50] (2.5,-0.5) -- (2.5,-2.0);
        \node[right, black!50, font=\large] at (2.5,-1.25) {$g$};

        % Coordinate system
        \draw[->, thick] (-2.0,-4.0) -- (-1.2,-4.0) node[right] {$x$};
        \draw[->, thick] (-2.0,-4.0) -- (-2.0,-3.2) node[above] {$y$};

        % Optional damping indicators
        \draw[<-, thick, red!60, decorate, decoration={snake, amplitude=0.4mm, segment length=2mm}]
            (0.5,-0.15) arc (-10:-50:0.6);
        \node[red!60, right, font=\scriptsize] at (0.7,-0.35) {$b$};
        \draw[<-, thick, red!60, decorate, decoration={snake, amplitude=0.4mm, segment length=2mm}]
            ($(E1)+(0.4,-0.1)$) arc (-10:-40:0.5);
        \node[red!60, right, font=\scriptsize] at ($(E1)+(0.55,-0.25)$) {$b$};
    \end{tikzpicture}
    \caption{\textbf{Double pendulum (point-mass model).} Two point masses $m_1, m_2$ at the ends of massless rods of lengths $\ell_1, \ell_2$. The angles $\theta_1$ and $\theta_2$ are both measured from the downward vertical (absolute angles). Optional viscous damping $b$ acts at both joints.}
    \label{fig:double_pendulum_detail}
\end{figure}

\paragraph{Configuration space.}
The double pendulum has two degrees of freedom: $q = (\theta_1, \theta_2) \in \mathbb{S}^1 \times \mathbb{S}^1$, where $\theta_1$ is the angle of link 1 from the downward vertical, and $\theta_2$ is the angle of link 2 from the downward vertical (both absolute angles).

\paragraph{State space.}
The phase-space state is $x = (q, p) = (\theta_1, \theta_2, p_1, p_2) \in (\mathbb{S}^1)^2 \times \mathbb{R}^2$, where $p = M(q)\dot{q}$ are the canonical momenta.

\paragraph{Kinematics.}
Using shorthand $s_i = \sin\theta_i$, $c_i = \cos\theta_i$:
\begin{align}
    \mathbf{r}_1 &= \begin{bmatrix} \ell_1 s_1 \\ -\ell_1 c_1 \end{bmatrix}, &
    \mathbf{r}_2 &= \begin{bmatrix} \ell_1 s_1 + \ell_2 s_2 \\ -\ell_1 c_1 - \ell_2 c_2 \end{bmatrix}.
\end{align}

\paragraph{Kinetic energy.}
For point masses at the ends of massless rods:
\begin{equation}
    E_{\mathrm{kin}} = \frac{1}{2} m_1 \|\dot{\mathbf{r}}_1\|^2 + \frac{1}{2} m_2 \|\dot{\mathbf{r}}_2\|^2 = \frac{1}{2} \dot{q}^\T M(q) \dot{q},
\end{equation}
where the \textbf{configuration-dependent mass matrix} is:
\begin{equation}
    M(q) = \begin{bmatrix}
        (m_1 + m_2) \ell_1^2 & m_2 \ell_1 \ell_2 \cos(\theta_1 - \theta_2) \\
        m_2 \ell_1 \ell_2 \cos(\theta_1 - \theta_2) & m_2 \ell_2^2
    \end{bmatrix}.
    \label{eq:double_pendulum_mass}
\end{equation}
Note that $M$ depends on the angle difference $(\theta_1 - \theta_2)$.

\paragraph{Potential energy.}
Taking the pivot level as the zero reference:
\begin{equation}
    V(\theta_1, \theta_2) = -(m_1 + m_2) g \ell_1 \cos\theta_1 - m_2 g \ell_2 \cos\theta_2.
    \label{eq:double_pendulum_potential}
\end{equation}

\paragraph{Hamiltonian.}
With generalized momenta $p = M(q) \dot{q}$:
\begin{equation}
    H(q, p) = \frac{1}{2} p^\T M(q)^{-1} p + V(\theta_1, \theta_2).
    \label{eq:double_pendulum_hamiltonian}
\end{equation}

\paragraph{Equations of motion.}
The Euler--Lagrange equations yield (in manipulator form):
\begin{equation}
    M(q) \ddot{q} + C(q, \dot{q}) \dot{q} + G(q) = -b \dot{q},
\end{equation}
where $C(q, \dot{q})$ contains Coriolis/centrifugal terms involving $\sin(\theta_1 - \theta_2)$, $G(q) = \nabla_q V(q)$ is the gravity vector, and $b$ is the viscous damping coefficient applied equally to both joints.

\paragraph{Separable approximation.}
In our PHAST experiments, we use a \emph{separable} Hamiltonian with constant $M$ (Sec.~\ref{sec:intro:regimes}). This is an approximation; the true double pendulum has configuration-dependent inertia as shown in Eq.~\eqref{eq:double_pendulum_mass}.

\paragraph{Damping model.}
\begin{itemize}
    \item \textbf{Conservative}: $b = 0$.
    \item \textbf{Damped}: $b = 0.2$ (viscous damping on both joints).
\end{itemize}

\paragraph{Simulator parameters.}
In our experiments: $m_1 = m_2 = 1$, $\ell_1 = \ell_2 = 1$, $g = 9.81$, $\dt = 0.01$.

% ---------------------------------------------------------------------------
% B.5 Detailed Table 1 Mappings
% ---------------------------------------------------------------------------
\subsection{Detailed Port-Hamiltonian Mappings for Table~\ref{tab:structured_hamiltonians}}
\label{app:table1_mappings}

Section~\ref{sec:intro} walks through three representative entries
(simple pendulum, cart-pole, RLC circuit).
Here we provide the complete $(V,\Mmat,\Dmat)$ decomposition for every
system in Table~\ref{tab:structured_hamiltonians}, organized by knowledge
regime.

% ---- KNOWN regime ----
\paragraph{KNOWN regime.}
In these systems, the potential $V$ and mass $\Mmat$ are fully specified by the physics; only the dissipation $\Dmat$ (or a small number of scalar parameters) must be learned.
This makes the inverse problem well-posed: recovering $\Dmat$ from trajectory data has a unique solution.

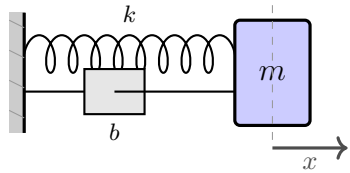
\begin{figure}[H]
    \centering
    \begin{tikzpicture}[scale=1.0]
        % Wall
        \fill[black!20] (-0.2,-0.8) rectangle (0,0.8);
        \draw[very thick] (0,-0.8) -- (0,0.8);
        \foreach \y in {-0.6,-0.2,0.2,0.6}
            \draw[black!40] (0,\y) -- (-0.2,\y+0.1);

        % Spring (upper path)
        \draw[thick, decorate, decoration={coil, aspect=0.5, segment length=3mm, amplitude=2.5mm}]
            (0,0.25) -- (2.8,0.25);
        \node[above, font=\small] at (1.4,0.55) {$k$};

        % Dashpot (lower path)
        \draw[thick] (0,-0.25) -- (0.8,-0.25);
        \draw[thick, fill=black!10] (0.8,-0.55) rectangle (1.6,0.05);
        \draw[thick] (1.2,-0.25) -- (2.8,-0.25);
        \node[below, font=\small] at (1.2,-0.55) {$b$};

        % Mass block
        \draw[very thick, fill=blue!20, rounded corners=2pt] (2.8,-0.7) rectangle (3.8,0.7);
        \node[font=\large] at (3.3,0) {$m$};

        % Displacement arrow
        \draw[->, very thick, black!70] (3.3,-1.0) -- (4.3,-1.0);
        \node[below, black!70] at (3.8,-1.0) {$x$};

        % Equilibrium reference
        \draw[dashed, black!40] (3.3,0.9) -- (3.3,-0.9);
    \end{tikzpicture}
    \caption{\textbf{Spring--mass--damper.} A mass $m$ connected to a wall by a spring (stiffness $k$) and dashpot (damping $b$) in parallel.  Displacement $x$ is measured from equilibrium.}
    \label{fig:spring_mass_detail}
\end{figure}

\paragraph{Spring--mass (KNOWN, mechanical).}
The damped harmonic oscillator is the simplest non-trivial
port-Hamiltonian system and serves as the ``hello world'' for our
benchmark: if a method cannot recover the damping coefficient of a
linear spring, it will struggle with anything more complex.
Despite its simplicity, the system already exhibits the
conservative--dissipative split that defines the port-Hamiltonian
framework: the spring stores energy, the dashpot removes it, and
the mass mediates how quickly energy converts between potential and
kinetic forms.
\begin{itemize}
  \item \emph{State}: $q = x$ (displacement from equilibrium),
    $p = m\dot x$ (linear momentum).
  \item \emph{Potential}: $V(x) = \tfrac{1}{2}kx^2$ --- a quadratic
    well whose gradient $-kx$ is the familiar restoring force.
  \item \emph{Mass}: $\Mmat = m$ --- a scalar constant.
    In the port-Hamiltonian picture, $m$ controls the ``exchange rate''
    between momentum and velocity: $\dot x = p/m$.
  \item \emph{Damping}: $\Dmat = b$ --- viscous friction ($F_d = -b\dot x$).
    Energy dissipation is $\dot H = -b\dot x^2 \le 0$, the defining
    passivity inequality.
  \item \emph{Hamiltonian}:
    $H(x,p) = \tfrac{1}{2}kx^2 + \tfrac{p^2}{2m}$.
  \item \emph{Identifiability}: With $V$ and $m$ both given, the only
    unknown is~$b$---recoverable from the decay rate of oscillations.
\end{itemize}

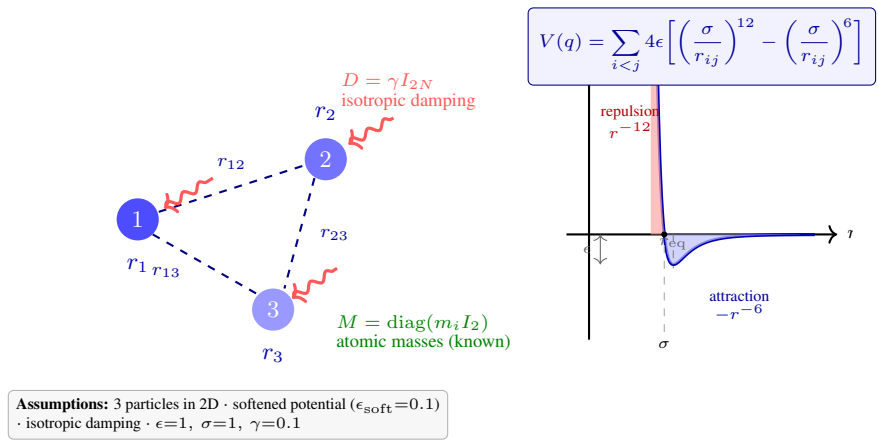
\begin{figure}[H]
    \centering
    \begin{tikzpicture}[scale=1.0]
        % --- Particle cluster (left side) ---
        % Particles with colored borders
        \fill[blue!70] (0,0) circle (0.28);
        \node[white, font=\small\bfseries] at (0,0) {$1$};
        \node[below, font=\small, blue!70!black] at (0,-0.4) {$r_1$};

        \fill[blue!55] (2.5,0.8) circle (0.28);
        \node[white, font=\small\bfseries] at (2.5,0.8) {$2$};
        \node[above, font=\small, blue!70!black] at (2.5,1.2) {$r_2$};

        \fill[blue!40] (1.8,-1.2) circle (0.28);
        \node[white, font=\small\bfseries] at (1.8,-1.2) {$3$};
        \node[below, font=\small, blue!70!black] at (1.8,-1.6) {$r_3$};

        % Pairwise interaction lines
        \draw[thick, blue!50!black, dashed] (0.28,0.1) -- (2.22,0.72);
        \node[above, font=\scriptsize, blue!60!black] at (1.25,0.55) {$r_{12}$};
        \draw[thick, blue!50!black, dashed] (0.2,-0.2) -- (1.6,-1.0);
        \node[left, font=\scriptsize, blue!60!black] at (0.7,-0.7) {$r_{13}$};
        \draw[thick, blue!50!black, dashed] (2.35,0.55) -- (1.95,-0.92);
        \node[right, font=\scriptsize, blue!60!black] at (2.3,-0.2) {$r_{23}$};

        % Velocity-proportional damping (wavy arrows) — red for damping
        \draw[<-, very thick, red!65, decorate, decoration={snake, amplitude=0.5mm, segment length=2.5mm}]
            (0.35,0.15) -- (1.0,0.55);
        \draw[<-, very thick, red!65, decorate, decoration={snake, amplitude=0.5mm, segment length=2.5mm}]
            (2.8,1.0) -- (3.4,1.35);
        \draw[<-, very thick, red!65, decorate, decoration={snake, amplitude=0.5mm, segment length=2.5mm}]
            (2.05,-1.05) -- (2.65,-0.65);

        % Damping label
        \node[red!65, font=\scriptsize, align=left] at (3.6,1.7)
            {$\textcolor{red!65}{D = \gamma I_{2N}}$\\[-0.1em]\scriptsize isotropic damping};

        % Mass label
        \node[green!50!black, font=\scriptsize, align=left] at (3.8,-1.5)
            {$\textcolor{green!50!black}{M = \mathrm{diag}(m_i I_2)}$\\[-0.1em]\scriptsize atomic masses (known)};

        % --- LJ potential inset (right side, larger) ---
        \begin{scope}[shift={(6.0,-0.2)}]
            \clip (-0.5,-1.6) rectangle (3.5,2.4);
            % Axes
            \draw[->, thick] (0,-1.4) -- (0,2.2) node[above, font=\scriptsize] {$V(r)$};
            \draw[->, thick] (-0.3,0) -- (3.3,0) node[right, font=\scriptsize] {$r$};

            % LJ curve — correct (12,6) potential: V(r) = 4ε[(σ/r)^12 - (σ/r)^6]
            % Scaled by 0.4 for visual fit: 1.6*(x^{-12} - x^{-6})
            % Computed as 1.6 * (1/x^6) * (1/x^6 - 1) to avoid x^12
            \draw[very thick, blue!70!black, smooth]
                plot[domain=0.82:3.0, samples=100]
                (\x, {1.6*(1/(\x*\x*\x*\x*\x*\x))*(1/(\x*\x*\x*\x*\x*\x)-1)});

            % Repulsive region shading (r < sigma = 1.0, V > 0)
            \fill[red!35, opacity=0.7]
                plot[domain=0.82:1.0, samples=40]
                (\x, {min(2.4, 1.6*(1/(\x*\x*\x*\x*\x*\x))*(1/(\x*\x*\x*\x*\x*\x)-1))})
                -- (1.0,0) -- (0.82,0) -- cycle;
            \node[red!70!black, font=\tiny, align=center] at (0.55,1.5) {repulsion\\$r^{-12}$};

            % Attractive region shading (r > sigma, V < 0)
            \fill[blue!25, opacity=0.7]
                plot[domain=1.0:2.8, samples=50]
                (\x, {min(0, 1.6*(1/(\x*\x*\x*\x*\x*\x))*(1/(\x*\x*\x*\x*\x*\x)-1))})
                -- (2.8,0) -- (1.0,0) -- cycle;
            \node[blue!60!black, font=\tiny, align=center] at (2.0,-0.95) {attraction\\$-r^{-6}$};

            % Equilibrium distance at 2^{1/6} ≈ 1.122
            \draw[dashed, black!50] (1.12,0) -- (1.12,-0.45);
            \node[below, font=\tiny, black!60] at (1.12,0.08) {$r_{\mathrm{eq}}$};

            % Sigma marker
            \draw[dashed, black!35] (1.0,-1.3) -- (1.0,0.05);
            \node[below, font=\tiny] at (1.0,-1.3) {$\sigma$};

            % Well depth (minimum at r_eq is ~-0.4 in scaled units)
            \draw[<->, thin, black!60] (0.15,-0.40) -- (0.15,0);
            \node[left, font=\tiny, black!60] at (0.15,-0.20) {$\epsilon$};

            % Zero crossing
            \fill[black] (1.0,0) circle (0.04);
        \end{scope}

        % Potential equation box
        \node[draw=blue!50!black, fill=blue!5, rounded corners=2pt, font=\scriptsize,
              inner sep=4pt, align=center] at (7.5,2.3)
            {$\textcolor{blue!60!black}{V(q) = \displaystyle\sum_{i<j} 4\epsilon\!\left[\left(\frac{\sigma}{r_{ij}}\right)^{\!12} - \left(\frac{\sigma}{r_{ij}}\right)^{\!6}\right]}$};

        % Assumptions box
        \node[draw=black!30, fill=black!3, rounded corners=2pt, font=\tiny,
              inner sep=3pt, align=left] at (1.2,-2.6)
            {\textbf{Assumptions:} 3 particles in 2D $\cdot$ softened potential ($\epsilon_{\mathrm{soft}}{=}0.1$)\\
             $\cdot$ isotropic damping $\cdot$ $\epsilon{=}1,\,\sigma{=}1,\,\gamma{=}0.1$};
    \end{tikzpicture}
    \caption{\textbf{Lennard--Jones cluster}~\citep{jones1924determination}.
    $N{=}3$ particles interact pairwise via the LJ potential.
    Inset: the repulsive core ($r^{-12}$, \textcolor{red!70!black}{red shading}) dominates at $r < \sigma$; the attractive well ($-r^{-6}$, \textcolor{blue!60!black}{blue shading}) has depth~$\epsilon$ near the equilibrium distance~$r_{\mathrm{eq}} = 2^{1/6}\sigma$.
    Dissipation is modeled by isotropic velocity-proportional damping (\textcolor{red!65}{$D = \gamma I$}).}
    \label{fig:lj_detail}
\end{figure}

\paragraph{Lennard--Jones cluster (KNOWN, molecular dynamics).}
Molecular dynamics is a natural stress test for structure-preserving
learners: the Lennard--Jones potential has a steep repulsive wall that
amplifies integration errors, yet energy must remain bounded for the
cluster to stay intact.
A port-Hamiltonian integrator respects this energy budget by
construction, making it a principled alternative to hand-tuned
thermostats.
We simulate $N{=}3$ particles in 2D interacting through the
classical LJ pair potential~\citep{jones1924determination}, with a
velocity-proportional drag that models coupling to an implicit heat
bath.
\begin{itemize}
  \item \emph{State}: $q = (r_1, \ldots, r_N) \in \R^{2N}$ (particle positions in 2D),
    $p = (m_1\dot r_1, \ldots, m_N\dot r_N)$ (momenta).
  \item \emph{Potential}: $\textcolor{blue!60!black}{V(q) = \sum_{i<j} 4\epsilon\bigl[(\sigma/r_{ij})^{12}
    - (\sigma/r_{ij})^6\bigr]}$,
    where $r_{ij} = \|r_i - r_j\|$, $\epsilon$ is the well depth, and
    $\sigma$ is the particle diameter.
    The $r^{-12}$ term encodes Pauli repulsion (atoms cannot overlap);
    the $r^{-6}$ term encodes van der Waals attraction (induced-dipole
    interactions at longer range).
    Together they create a potential well at
    $r_{\mathrm{eq}} = 2^{1/6}\sigma$ that traps the cluster.
  \item \emph{Mass}: $\textcolor{green!50!black}{\Mmat = \mathrm{diag}(m_1 I_2, \ldots, m_N I_2)}$
    --- each particle carries a known scalar mass; the block-diagonal
    structure reflects the absence of cross-particle inertial coupling.
  \item \emph{Damping}: $\textcolor{red!65}{\Dmat = \gamma I_{2N}}$ ---
    isotropic friction that damps every degree of freedom equally.
    Physically, $\gamma$ controls how quickly kinetic energy drains to
    the heat bath; in the conservative limit $\gamma{\to}0$ the cluster
    oscillates forever.
  \item \emph{Identifiability}: All physics (V, M) is given, so the
    inverse problem reduces to recovering a single scalar~$\gamma$.
\end{itemize}

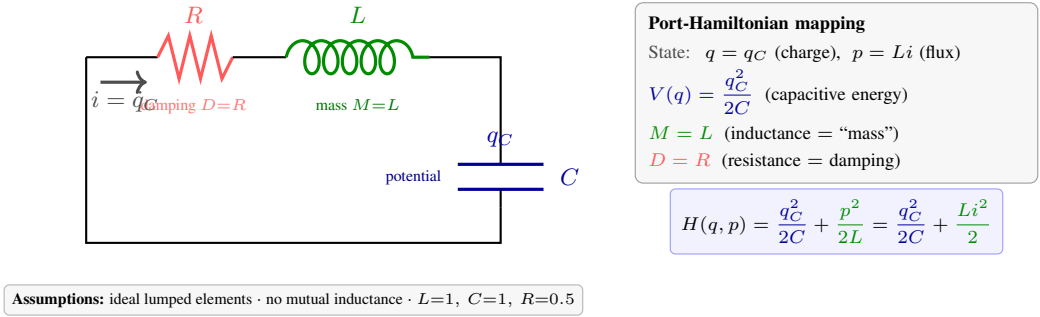
\begin{figure}[H]
    \centering
    \begin{tikzpicture}[scale=0.95]
        % --- Circuit loop ---
        % Top wire
        \draw[thick] (0,0) -- (1.0,0);

        % Resistor (zigzag) — RED for damping
        \draw[very thick, red!65] (1.0,0) -- (1.15,0.3) -- (1.3,-0.3) -- (1.45,0.3)
            -- (1.6,-0.3) -- (1.75,0.3) -- (1.9,-0.3) -- (2.05,0);
        \node[above, font=\small\bfseries, red!65] at (1.5,0.35) {$R$};
        \node[below, font=\tiny, red!55] at (1.5,-0.45) {damping $D{=}R$};

        % Wire to inductor
        \draw[thick] (2.05,0) -- (2.8,0);

        % Inductor (coil) — GREEN for mass/inertia
        \draw[very thick, green!55!black, decorate, decoration={coil, aspect=0.7, segment length=2.5mm, amplitude=2.2mm}]
            (2.8,0) -- (4.8,0);
        \node[above, font=\small\bfseries, green!55!black] at (3.8,0.35) {$L$};
        \node[below, font=\tiny, green!45!black] at (3.8,-0.45) {mass $M{=}L$};

        % Right wire down
        \draw[thick] (4.8,0) -- (5.8,0) -- (5.8,-1.5);

        % Capacitor — BLUE for potential
        \draw[very thick, blue!60!black] (5.2,-1.5) -- (6.4,-1.5);
        \draw[very thick, blue!60!black] (5.2,-1.85) -- (6.4,-1.85);
        \draw[thick] (5.8,-1.85) -- (5.8,-2.1);
        \node[right, font=\small\bfseries, blue!60!black] at (6.5,-1.68) {$C$};
        \node[left, font=\tiny, blue!50!black] at (5.1,-1.68) {potential};

        % Charge label
        \node[font=\small, blue!70!black] at (5.8,-1.15) {$q_C$};

        % Bottom wire
        \draw[thick] (5.8,-2.1) -- (5.8,-2.6) -- (0,-2.6) -- (0,0);

        % Current arrow
        \draw[->, very thick, black!70] (0.2,-0.35) -- (0.85,-0.35);
        \node[below, font=\small, black!70] at (0.55,-0.35) {$i = \dot q_C$};

        % --- pH mapping box (right side) ---
        \node[draw=black!40, fill=black!3, rounded corners=3pt, font=\scriptsize,
              inner sep=5pt, align=left, text width=5.0cm] at (10.5,-0.5)
            {\textbf{Port-Hamiltonian mapping}\\[3pt]
             \textcolor{black!70}{State:}\; $q = q_C$ (charge),\; $p = Li$ (flux)\\[2pt]
             \textcolor{blue!60!black}{$V(q) = \dfrac{q_C^2}{2C}$}\; (capacitive energy)\\[4pt]
             \textcolor{green!55!black}{$M = L$}\; (inductance $=$ ``mass'')\\[2pt]
             \textcolor{red!65}{$D = R$}\; (resistance $=$ damping)};

        % Hamiltonian equation
        \node[draw=blue!40, fill=blue!5, rounded corners=2pt, font=\scriptsize,
              inner sep=4pt, align=center] at (10.5,-2.3)
            {$H(q,p) = \textcolor{blue!60!black}{\dfrac{q_C^2}{2C}}
              + \textcolor{green!55!black}{\dfrac{p^2}{2L}}
              = \textcolor{blue!60!black}{\dfrac{q_C^2}{2C}}
              + \textcolor{green!55!black}{\dfrac{Li^2}{2}}$};

        % Assumptions
        \node[draw=black!30, fill=black!3, rounded corners=2pt, font=\tiny,
              inner sep=3pt, align=left] at (2.9,-3.4)
            {\textbf{Assumptions:} ideal lumped elements $\cdot$ no mutual inductance
             $\cdot$ $L{=}1,\,C{=}1,\,R{=}0.5$};
    \end{tikzpicture}
    \caption{\textbf{Series RLC circuit}~\citep{horowitz2015art}.
    Charge $q_C$ on the capacitor is the generalized coordinate; flux linkage $p = Li$ is the conjugate momentum.
    Color coding shows the port-Hamiltonian mapping:
    \textcolor{blue!60!black}{capacitance $C$} $\to$ potential $V$,
    \textcolor{green!55!black}{inductance $L$} $\to$ mass $M$,
    \textcolor{red!65}{resistance $R$} $\to$ damping $D$.
    The electrical energy $H = q_C^2/2C + Li^2/2$ is the storage function.}
    \label{fig:rlc_detail}
\end{figure}

\paragraph{Series RLC circuit (KNOWN, electrical).}
The series RLC circuit is perhaps the cleanest non-mechanical
port-Hamiltonian system: every textbook component maps one-to-one onto
the $(V, M, D)$ triple, making it an ideal sanity check for
structure-preserving learning in a domain where the analogy is
exact rather than approximate~\citep{horowitz2015art}.
\begin{itemize}
  \item \emph{State}: $q = q_C$ (charge on the capacitor), $p = Li$
    (magnetic flux linkage through the inductor, conjugate to~$q_C$).
    Current $i = \dot q_C$ is the generalized velocity.
  \item \emph{Potential}: $\textcolor{blue!60!black}{V(q) = q_C^2 / 2C}$
    --- electrostatic energy stored in the capacitor.
    The restoring ``force'' $-\partial V/\partial q_C = -q_C/C$ is
    simply the capacitor voltage, driving current back around the loop.
  \item \emph{Mass}: $\textcolor{green!55!black}{\Mmat = L}$ ---
    inductance plays the role of inertia.
    Just as a heavy mass resists changes in velocity, a large inductance
    resists changes in current (Lenz's law).
  \item \emph{Damping}: $\textcolor{red!65}{\Dmat = R}$ ---
    resistance dissipates electrical energy as Joule heat, exactly
    analogous to viscous friction converting kinetic energy to thermal
    energy.
  \item \emph{Identifiability}: With $V$ and $L$ given, the resistance
    $R$ is uniquely recoverable from the oscillation decay envelope.
\end{itemize}

\begin{figure}[H]
    \centering
    \begin{tikzpicture}[scale=1.0]
        % Thermal body 1 — blue border for potential (stored energy)
        \draw[very thick, blue!60!black, fill=red!12, rounded corners=4pt] (0,0) rectangle (3.0,2.2);
        \node[font=\large, blue!70!black] at (1.5,1.3) {$T_1$};
        \node[font=\scriptsize, black!60] at (1.5,0.65) {heat capacity $c_1$};
        \node[font=\tiny, blue!50!black] at (1.5,0.25)
            {$\textcolor{blue!60!black}{V_1 = \tfrac{1}{2}c_1 T_1^2}$};

        % Mass annotation above body 1
        \node[font=\scriptsize, green!55!black] at (1.5,2.75)
            {$\textcolor{green!55!black}{M_1 = \tau_1}$};
        \draw[->, thick, green!55!black] (1.5,2.55) -- (1.5,2.25);

        % Thermal body 2 — blue border for potential (stored energy)
        \draw[very thick, blue!60!black, fill=blue!12, rounded corners=4pt] (6.0,0) rectangle (9.0,2.2);
        \node[font=\large, blue!70!black] at (7.5,1.3) {$T_2$};
        \node[font=\scriptsize, black!60] at (7.5,0.65) {heat capacity $c_2$};
        \node[font=\tiny, blue!50!black] at (7.5,0.25)
            {$\textcolor{blue!60!black}{V_2 = \tfrac{1}{2}c_2 T_2^2}$};

        % Mass annotation above body 2
        \node[font=\scriptsize, green!55!black] at (7.5,2.75)
            {$\textcolor{green!55!black}{M_2 = \tau_2}$};
        \draw[->, thick, green!55!black] (7.5,2.55) -- (7.5,2.25);

        % Coupling arrows — red for dissipation
        \draw[->, very thick, red!65] (3.1,1.5) -- (5.9,1.5);
        \node[above, font=\small, red!65] at (4.5,1.5) {$\dot Q_{12}$};

        \draw[<-, very thick, red!65] (3.1,0.7) -- (5.9,0.7);
        \node[below, font=\scriptsize, red!60] at (4.5,0.55)
            {$\textcolor{red!65}{D = \kappa}$:\; $\dot Q = \kappa(T_1 - T_2)$};

        % Individual heat loss arrows to environment (downward)
        \draw[->, thick, red!45, decorate, decoration={snake, amplitude=0.4mm, segment length=2.5mm}]
            (1.5,-0.1) -- (1.5,-0.8);
        \node[below, font=\tiny, red!45] at (1.5,-0.8) {loss $\kappa_{\mathrm{loss}} T_1$};

        \draw[->, thick, red!45, decorate, decoration={snake, amplitude=0.4mm, segment length=2.5mm}]
            (7.5,-0.1) -- (7.5,-0.8);
        \node[below, font=\tiny, red!45] at (7.5,-0.8) {loss $\kappa_{\mathrm{loss}} T_2$};

        % Hamiltonian equation box
        \node[draw=blue!40, fill=blue!5, rounded corners=2pt, font=\scriptsize,
              inner sep=4pt, align=center] at (4.5,-1.7)
            {$H = \textcolor{blue!60!black}{\tfrac{1}{2}c_1 T_1^2 + \tfrac{1}{2}c_2 T_2^2}
              + \textcolor{green!55!black}{\tfrac{p_1^2}{2\tau_1} + \tfrac{p_2^2}{2\tau_2}}$
              \quad where $q{=}(T_1,T_2)$, $p_i{=}\tau_i\dot T_i$};

        % Assumptions box
        \node[draw=black!30, fill=black!3, rounded corners=2pt, font=\tiny,
              inner sep=3pt, align=left] at (4.5,-2.6)
            {\textbf{Assumptions:} lumped-parameter model (spatially uniform $T_i$)
             $\cdot$ two coupled thermal masses\\
             $\cdot$ $c_1{=}c_2{=}1$,\, coupling $\kappa{=}0.5$,\, loss $\kappa_{\mathrm{loss}}{=}0.1$
             $\cdot$ linear Newton cooling};
    \end{tikzpicture}
    \caption{\textbf{Lumped heat exchange}~\citep{incropera2007fundamentals}.
    Two thermal masses with temperatures $T_1$ and $T_2$ are the generalized coordinates ($n_{\mathrm{dof}}{=}2$).
    Stored thermal energy \textcolor{blue!60!black}{$V_i = \tfrac{1}{2}c_i T_i^2$} is the potential for each mass;
    the thermal time constants \textcolor{green!55!black}{$\tau_i$} play the role of inertia;
    and the coupling coefficient \textcolor{red!65}{$\kappa$} provides mutual heat exchange
    while individual losses $\kappa_{\mathrm{loss}}$ dissipate energy to the environment.
    This is a lumped-parameter model: temperature is assumed spatially uniform within each thermal mass
    (valid at low Biot number).}
    \label{fig:heat_exchange_detail}
\end{figure}
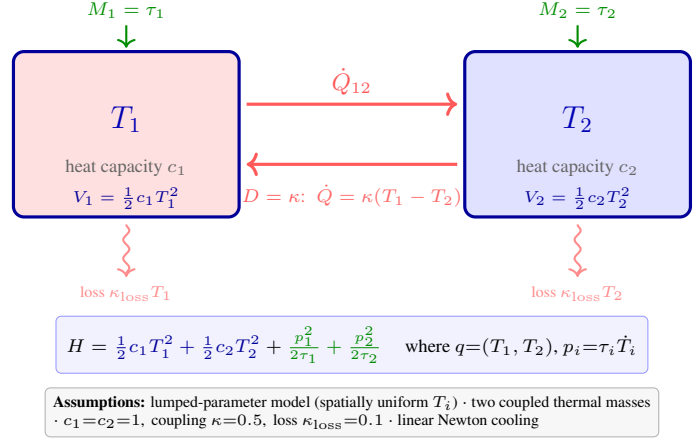

\paragraph{Heat exchange (KNOWN, thermal).}
This environment tests whether a port-Hamiltonian learner can
handle a system where ``position'' is temperature and ``momentum''
has no mechanical meaning at all~\citep{incropera2007fundamentals}.
Two thermal masses exchange heat with each other through a
coupling coefficient~$\kappa$ and independently lose heat to the
environment at rate~$\kappa_{\mathrm{loss}}$.
The key analogy: temperature differences create a thermodynamic
``force'' (like a compressed spring), while thermal time constants
resist instantaneous temperature change (like inertia resists
acceleration).
The lumped-parameter assumption---each mass has a single, spatially
uniform temperature---is valid when internal conduction is fast
relative to surface transfer (low Biot number).
\begin{itemize}
  \item \emph{State}: $q = (T_1, T_2)$ (temperatures of the two masses),
    $p = (\tau_1\dot T_1, \tau_2\dot T_2)$ (thermal ``momenta'').
    The Legendre structure $\partial H/\partial p_i = p_i/\tau_i$
    gives a thermal ``velocity''---the rate of temperature change.
  \item \emph{Potential}: $\textcolor{blue!60!black}{V(T_1,T_2)
    = \tfrac{1}{2}c_1\,T_1^2 + \tfrac{1}{2}c_2\,T_2^2
    + \tfrac{1}{2}\kappa(T_1 - T_2)^2}$ ---
    the first two terms store thermal energy in each mass individually;
    the coupling term $\tfrac{1}{2}\kappa(T_1{-}T_2)^2$ acts like a
    spring connecting the two temperatures, driving heat from hot to
    cold whenever $T_1 \neq T_2$.
  \item \emph{Mass}: $\textcolor{green!55!black}{\Mmat = \mathrm{diag}(\tau_1, \tau_2) = I}$
    --- thermal time constants play the role of inertia: a large $\tau$
    means the mass responds sluggishly to energy input, just as a heavy
    object accelerates slowly under force.
    Set to unity in our experiments.
  \item \emph{Damping}: $\textcolor{red!65}{\Dmat = \kappa_{\mathrm{loss}} I}$
    --- heat loss to the environment.
    This is \emph{distinct} from the inter-body coupling~$\kappa$:
    coupling redistributes energy between masses, while
    $\kappa_{\mathrm{loss}}$ removes it from the system entirely.
  \item \emph{Identifiability}: With $V$ and $\Mmat$ given,
    $\kappa_{\mathrm{loss}}$ is uniquely recoverable---the only
    unknown is how fast energy leaves the system.
\end{itemize}

% ---- PARTIAL regime ----
\paragraph{PARTIAL regime.}
Here the functional form of $V$ or $\Mmat$ is known (e.g., from first principles or CAD), but some parameters must be learned from data.
Anchoring part of the model breaks gauge freedom and makes the remaining unknowns identifiable.

\paragraph{Pendulum with unknown mass (PARTIAL, mechanical).}
This variant of the simple pendulum (Fig.~\ref{fig:pendulum_detail})
tests a core prediction of our identifiability theory: when the
potential~$V$ is anchored by physics, the mass--damping gauge freedom
is broken and both become recoverable.
The setup is deliberately minimal---a single unknown scalar~$m$---so
that success or failure is unambiguous.
\begin{itemize}
  \item \emph{State}: $q = \theta$ (angle from vertical),
    $p = m\ell^2\dot\theta$ (angular momentum).
  \item \emph{Potential}: $V(\theta) = -g\cos\theta$ --- given from
    gravity.
    Because $mg\ell$ is absorbed into a single effective parameter~$g$,
    the potential shape is fully specified independent of the unknown
    mass.
  \item \emph{Mass}: $\Mmat = m\ell^2$ --- a single positive scalar to
    be learned.
    This is the simplest possible PARTIAL scenario: one unknown
    inertia parameter.
  \item \emph{Damping}: Learned with bounded strength.
  \item \emph{Identifiability}: Anchoring $V$ is the key.
    Without it, any rescaling $V \to \alpha V$, $\Mmat \to \alpha\Mmat$
    produces identical trajectories (gauge freedom).
    Fixing $V$ breaks this symmetry, making both $m$ and $\Dmat$
    uniquely recoverable.
\end{itemize}

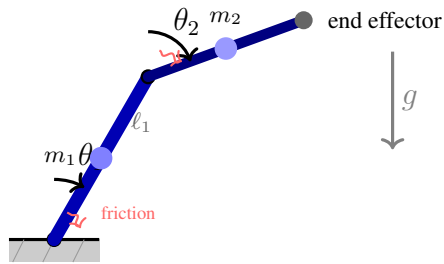
\begin{figure}[H]
    \centering
    \begin{tikzpicture}[scale=1.0]
        % Base
        \fill[black!20] (-0.6,-0.3) rectangle (0.6,0);
        \draw[very thick] (-0.6,0) -- (0.6,0);
        \foreach \x in {-0.4,0,0.4}
            \draw[black!40] (\x,0) -- (\x-0.15,-0.3);
        \fill[black] (0,0) circle (0.1);

        % Link 1
        \def\tOne{60}
        \coordinate (J1) at (0,0);
        \coordinate (E1) at (\tOne:2.5);
        \draw[line width=5pt, blue!70!black, line cap=round] (J1) -- (E1);

        % Joint 1 label
        \draw[->, very thick] (0,0.8) arc (90:\tOne:0.8);
        \node[font=\large] at (0.5,1.1) {$\theta_1$};

        % Mass 1
        \coordinate (M1) at (\tOne:1.25);
        \fill[blue!50] (M1) circle (0.15);
        \node[left, font=\small] at ($(M1)+(-0.2,0)$) {$m_1$};

        % Link 1 length
        \node[right, font=\small, black!60] at (\tOne:1.8) {$\ell_1$};

        % Joint 2
        \fill[black] (E1) circle (0.1);

        % Link 2
        \def\tTwo{20}
        \coordinate (E2) at ($(E1)+(\tTwo:2.2)$);
        \draw[line width=4pt, blue!50!black, line cap=round] (E1) -- (E2);

        % Joint 2 angle
        \draw[->, very thick] ($(E1)+(90:0.6)$) arc (90:\tTwo:0.6);
        \node[font=\large] at ($(E1)+(55:0.9)$) {$\theta_2$};

        % Mass 2
        \coordinate (M2) at ($(E1)!0.5!(E2)$);
        \fill[blue!40] (M2) circle (0.15);
        \node[above, font=\small] at ($(M2)+(0,0.2)$) {$m_2$};

        % End effector
        \fill[black!60] (E2) circle (0.12);
        \node[right, font=\small] at ($(E2)+(0.2,0)$) {end effector};

        % Gravity
        \draw[->, very thick, black!50] (4.5,2.5) -- (4.5,1.2);
        \node[right, black!50, font=\large] at (4.5,1.85) {$g$};

        % Friction at joints
        \draw[<-, thick, red!60, decorate, decoration={snake, amplitude=0.4mm, segment length=2mm}]
            (0.35,0.15) arc (20:\tOne:0.4);
        \node[red!60, font=\scriptsize, right] at (0.5,0.35) {friction};

        \draw[<-, thick, red!60, decorate, decoration={snake, amplitude=0.4mm, segment length=2mm}]
            ($(E1)+(0.45,0.2)$) -- ($(E1)+(0.15,0.35)$);
    \end{tikzpicture}
    \caption{\textbf{Robot arm (2-DOF manipulator).} Joint angles $\theta_1, \theta_2$ are generalized coordinates.  Link inertias and gravity torques are known from CAD; joint friction (red) is learned from trajectory data.}
    \label{fig:robot_arm_detail}
\end{figure}

\paragraph{Robot arm (PARTIAL, robotics).}
Industrial manipulators are a compelling PARTIAL-regime test case:
CAD models provide excellent link inertias and gravity torques, but
joint friction---the dominant source of trajectory error in
practice---must be calibrated from motion data.
This is exactly the scenario our framework targets: anchor the
well-understood physics ($V$, $\Mmat$) and let the learner focus on
the residual dissipation.
The robot arm is also the first system in our benchmark with a
\emph{configuration-dependent} mass matrix, where off-diagonal
entries couple different joints through Coriolis and centrifugal
effects.
\begin{itemize}
  \item \emph{State}: $q = (\theta_1, \ldots, \theta_n)$ (joint angles),
    $p = \Mmat(q)\dot q$ (generalized momenta).
  \item \emph{Potential}: $V(q)$ is a gravity regressor---a linear
    combination of known trigonometric basis functions weighted by
    link-mass--length products:
    $V(q) = \sum_i m_i g\, \ell_{\mathrm{cm},i}\sin(\cdot)$.
    The functional form comes from rigid-body kinematics; parameters
    may be refined from data.
  \item \emph{Mass}: $\Mmat(q)$ is the $n \times n$
    configuration-dependent inertia tensor, computed from CAD link
    inertias via the recursive Newton--Euler equations.
    Unlike the scalar-mass systems above, $\Mmat(q)$ varies with
    configuration and has non-zero off-diagonal entries that encode
    how motion at one joint creates torques at another.
  \item \emph{Damping}: $\Dmat(q)$ encodes joint friction---typically
    diagonal, with separate Coulomb + viscous friction per joint,
    learned from trajectory data.
  \item \emph{Identifiability}: With both $V$ and $\Mmat$ anchored by
    CAD, the remaining friction parameters are fully identifiable.
\end{itemize}

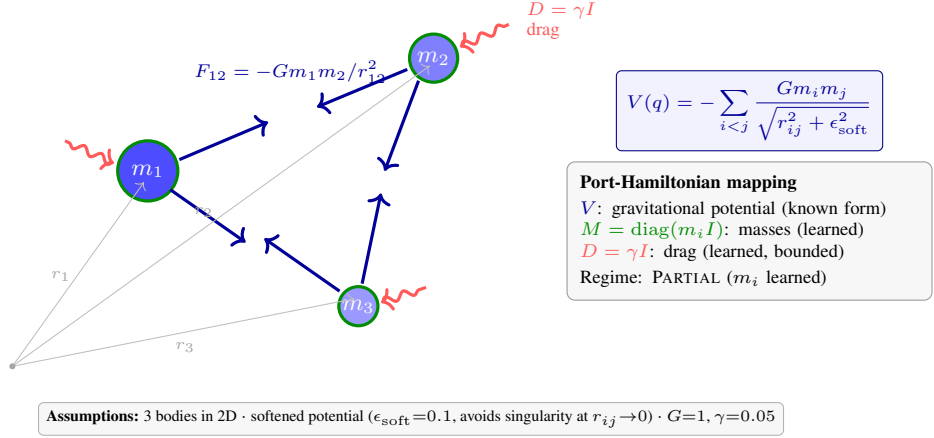
\begin{figure}[H]
    \centering
    \begin{tikzpicture}[scale=1.0]
        % Three bodies — sized proportional to mass, green borders for "mass" component
        \draw[very thick, green!55!black, fill=blue!70] (0,0) circle (0.40);
        \node[white, font=\small\bfseries] at (0,0) {$m_1$};

        \draw[very thick, green!55!black, fill=blue!50] (3.8,1.5) circle (0.32);
        \node[white, font=\small\bfseries] at (3.8,1.5) {$m_2$};

        \draw[very thick, green!55!black, fill=blue!40] (2.8,-1.8) circle (0.26);
        \node[white, font=\small\bfseries] at (2.8,-1.8) {$m_3$};

        % Gravitational force arrows — blue for conservative potential
        \draw[->, very thick, blue!60!black] (0.42,0.15) -- (1.6,0.65);
        \draw[->, very thick, blue!60!black] (3.45,1.35) -- (2.25,0.85);
        \node[above, font=\scriptsize, blue!60!black] at (1.9,1.05)
            {$F_{12} = -Gm_1m_2/r_{12}^2$};

        \draw[->, very thick, blue!60!black] (0.3,-0.25) -- (1.3,-0.95);
        \draw[->, very thick, blue!60!black] (2.55,-1.6) -- (1.55,-0.9);

        \draw[->, very thick, blue!60!black] (3.6,1.2) -- (3.15,0.0);
        \draw[->, very thick, blue!60!black] (2.85,-1.5) -- (3.1,-0.3);

        % Drag annotation — red wavy arrows for damping
        \draw[<-, very thick, red!65, decorate, decoration={snake, amplitude=0.5mm, segment length=2.5mm}]
            (-0.45,0.1) -- (-1.1,0.4);
        \draw[<-, very thick, red!65, decorate, decoration={snake, amplitude=0.5mm, segment length=2.5mm}]
            (4.15,1.65) -- (4.8,1.95);
        \draw[<-, very thick, red!65, decorate, decoration={snake, amplitude=0.5mm, segment length=2.5mm}]
            (3.1,-1.75) -- (3.7,-1.55);
        \node[red!65, font=\scriptsize, align=left] at (5.5,2.0)
            {$\textcolor{red!65}{D = \gamma I}$\\[-0.1em]\scriptsize drag};

        % Position vectors from origin
        \draw[->, thin, black!25] (-1.8,-2.6) -- (-0.05,-0.15);
        \node[left, font=\tiny, black!35] at (-0.9,-1.4) {$r_1$};
        \draw[->, thin, black!25] (-1.8,-2.6) -- (3.75,1.4);
        \node[left, font=\tiny, black!35] at (1.0,-0.55) {$r_2$};
        \draw[->, thin, black!25] (-1.8,-2.6) -- (2.75,-1.7);
        \node[below, font=\tiny, black!35] at (0.5,-2.15) {$r_3$};
        \fill[black!35] (-1.8,-2.6) circle (0.04);

        % Potential equation box
        \node[draw=blue!50!black, fill=blue!5, rounded corners=2pt, font=\scriptsize,
              inner sep=4pt, align=center] at (8.0,0.8)
            {$\textcolor{blue!60!black}{V(q) = -\displaystyle\sum_{i<j} \frac{Gm_im_j}{\sqrt{r_{ij}^2 + \epsilon_{\mathrm{soft}}^2}}}$};

        % pH mapping box
        \node[draw=black!40, fill=black!3, rounded corners=3pt, font=\scriptsize,
              inner sep=5pt, align=left, text width=4.5cm] at (8.0,-0.8)
            {\textbf{Port-Hamiltonian mapping}\\[2pt]
             \textcolor{blue!60!black}{$V$}: gravitational potential (known form)\\
             \textcolor{green!55!black}{$M = \mathrm{diag}(m_i I)$}: masses (learned)\\
             \textcolor{red!65}{$D = \gamma I$}: drag (learned, bounded)\\[2pt]
             Regime: \textsc{Partial} ($m_i$ learned)};

        % Assumptions box
        \node[draw=black!30, fill=black!3, rounded corners=2pt, font=\tiny,
              inner sep=3pt, align=left] at (3.5,-3.3)
            {\textbf{Assumptions:} 3 bodies in 2D $\cdot$ softened potential ($\epsilon_{\mathrm{soft}}{=}0.1$, avoids singularity at $r_{ij}{\to}0$)
             $\cdot$ $G{=}1$, $\gamma{=}0.05$};
    \end{tikzpicture}
    \caption{\textbf{N-body gravitational system}~\citep{binney2008galactic}.
    $N{=}3$ point masses interact pairwise via Newtonian gravity.
    The \textcolor{blue!60!black}{gravitational potential} $V$ is known in functional form but individual
    \textcolor{green!55!black}{masses $m_i$} are learned (PARTIAL regime).
    A softened potential $r_{ij}^2 + \epsilon_{\mathrm{soft}}^2$ replaces $r_{ij}^2$ to avoid
    the singularity at close approach.
    \textcolor{red!65}{Drag $D = \gamma I$} models residual dissipation (e.g., tidal friction or numerical damping);
    it is learned with tight bounds since point-mass gravitational systems are nearly conservative.}
    \label{fig:nbody_detail}
\end{figure}

\paragraph{N-body gravity (PARTIAL, astrophysics).}
Gravitational N-body problems are the prototypical example of a system
where the functional form of the physics is known (Newton's law) but
the parameters---individual masses---must be inferred from observed
trajectories~\citep{binney2008galactic}.
This makes N-body an ideal PARTIAL-regime benchmark: the learner is
given the gravitational template and must fill in the mass values,
a task that is meaningful in astrophysics (e.g., inferring stellar
masses from orbital data).
We use Plummer softening to regularize the $r{\to}0$ singularity, a
standard technique in computational astrophysics that replaces the
point-particle divergence with a smooth core.
\begin{itemize}
  \item \emph{State}: $q = (r_1, \ldots, r_N) \in \R^{2N}$ (positions in 2D),
    $p = (m_1\dot r_1, \ldots, m_N\dot r_N)$.
  \item \emph{Potential}: $\textcolor{blue!60!black}{V(q) = -\sum_{i<j} Gm_im_j / \sqrt{r_{ij}^2 + \epsilon_{\mathrm{soft}}^2}}$ ---
    the gravitational potential is known in functional form, but the
    mass products $m_im_j$ are unknown parameters to be learned.
    This is what makes the system PARTIAL rather than KNOWN: the
    \emph{shape} of $V$ is given, but its \emph{strength} is not.
  \item \emph{Mass}: $\textcolor{green!55!black}{\Mmat = \mathrm{diag}(m_1 I_2, \ldots, m_N I_2)}$ ---
    individual masses appear both here and in $V$, coupling the
    two components.
    Each $m_i$ is learned as a positive scalar.
  \item \emph{Damping}: $\textcolor{red!65}{\Dmat = \gamma I}$ ---
    a small drag modeling tidal dissipation or dynamical friction.
    For isolated point masses this is nearly zero; we learn it with
    tight bounds to test whether PHAST can recover a weak signal.
  \item \emph{Identifiability}: The masses $m_i$ appear as products
    $m_im_j$ in $V$ and individually in $\Mmat$, so they are jointly
    identifiable up to an overall scale.
    Fixing $G$ resolves this gauge freedom.
\end{itemize}

\paragraph{General partially known system (PARTIAL, template).}
This row in Table~\ref{tab:structured_hamiltonians} represents the
generic PARTIAL recipe:
\begin{itemize}
  \item \emph{Potential}: $V(q) = \bar V(q) + \varepsilon\,\tilde V(q)$,
    where $\bar V$ is a known physics template and
    $\tilde V$ is a neural residual bounded by $\varepsilon$.
    This allows the model to correct for unmodeled effects (e.g.,
    spring nonlinearities, thermal expansion) while staying close to
    the physics prior.
  \item \emph{Mass}: $\Mmat = m_0 I + UU^\top$ --- a diagonal-plus-low-rank
    SPD parameterization, where $m_0 > 0$ is a known base inertia and
    $UU^\top$ captures learned corrections.
  \item \emph{Damping}: Learned with bounded strength
    ($\sum_i \beta_i \le \bar\beta$, Eq.~\ref{eq:damping_bound}).
  \item \emph{Identifiability}: Partial --- depends on how much of $V$
    and $\Mmat$ is anchored.
\end{itemize}

% ---- UNKNOWN regime ----
\paragraph{UNKNOWN regime.}
When no physics is available, all three components $(V, \Mmat, \Dmat)$ are learned from data.
The port-Hamiltonian structure still guarantees passivity ($dH/dt \le 0$), which stabilizes long-horizon rollouts, but the recovered parameters are no longer uniquely identifiable (see the gauge-freedom discussion in Appendix~\ref{app:ablations:gauge_freedom}).

\paragraph{Black-box dynamics (UNKNOWN, general).}
When no physics is available, all three components are learned:
\begin{itemize}
  \item \emph{Potential}: $V(q) = f_\theta(q)$, a neural network with
    no structural constraint beyond smoothness.
  \item \emph{Mass}: $\Mmat(q)$ is parameterized as an SPD neural network
    (diagonal-plus-low-rank, Eq.~\ref{eq:householder_mass}).
  \item \emph{Damping}: $\Dmat(q)$ is parameterized as a PSD neural network
    (Eq.~\ref{eq:householder_damping}).
  \item \emph{Identifiability}: No --- the inverse problem is
    underdetermined.
    Multiple $(V, \Mmat, \Dmat)$ triples can generate identical
    trajectories (gauge freedom,
    Appendix~\ref{app:ablations:gauge_freedom}).
    Forecasting may still be accurate, but recovered parameters lack
    unique physical meaning.
\end{itemize}

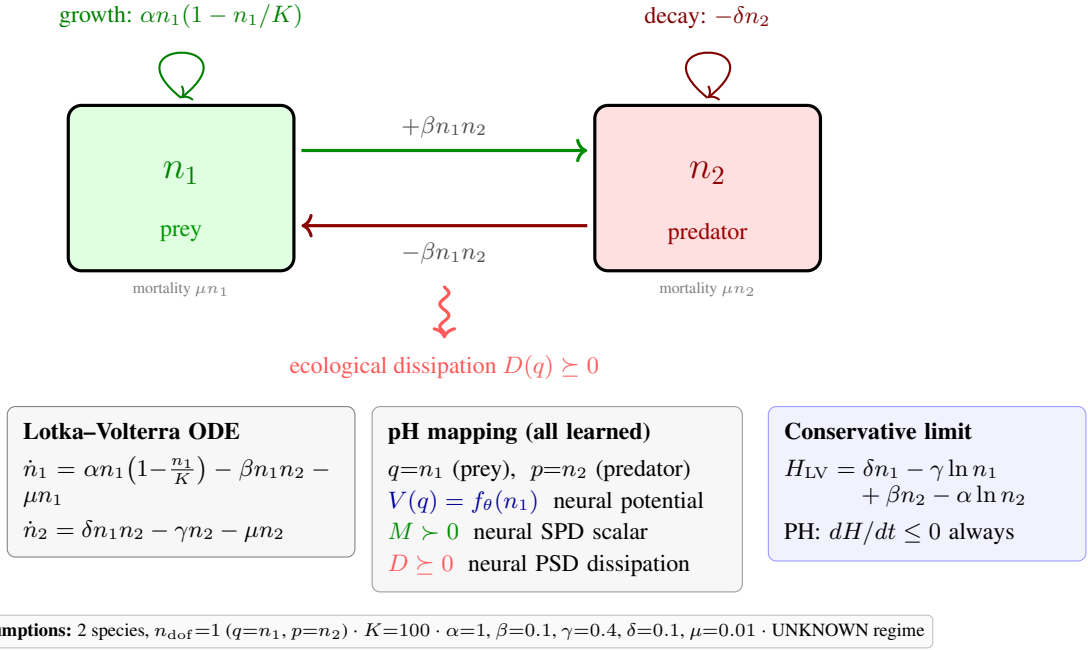
\begin{figure}[H]
    \centering
    \begin{tikzpicture}[scale=1.0]
        % ============================================================
        % Row 1: Ecosystem schematic (centered, width ~10cm)
        % ============================================================

        % Prey population box
        \draw[very thick, fill=green!12, rounded corners=5pt] (0,0) rectangle (3.0,2.2);
        \node[font=\Large, green!50!black] at (1.5,1.3) {$n_1$};
        \node[font=\small, green!60!black] at (1.5,0.5) {prey};

        % Predator population box
        \draw[very thick, fill=red!12, rounded corners=5pt] (7.0,0) rectangle (10.0,2.2);
        \node[font=\Large, red!55!black] at (8.5,1.3) {$n_2$};
        \node[font=\small, red!60!black] at (8.5,0.5) {predator};

        % Interaction arrows — well separated from boxes
        \draw[->, very thick, green!55!black] (3.1,1.6) -- (6.9,1.6);
        \node[above, font=\small, black!70] at (5.0,1.65) {$+\beta n_1 n_2$};

        \draw[->, very thick, red!55!black] (6.9,0.6) -- (3.1,0.6);
        \node[below, font=\small, black!70] at (5.0,0.5) {$-\beta n_1 n_2$};

        % Self-growth (prey) with carrying capacity
        \draw[->, thick, green!50!black] (1.5,2.3) .. controls (0.5,3.1) and (2.5,3.1) .. (1.5,2.3);
        \node[above, font=\small, green!50!black] at (1.5,3.1)
            {growth: $\alpha n_1(1 - n_1/K)$};

        % Self-decay (predator)
        \draw[->, thick, red!50!black] (8.5,2.3) .. controls (7.5,3.1) and (9.5,3.1) .. (8.5,2.3);
        \node[above, font=\small, red!50!black] at (8.5,3.1) {decay: $-\delta n_2$};

        % Dissipation annotation — red, below center
        \draw[->, very thick, red!65, decorate, decoration={snake, amplitude=0.6mm, segment length=3mm}]
            (5.0,-0.2) -- (5.0,-0.9);
        \node[below, font=\small, red!65] at (5.0,-1.0)
            {ecological dissipation $\textcolor{red!65}{D(q)\succeq 0}$};

        % Background mortality labels inside boxes
        \node[font=\tiny, black!50] at (1.5,-0.25) {mortality $\mu n_1$};
        \node[font=\tiny, black!50] at (8.5,-0.25) {mortality $\mu n_2$};

        % ============================================================
        % Row 2: Three info panels side by side (below the schematic)
        % ============================================================

        % ODE box (left)
        \node[draw=black!50, fill=black!3, rounded corners=3pt, font=\small,
              inner sep=6pt, align=left, text width=4.2cm,
              anchor=north] (odebox) at (1.5,-1.8)
            {\textbf{Lotka--Volterra ODE}\\[4pt]
             $\dot n_1 = \alpha n_1\!\left(1{-}\frac{n_1}{K}\right) - \beta n_1 n_2 - \mu n_1$\\[3pt]
             $\dot n_2 = \delta n_1 n_2 - \gamma n_2 - \mu n_2$};

        % pH mapping box (center)
        \node[draw=black!40, fill=black!3, rounded corners=3pt, font=\small,
              inner sep=6pt, align=left, text width=4.5cm,
              anchor=north] (phbox) at (6.5,-1.8)
            {\textbf{pH mapping (all learned)}\\[4pt]
             $q{=}n_1$ (prey),\; $p{=}n_2$ (predator)\\[2pt]
             \textcolor{blue!60!black}{$V(q) = f_\theta(n_1)$}\; neural potential\\[2pt]
             \textcolor{green!55!black}{$M\succ 0$}\; neural SPD scalar\\[2pt]
             \textcolor{red!65}{$D\succeq 0$}\; neural PSD dissipation};

        % Conservative Hamiltonian box (right)
        \node[draw=blue!40, fill=blue!5, rounded corners=3pt, font=\small,
              inner sep=6pt, align=left, text width=3.8cm,
              anchor=north west] (consbox) at (9.3,-1.8)
            {\textbf{Conservative limit}\\[4pt]
             $H_{\mathrm{LV}} = \delta n_1 - \gamma\ln n_1$\\
             $\phantom{H_{\mathrm{LV}} =}{}+ \beta n_2 - \alpha\ln n_2$\\[4pt]
             {\footnotesize PH: $dH/dt \le 0$ always}};

        % Assumptions line at the very bottom
        \node[draw=black!25, fill=black!2, rounded corners=2pt, font=\scriptsize,
              inner sep=3pt, align=left] at (5.0,-4.8)
            {\textbf{Assumptions:} 2 species, $n_{\mathrm{dof}}{=}1$ ($q{=}n_1$, $p{=}n_2$) $\cdot$ $K{=}100$ $\cdot$ $\alpha{=}1$, $\beta{=}0.1$, $\gamma{=}0.4$, $\delta{=}0.1$, $\mu{=}0.01$ $\cdot$ UNKNOWN regime};
    \end{tikzpicture}
    \caption{\textbf{Predator--prey (Lotka--Volterra)}~\citep{murray2002mathematical,strogatz2018nonlinear}.
    Prey density $n_1$ is the generalized coordinate~$q$; predator density $n_2$ serves as the conjugate momentum~$p$ ($n_{\mathrm{dof}}{=}1$).
    All port-Hamiltonian components are learned (UNKNOWN regime):
    \textcolor{blue!60!black}{$V(q)$} is a neural potential generalizing the classical Lotka--Volterra conserved quantity,
    \textcolor{green!55!black}{$M$} is a learned SPD ``ecological inertia,'' and
    \textcolor{red!65}{$D$} is a learned PSD dissipation encoding mortality, disease, and resource depletion.
    The port-Hamiltonian structure guarantees bounded dynamics even without parameter recovery.}
    \label{fig:predator_prey_detail}
\end{figure}

\paragraph{Predator--prey (UNKNOWN, ecology).}
This is the most challenging benchmark in our suite: an ecological
system with no known physics, where the port-Hamiltonian framework must
be justified purely by its structural benefits rather than by any
mechanical analogy~\citep{murray2002mathematical,strogatz2018nonlinear}.
The classical Lotka--Volterra model admits a conserved quantity
$H = \delta n_1 - \gamma\ln n_1 + \beta n_2 - \alpha\ln n_2$ in the
conservative limit, hinting that a Hamiltonian viewpoint is natural.
Our simulator breaks conservation with a carrying-capacity term
$\alpha n_1(1{-}n_1/K)$ and background mortality~$\mu$, making the
dynamics dissipative---exactly the setting port-Hamiltonian structure
is designed for.
The system has $n_{\mathrm{dof}} = 1$: prey density~$n_1$ is the
generalized coordinate and predator density~$n_2$ serves as the
conjugate momentum.
This pairing has no interpretation as mass~$\times$~velocity, but it
is not arbitrary: it lets the skew-symmetric~$J$ matrix couple prey
growth to predator response (and vice versa), while the
positive-semidefinite~$D$ captures all irreversible losses.
The guarantee $dH/dt \le 0$ then enforces bounded population
dynamics---populations cannot diverge---even though no physics is
provided.
\begin{itemize}
  \item \emph{State}: $q = n_1$ (prey density), $p = n_2$ (predator
    density).
    The $(q,p)$ pairing is non-standard but structurally motivated:
    it lets the pH framework enforce that predator--prey oscillations
    remain bounded.
  \item \emph{Potential}: $\textcolor{blue!60!black}{V(q) = f_\theta(n_1)}$
    --- a neural network of prey density alone, generalizing the
    prey-dependent part of the classical conserved quantity.
    The predator-dependent terms are absorbed into the Hamiltonian's
    kinetic part $p^2/2M$.
  \item \emph{Mass}: $\textcolor{green!55!black}{M(q)}$ --- a learned
    SPD scalar that controls how sensitively predator density (the
    ``momentum'') responds to changes in the ecological potential.
    Larger $M$ means more sluggish predator response, analogous to
    heavier inertia.
  \item \emph{Damping}: $\textcolor{red!65}{D(q)}$ --- a learned PSD
    scalar encoding irreversible ecological losses: disease,
    emigration, starvation, and resource depletion.
  \item \emph{Identifiability}: No --- the same gauge freedom as
    any fully black-box system applies.
    But identifiability is not the goal here; the value of pH
    structure is \emph{stability}: long-horizon rollouts remain
    physical (bounded populations) where unconstrained models
    diverge.
\end{itemize}

% ===========================================================================
% C. Architecture Details
% ===========================================================================
\section{Architecture Details}
\label{app:arch}

\subsection{Terminology: Depth vs.\ Substeps vs.\ Blocks}
\label{app:arch:terminology}

The word \emph{layer} is overloaded in the literature; we disambiguate:
\begin{itemize}
    \item \textbf{Observer depth} (FD+TCN layers) refers to the number of layers in the causal observer $o_\phi$ in Eq.~\eqref{eq:qonly_observer}.
    \item \textbf{PHAST substeps} $L$ refer to numerical \emph{substepping} of the same port-Hamiltonian update $\Phi_{\dt}$ within one environment step.
    All substeps reuse the same learned modules $(V,\Mmat,\Dmat)$ (parameters shared); only the intermediate states differ:
    \begin{equation}
      x_{t+1} = \underbrace{\Phi_{\dt/L} \circ \cdots \circ \Phi_{\dt/L}}_{L\ \text{substeps}}(x_t).
      \label{eq:phast_substeps}
    \end{equation}
    \item \textbf{PHAST blocks} (optional) refer to a sequence-model wrapper (projections + residual) around the PHAST transition; our q-only experiments unroll $\Phi_{\dt}$ directly and do not stack PHAST blocks.
\end{itemize}

\subsection{Velocity Computation}
\label{app:arch:velocity}

PHAST supports a general configuration-dependent mass $\Mmat(q)\succ 0$;
in our main experiments we use a constant-mass approximation $\Mmat(q)\approx\Mmat$
for efficiency.
In the UNKNOWN regime, PHAST learns this (constant) mass matrix $\Mmat \succ 0$ with a diagonal-plus-low-rank parameterization:
\begin{equation}
  \Mmat = \mathrm{diag}(d) + \sum_{i=1}^{r}\alpha_i\,k_i k_i^\T,
  \qquad d_j>0\ (j=1,\ldots,n),\ \alpha_i\ge 0.
  \label{eq:householder_mass}
\end{equation}
Equivalently, letting $\Lambda=\mathrm{diag}(d)$ and $U=[\sqrt{\alpha_1}k_1,\ldots,\sqrt{\alpha_r}k_r]$, we have $\Mmat=\Lambda+UU^\T$.
We compute $v=\Mmat^{-1}p$ via the Woodbury identity (Alg.~\ref{alg:velocity}).
The full Strang-splitting update that uses this primitive is given in Algorithm~\ref{alg:phast_step} (main paper).

\begin{algorithm}[t]
\caption{Compute $v = \Mmat^{-1}p$ via Woodbury}
\label{alg:velocity}
\begin{algorithmic}[1]
\REQUIRE $p \in \R^n$, diagonal $d \in \R_{>0}^n$ defining $\Lambda=\mathrm{diag}(d)$, and $U\in\R^{n\times r}$ such that $\Mmat=\Lambda+UU^\T$
\STATE $z \leftarrow p \oslash d$ \COMMENT{$z=\Lambda^{-1}p$;\; $O(n)$ element-wise division}
\STATE $Z \leftarrow U \oslash d$ \COMMENT{$Z=\Lambda^{-1}U$;\; $O(nr)$ column-wise division}
\STATE $A \leftarrow I_r + U^\T Z$ \COMMENT{$O(nr^2)$}
\STATE $c \leftarrow A^{-1}(U^\T z)$ \COMMENT{$O(r^3)$}
\STATE $v \leftarrow z - Zc$ \COMMENT{$O(nr)$}
\RETURN $v$
\end{algorithmic}
\end{algorithm}

\subsection{Q-Only Pipeline}
\label{app:arch:qonly}

When only positions $q_t$ are observed (common in vision-based robotics), the Markov state is $x_t=(q_t,p_t)$ but $p_t$ is unobserved.
PHAST uses a causal observer and a canonicalizer to infer a single initial canonical state from a burn-in context, then predicts open-loop with the same PHAST core transition.

\paragraph{Observer and takeover rollouts.}
We use the finite-difference + TCN observer in Eq.~\eqref{eq:qonly_observer} and infer a single initial state at the end of the burn-in window, then roll out open-loop via Eq.~\eqref{eq:qonly_takeover}.
This yields a parameter-efficient q-only pipeline that is forecasting-aligned (no measurements after burn-in).

\begin{enumerate}
    \item \textbf{Canonicalizer} $C_\psi: (q, \hat{\dot{q}}) \mapsto (q, \hat{p})$. In our q-only experiments we use the identity canonicalizer, so $\hat p=\hat{\dot q}$. When $\Mmat$ is known (and not identity), one could instead map $\hat{p} = \Mmat\,\hat{\dot{q}}$.
    \item \textbf{PHAST core transition} $\Phi_\theta: (q, \hat{p}) \to (q^+, \hat{p}^+)$. Same as full-state; project to $q^+$ for output.
\end{enumerate}

\paragraph{Interpretation of $\hat p$ (q-only).}
We use $\hat p$ to denote the inferred conjugate coordinate for notational consistency with the port-Hamiltonian template.
With identity canonicalization, $\hat p$ is a velocity-like proxy used to form an approximately Markov phase state from q-only observations (it coincides with generalized momentum only when $\Mmat=I$ up to scaling).

\begin{table}[t]
    \centering
    \caption{\textbf{What is learned vs.\ shared in the q-only PHAST pipeline.} PHAST core transition substeps $L$ reuse the same physics components $(V,\Mmat,\Dmat)$ and differ only through intermediate states (and, if enabled, per-substep timesteps). Regime choices determine which physics components are learned (Fig.~\ref{fig:phast_regime_spectrum}).}
    \label{tab:qonly_sharing}
    \small
    \begin{tabular}{@{}llll@{}}
        \toprule
        \textbf{Component} & \textbf{Parameters} & \textbf{Shared across $t$?} & \textbf{Shared across substeps $L$?} \\
        \midrule
        Observer $o_\phi$ (FD+TCN) & $\phi$ & \cmark & --- \\
        Canonicalizer $C_\psi$ & $\psi$ & \cmark & --- \\
        Physics $(V,\Mmat,\Dmat)$ & $\theta$ & \cmark & \cmark \\
        Substep timesteps $\{\delta t_s\}_{s=1}^L$ & $\tau$ & \cmark & per-substep \\
        \bottomrule
    \end{tabular}
\end{table}

\paragraph{Rollout modes.}
\begin{itemize}
    \item \textbf{Takeover (forecasting)}: infer a single initial state from burn-in, then pure integration (no measurements after burn-in).
    \item \textbf{Self-conditioned (forecasting)}: re-infer velocity at each step from predicted positions (errors compound; more observer-dependent).
    \item \textbf{Predict--correct (filtering)}: online state estimation with measurement updates (requires measurements at every step).
    \item \textbf{Feedback control}: online stabilization/tracking with inputs $u_t$ (requires measurements; see \ifdefined\isarxiv Sec.~\ref{app:casimir}\else Appendix~\ref{app:casimir}\fi).
\end{itemize}

% ===========================================================================
% D. Potential Parameterizations
% ===========================================================================
\section{Potential Energy Parameterizations}
\label{app:potentials}

\subsection{Structured Potentials (KNOWN Regime)}

\paragraph{Cosine (pendulum-like).}
\begin{equation}
    V(q) = \sum_{i=1}^{n} a_i\,(1 - \cos q_i), \qquad \grad V(q) = [a_1\sin q_1, \ldots, a_n\sin q_n]^\T.
\end{equation}

\paragraph{Quadratic (spring-like).}
\begin{equation}
    V(q) = \half\,q^\T K_s\,q, \qquad \grad V(q) = K_s\,q.
\end{equation}
where $K_s \succeq 0$ is a stiffness matrix.

\subsection{Neural Potentials (UNKNOWN Regime)}

\paragraph{MLP potential.}
\begin{equation}
    V_\eta(q) = f_\eta(q) \in \R, \qquad \grad V_\eta(q)\ \text{computed by automatic differentiation}.
\end{equation}
In our experiments, $f_\eta$ is a 3-layer MLP with SiLU activations.

\paragraph{Periodic neural potential.}
For angular coordinates, use Fourier features:
\begin{equation}
    \phi(q_i) = [\sin(q_i), \cos(q_i), \sin(2q_i), \cos(2q_i), \ldots], \qquad V_\eta(q) = f_\eta(\phi(q)).
\end{equation}

\subsection{Hybrid Potentials (PARTIAL Regime)}

\begin{equation}
    V(q) = \Vbar(q) + \varepsilon\,\Vtilde(q),
\end{equation}
where $\Vbar(q)$ is structured (given), $\Vtilde(q)$ is neural (learned), and $\varepsilon = \mathrm{softplus}(\rho_\varepsilon)$ is a learnable scale initialized small.

% ===========================================================================
% E. Training Details
% ===========================================================================
\section{Training Details}
\label{app:training}

\subsection{Loss Function Breakdown}

Let $y_{0:T-1}$ denote the observed sequence (in q-only, $y_t=q_t$), where $T$ is the sequence length. Let $\mathrm{err}(\hat y, y)$ denote the appropriate per-step error on the manifold (Appendix~\ref{app:metrics}).
For full-state training (when $x=(q,p)$ is available), we use teacher-forced one-step predictions $\hat x_{t+1}=\Phi_{\dt}(x_t)$ and define $H_t=\Ham(x_t)$ and $\hat H_{t+1}=\Ham(\hat x_{t+1})$.
We use the following optional losses:
\begin{align}
  \mathcal{L}_{\text{data}}
  &= \frac{1}{T-1}\sum_{t=0}^{T-2} \mathrm{err}(\hat y_{t+1}, y_{t+1}), \\
  \mathcal{L}_{\text{pass}}
  &= \frac{1}{T-1}\sum_{t=0}^{T-2} \max(0, \hat H_{t+1} - H_t), \\
  \mathcal{L}_{\text{energy}}
  &= \frac{1}{T-1}\sum_{t=0}^{T-2}\left|\frac{\hat H_{t+1}-H_t}{\dt} + v_t^\T \Dmat(q_t) v_t\right|,
     \quad v_t := \Mmat^{-1}p_t, \\
  \mathcal{L}_{\text{roll}}
  &= \mathbb{E}_{t_0}\left[\frac{1}{H_{\mathrm{roll}}}\sum_{h=1}^{H_{\mathrm{roll}}}\mathrm{err}(\tilde y_{t_0+h}, y_{t_0+h})\right],
     \quad \tilde x_0=x_{t_0},\ \tilde x_{h+1}=\Phi_{\dt}(\tilde x_h),\ t_0\sim\mathrm{Unif}\{0,\ldots,T-1-H_{\mathrm{roll}}\},
\end{align}
	where $\mathcal{L}_{\text{roll}}$ evaluates open-loop rollouts (no teacher forcing).
	In the q-only setting, energy diagnostics can be formed using finite-difference velocity estimates as in Appendix~\ref{app:metrics}.

	\subsection{Optional Rollout-Aligned Training (Pseudocode)}
	\label{app:training:rollout_alg}
	
	\begin{algorithm}[t]
	\caption{Optional rollout-aligned training (full-state; uses Alg.~\ref{alg:phast_step} for $\Phi_{\dt}$)}
	\label{alg:phast_train}
	\begin{algorithmic}[1]
	\REQUIRE Batch of trajectories $\{x^{(b)}_{0:T-1}\}_{b=1}^B$, rollout horizon $H_{\mathrm{roll}}$, weights $\lambda_\bullet$
	\STATE $\mathcal{L}_{\text{data}},\mathcal{L}_{\text{pass}},\mathcal{L}_{\text{energy}},\mathcal{L}_{\text{roll}} \leftarrow 0$
	\FOR{$b=1$ to $B$}
	\STATE \textbf{Teacher-forced next-step losses}
	  \FOR{$t=0$ to $T-2$}
	    \STATE $\hat x_{t+1} \leftarrow \Phi_{\dt}(x^{(b)}_t)$ \COMMENT{Alg.~\ref{alg:phast_step}}
	    \STATE $(q_t,p_t) \leftarrow x^{(b)}_t$ \COMMENT{unpack $x^{(b)}_t$}
	    \STATE $v_t \leftarrow \Mmat^{-1}p_t$
	    \STATE $\mathcal{L}_{\text{data}} \mathrel{+}= \|\hat x_{t+1} - x^{(b)}_{t+1}\|^2$
	    \STATE $\mathcal{L}_{\text{pass}} \mathrel{+}= \max(0,\Ham(\hat x_{t+1}) - \Ham(x^{(b)}_t))$
	    \STATE $\mathcal{L}_{\text{energy}} \mathrel{+}= \Big|\frac{\Ham(\hat x_{t+1})-\Ham(x^{(b)}_t)}{\dt} + v_t^\T\Dmat(q_t)v_t\Big|$
	  \ENDFOR
	  \STATE \textbf{Open-loop rollout loss}
	  \STATE Choose start index $t_0 \in \{0,\ldots,T-1-H_{\mathrm{roll}}\}$
	  \STATE $\tilde x_0 \leftarrow x^{(b)}_{t_0}$
	  \FOR{$h=1$ to $H_{\mathrm{roll}}$}
	    \STATE $\tilde x_h \leftarrow \Phi_{\dt}(\tilde x_{h-1})$ \COMMENT{Alg.~\ref{alg:phast_step}}
	    \STATE $\mathcal{L}_{\text{roll}} \mathrel{+}= \|\tilde x_h - x^{(b)}_{t_0+h}\|^2$
	  \ENDFOR
	\ENDFOR
	\STATE $\mathcal{L} \leftarrow \lambda_{\text{data}}\mathcal{L}_{\text{data}} + \lambda_{\text{pass}}\mathcal{L}_{\text{pass}} + \lambda_{\text{energy}}\mathcal{L}_{\text{energy}} + \lambda_{\text{roll}}\,\mathcal{L}_{\text{roll}}$
	\STATE Update parameters with $\nabla \mathcal{L}$
	\end{algorithmic}
	\end{algorithm}
	
	\subsection{Hyperparameters}

	\begin{table}[h]
	    \centering
	    \caption{Hyperparameters for the q-only experiments (Sec.~\ref{sec:experiments}).}
    \label{tab:hyperparams}
    \small
    \begin{tabular}{@{}lccc@{}}
        \toprule
        \textbf{Hyperparameter} & \textbf{KNOWN} & \textbf{PARTIAL} & \textbf{UNKNOWN} \\
        \midrule
        Damping terms $r$ & 2 & 2 & 2 \\
        Damping bound $\bbar$ (Windy) & unbounded & $\Delta d \,(=0.5)$ & unbounded \\
        $d_0$ policy (Windy) & learned & fixed & learned \\
        Potential hidden dim & --- & 64 & 64 \\
        Potential layers & --- & 3 & 3 \\
        Mass terms (if learned) & --- & --- & $\min(4,n)$ \\
        \midrule
        Observer & FD+TCN & FD+TCN & FD+TCN \\
        Observer hidden dim & 32 & 32 & 32 \\
        Observer layers & 2 & 2 & 2 \\
        Canonicalizer & identity & identity & identity \\
        PHAST core transition trainable & yes & yes & yes \\
        \midrule
        Optimizer & AdamW & AdamW & AdamW \\
        Learning rate & $10^{-3}$ & $10^{-3}$ & $10^{-3}$ \\
        Batch size & 64 & 64 & 64 \\
        Epochs & 50 & 50 & 50 \\
        Rollout horizons (eval) & $\{10,50,100\}$ & $\{10,50,100\}$ & $\{10,50,100\}$ \\
        Burn-in context $K$ & 10 & 10 & 10 \\
        \midrule
        $\lambda_{\text{data}}$ & 1.0 & 1.0 & 1.0 \\
        $\lambda_{\text{pass}}$ & 0.0 & 0.0 & 0.0 \\
        $\lambda_{\text{energy}}$ & 0.0 & 0.0 & 0.0 \\
        $\lambda_{\text{roll}}$ & 0.0 & 0.0 & 0.0 \\
        \bottomrule
    \end{tabular}
\end{table}

\subsection{Baseline Hyperparameters}
\label{app:training:baselines}

\begin{table}[t]
    \centering
    \caption{\textbf{Baseline architecture hyperparameters (q-only experiments).} All baselines are trained as causal seq2seq predictors mapping $q_{0:t}\mapsto \hat q_{t+1}$ with shared global settings hidden dimension $d=64$ and depth $2$ (unless otherwise noted).}
    \label{tab:baseline_hparams}
    \scriptsize
    \begin{tabular}{@{}l l@{}}
        \toprule
        \textbf{Baseline} & \textbf{Key settings} \\
        \midrule
        GRU & dropout=0.0 \\
        S5 & $d_{\text{state}}=64$, internal $\dt_{\text{ssm}}=0.01$ \\
        LinOSS & internal $\dt_{\text{linoss}}=1.0$ (absorbed into learned eigenvalues) \\
        D-LinOSS & internal state dim $=\lfloor d/2\rfloor$, learnable per-oscillator timestep \\
        Transformer & $n_{\text{heads}}=4$, FF dim $=4d$, dropout=0.1 \\
        VPT & $n_{\text{ff\_sublayers}}=2$, skew-symmetric attention \\
        \bottomrule
    \end{tabular}
\end{table}

\paragraph{Dataset parameters (Windy Pendulum).}
Trajectories are generated by the simulator with $d_0=0.3$, $\Delta d=0.5$, $\dt=0.05$, length $T=200$,
initial conditions $\theta_0 \sim \mathrm{Unif}[-\pi,\pi]$, $p_0 \sim \mathcal{N}(0, 4^2)$,
and $N_{\text{train}}/N_{\text{val}}/N_{\text{test}}=1000/200/200$.
We report mean $\pm$ std over 5 random model seeds using a fixed dataset shared across models.

\paragraph{Initial-condition distributions (all benchmarks).}
All q-only benchmarks use trajectories of length $T{=}200$ with $N_{\text{train}}/N_{\text{val}}/N_{\text{test}}=1000/200/200$ (Sec.~\ref{sec:experiments:setup}).
Table~\ref{tab:dataset_ics} summarizes the initial-condition distributions used by the simulators.

\begin{table}[t]
    \centering
    \caption{\textbf{Simulator initial conditions for the thirteen q-only benchmarks.}}
    \label{tab:dataset_ics}
    \scriptsize
    \setlength{\tabcolsep}{3pt}
    \begin{tabular}{@{}p{0.22\linewidth} c p{0.48\linewidth} p{0.24\linewidth}@{}}
        \toprule
        \textbf{Benchmark} & $\dt$ & \textbf{Initial conditions} & \textbf{Damping} \\
        \midrule
        Pendulum (cons) & 0.05 &
          $\theta_0\sim\mathrm{Unif}[-\pi,\pi]$, $p_0\sim\mathcal{N}(0,3^2)$ &
          none \\
        Pendulum (damped) & 0.05 &
          $\theta_0\sim\mathrm{Unif}[-\pi,\pi]$, $p_0\sim\mathcal{N}(0,3^2)$ &
          constant ($\gamma{=}0.5$) \\
        Pendulum (windy) & 0.05 &
          $\theta_0\sim\mathrm{Unif}[-\pi,\pi]$, $p_0\sim\mathcal{N}(0,4^2)$ &
          $d(\theta)=d_0+\Delta d|\sin\theta|$ ($d_0{=}0.3$, $\Delta d{=}0.5$) \\
        Cart-Pole (windy) & 0.02 &
          $(x_0,\theta_0)\sim\mathrm{Unif}([-1,1]\times[-\pi,\pi])$, $p_0\sim\mathrm{Unif}[-2,2]^2$ &
          windy on $\theta$ ($d_0{=}0.3$, $\Delta d{=}0.5$) \\
        Oscillator (cons) & 0.02 &
          $q_0\sim\mathcal{N}(0,I)$, $p_0\sim\mathcal{N}(0,I)$ &
          none \\
        Oscillator (damped) & 0.02 &
          $q_0\sim\mathcal{N}(0,I)$, $p_0\sim\mathcal{N}(0,I)$ &
          constant ($\gamma{=}0.1$) \\
        Double Pendulum (cons) & 0.01 &
          $\theta_{1,0},\theta_{2,0}\sim\mathrm{Unif}[-\pi,\pi]$, $\omega_{1,0},\omega_{2,0}\sim\mathrm{Unif}[-2,2]$ &
          none \\
        Double Pendulum (damped) & 0.01 &
          $\theta_{1,0},\theta_{2,0}\sim\mathrm{Unif}[-\pi,\pi]$, $\omega_{1,0},\omega_{2,0}\sim\mathrm{Unif}[-2,2]$ &
          viscous ($b{=}0.2$ on both joints) \\
        \midrule
        RLC circuit (damped) & 0.02 &
          $q_0\sim\mathrm{Unif}[-2,2]$, $\varphi_0\sim\mathrm{Unif}[-2,2]$ &
          resistive ($R/L{=}0.5$) \\
        LJ-3 cluster (damped) & 0.002 &
          equilateral triangle at $r_{\mathrm{eq}}{=}2^{1/6}\sigma$, perturbation $\sigma{=}0.05$, $p_0\sim\mathcal{N}(0,0.1^2)$ &
          Langevin ($\gamma{=}0.1$) \\
        Heat exchange (damped) & 0.02 &
          $T_{1,0},T_{2,0}\sim 0.5+1.5\cdot\mathrm{Unif}[0,1]$, $p_0\sim\mathcal{N}(0,0.3^2)$ &
          heat loss ($\kappa_{\mathrm{loss}}{=}0.1$) \\
        N-body 3 (damped) & 0.01 &
          equilateral triangle $R{=}2.0$, perturbation $\sigma{=}0.3$, $p_0\sim\mathcal{N}(0,0.5^2)$ &
          drag ($\gamma{=}0.05$) \\
        Predator--prey (damped) & 0.1 &
          $x_0\sim\mathrm{Unif}[10,50]$ (prey), $y_0\sim\mathrm{Unif}[5,20]$ (predator) &
          intra-species ($\alpha/K$, $\mu$) \\
        \bottomrule
    \end{tabular}
\end{table}

\section{Evaluation Metrics}
\label{app:metrics}

We summarize the main metrics reported in Sec.~\ref{sec:experiments:eval} and in the forecasting/identifiability tables (e.g., Tables~\ref{tab:qonly_rollout_pendulum_h100} and~\ref{tab:windy_identifiability}). All metrics are computed on the test split under the evaluation protocol described in Sec.~\ref{sec:experiments:setup}.

\paragraph{One-step wrapped-angle MSE.}
In the q-only setting, the model predicts $\hat q_{t+1}$ from $q_{0:t}$.
We report the mean squared wrapped angular error:
\begin{equation}
  \mathrm{WrapMSE}
  = \frac{1}{B(T-1)}\sum_{b=1}^{B}\sum_{t=0}^{T-2} \big(\mathrm{wrap}(\hat q_{b,t+1}-q_{b,t+1})\big)^2,
\end{equation}
where $\mathrm{wrap}(\cdot)$ maps angle differences to $[-\pi,\pi]$, $B$ is the number of trajectories, and $T$ is the sequence length.

\paragraph{One-step Euclidean MSE.}
For Euclidean q-only environments (e.g., Oscillator), we report the standard mean squared error:
\begin{equation}
  \mathrm{MSE}
  = \frac{1}{B(T-1)}\sum_{b=1}^{B}\sum_{t=0}^{T-2} \|\hat q_{b,t+1}-q_{b,t+1}\|_2^2.
\end{equation}

\paragraph{Mixed-manifold MSE.}
For product manifolds with both Euclidean and angular coordinates (e.g., Cart-Pole with $q_t=(x_t,\theta_t)\in\mathbb{R}\times\mathbb{S}^1$),
we report a mixed metric that averages translation MSE and wrapped-angle MSE:
\begin{equation}
  \mathrm{MixedMSE}
  = \tfrac12\left(\mathrm{MSE}_x + \mathrm{WrapMSE}_\theta\right),
  \qquad
  \mathrm{MSE}_x = \frac{1}{B(T-1)}\sum_{b=1}^{B}\sum_{t=0}^{T-2} (\hat x_{b,t+1}-x_{b,t+1})^2.
\end{equation}
Open-loop rollouts use the analogous metric at horizon $H$.

\paragraph{Open-loop rollouts.}
For burn-in length $K$ and horizon $H$, we condition on a ground-truth prefix $q_{0:K-1}$ and then predict open-loop for $H$ steps.
Baselines produce autoregressive rollouts; PHAST additionally supports \emph{takeover} rollouts (burn-in then integrate).
We report wrapped rollout error and (where applicable) takeover rollout error at horizon $H$.
For Euclidean environments, we report rollout MSE at horizon $H$ (and takeover analogues where available).
For reproducibility, Table~\ref{tab:metric_keys} maps paper metrics to the identifiers used in our implementation.

\paragraph{Damping recovery.}
When a model exposes a diagonal damping prediction $\hat d_t \approx D(q_t)$, we compute
\begin{align}
  \mathrm{MAE}_D &= \frac{1}{BT}\sum_{b=1}^{B}\sum_{t=0}^{T-1}\left|\,\hat d_{b,t} - d_{b,t}\right|, \\
  R^2_D &= 1 - \frac{\sum_{b=1}^{B}\sum_{t=0}^{T-1} (\hat d_{b,t} - d_{b,t})^2}{\sum_{b=1}^{B}\sum_{t=0}^{T-1} (d_{b,t} - \bar d)^2 + 10^{-12}},
\end{align}
where $d_{b,t} = D(q_{b,t})$ is the ground-truth damping and $\bar d = \frac{1}{BT}\sum_{b=1}^{B}\sum_{t=0}^{T-1} d_{b,t}$ is its mean over the test set.

\paragraph{Passivity violations on rollouts.}
Using an energy estimate $\hat H$ formed from the predicted rollout (e.g., finite-difference velocity estimates in q-only settings), we report the fraction of rollout steps with an energy increase:
\begin{equation}
  \mathrm{PassViol}(H)
  = \frac{1}{BH}\sum_{b=1}^{B}\sum_{h=0}^{H-1}\mathbb{I}\!\left[\hat H_{b,h+1}-\hat H_{b,h} > \varepsilon\right],
\end{equation}
where $\varepsilon$ is a small tolerance (we use $\varepsilon=10^{-6}$ in code).

\paragraph{Energy budget residual on rollouts.}
For our input-free benchmarks ($u=0$), the continuous-time energy identity is
\begin{equation}
  \frac{dH_{\mathrm{env}}}{dt} = -D_{\mathrm{env}}(q)\,\|v\|_2^2,
\end{equation}
where $H_{\mathrm{env}}$ is the benchmark's analytic energy/Hamiltonian (Appendix~\ref{app:environments}) and $D_{\mathrm{env}}(\cdot)$ is the \emph{simulator} damping law (a scalar damping coefficient applied isotropically in our benchmarks).
On predicted rollouts $\{\hat q_{b,h}\}_{h=0}^{H}$, we form a finite-difference velocity $\hat v_{b,h} \approx (\hat q_{b,h+1}-\hat q_{b,h})/\dt$ (wrapping angular components) and an energy estimate $\hat H_{b,h}$ using $H_{\mathrm{env}}$ (e.g., $H_{\mathrm{env}}(\hat q_{b,h+1},\hat v_{b,h})$ in q-only).
We then report the mean absolute discrete-time residual
\begin{equation}
  \rho_{b,h}
  =
  \frac{\hat H_{b,h+1}-\hat H_{b,h}}{\dt}
  -
  \big(-D_{\mathrm{env}}(\hat q_{b,h})\,\|\hat v_{b,h}\|_2^2\big),
  \qquad
  \mathrm{EbudRes}(H) = \frac{1}{BH}\sum_{b=1}^{B}\sum_{h=0}^{H-1}|\rho_{b,h}|.
\end{equation}
Evaluating $D_{\mathrm{env}}$ on the predicted configuration makes this diagnostic comparable across baselines that do not learn an explicit $D(q)$; damping recovery metrics are reported separately when a model exposes $\hat D(q)$.

\begin{table}[t]
  \centering
  \caption{Paper metric names and their corresponding metric identifiers used in our implementation.}
  \label{tab:metric_keys}
  \scriptsize
  \begin{tabular}{@{}p{0.44\linewidth}p{0.52\linewidth}@{}}
    \toprule
    \textbf{Paper metric} & \textbf{Implementation key(s)} \\
    \midrule
    One-step wrapped-angle MSE ($\mathrm{WrapMSE}$) & \texttt{theta\_\allowbreak wrap\_\allowbreak mse} \\
    One-step Euclidean MSE ($\mathrm{MSE}$) & \texttt{mse} \\
    Rollout wrapped-angle MSE at horizon $H$ ($\mathrm{WrapMSE}_\theta^{\mathrm{roll}}(H)$) &
      \texttt{rollout\_\allowbreak theta\_\allowbreak wrap\_\allowbreak mse\_\allowbreak h\{H\}} \\
    Rollout Euclidean MSE at horizon $H$ ($\mathrm{MSE}^{\mathrm{roll}}(H)$) &
      \texttt{rollout\_\allowbreak mse\_\allowbreak h\{H\}} \\
    Rollout mixed-manifold MSE at horizon $H$ ($\mathrm{MixedMSE}^{\mathrm{roll}}(H)$) &
      \texttt{rollout\_\allowbreak mixed\_\allowbreak mse\_\allowbreak h\{H\}} \\
    Takeover rollouts (PHAST only; Eq.~\eqref{eq:qonly_takeover}) &
      \texttt{rollout\_\allowbreak takeover\_\allowbreak *} \\
	    Damping recovery ($R^2_D$, $\mathrm{MAE}_D$) &
	      \texttt{damping\_\allowbreak r2}, \texttt{damping\_\allowbreak mae} \\
	    Energy-budget residual ($\mathrm{EbudRes}(H)$) &
	      \texttt{rollout\_\allowbreak energy\_\allowbreak budget\_\allowbreak resid\_\allowbreak h\{H\}} \\
	    Passivity violation rate on rollouts ($\mathrm{PassViol}(H)$) &
	      \texttt{rollout\_\allowbreak passivity\_\allowbreak violations\_\allowbreak h\{H\}} \\
	    \bottomrule
	  \end{tabular}
\end{table}

\subsection{Error Sources and Diagnostics (q-only)}
\label{app:metrics:debugging}

The q-only pipeline introduces multiple amplification points beyond discretization error: (i) velocity inference from position-only measurements, (ii) the semantics of the inferred canonical state $(q,\hat p)$, and (iii) open-loop compounding when the model consumes its own predictions at test time.
Figure~\ref{fig:qonly_debug_graph} summarizes the computation graph and diagnostic comparisons that help localize these error sources.

\begin{figure}[t]
  \centering
  \begin{minipage}{0.98\linewidth}
    \centering
    \begin{tikzpicture}[
        block/.style={rectangle, draw=black, rounded corners=2pt, minimum height=0.75cm, minimum width=2.7cm, align=center, font=\scriptsize, fill=black!4},
        note/.style={rectangle, draw=black!45, rounded corners=2pt, align=left, font=\scriptsize, fill=blue!4, text width=3.5cm, inner sep=4pt},
        warn/.style={rectangle, draw=red!55, rounded corners=2pt, align=left, font=\scriptsize, fill=red!4, text width=3.5cm, inner sep=4pt},
        arrow/.style={-{Stealth[length=2mm]}, thick},
      ]
      \node[font=\scriptsize\bfseries] at (0,1.1) {Teacher-forced training (next-step)};
      \node[block] (q) at (0,0.3) {$q_{0:T-1}$ (ground truth)};
      \node[block] (obs) at (0,-0.75) {observer $o_{\phi}$\\[-0.1em]$\hat{\dot q}_{0:T-1}$};
      \node[block] (canon) at (0,-1.8) {canonicalizer $C_{\psi}$\\[-0.1em]$\hat x_t=(q_t,\hat p_t)$};
      \node[block] (core) at (0,-2.85) {PHAST core transition $\Phi_{\dt}$\\[-0.1em]$(q_t,\hat p_t)\mapsto(\hat q_{t+1},\hat p_{t+1})$};
      \node[block, fill=white] (loss) at (0,-3.9) {loss on $\hat q_{t+1}$};

      \draw[arrow] (q) -- (obs);
      \draw[arrow] (obs) -- (canon);
      \draw[arrow] (canon) -- (core);
      \draw[arrow] (core) -- (loss);

      \node[note] (n1) at (4.2,-0.75) {\textbf{Observer error proxies}\\
      MAE of inferred $\hat p$ (when $p$ is available)\\
      MAE of FD proxy (from $q$ history)};
      \node[warn] (n2) at (4.2,-2.85) {\textbf{Train--test mismatch}\\
      teacher forcing uses ground-truth history;\\
      open-loop drift can still grow.};

      \draw[arrow, black!50] (obs.east) -- (n1.west);
      \draw[arrow, red!55] (core.east) -- (n2.west);
    \end{tikzpicture}
  \end{minipage}

  \vspace{0.6em}

  \begin{minipage}{0.98\linewidth}
    \centering
    \begin{tikzpicture}[
        block/.style={rectangle, draw=black, rounded corners=2pt, minimum height=0.75cm, minimum width=2.5cm, align=center, font=\scriptsize, fill=black!4},
        note/.style={rectangle, draw=black!45, rounded corners=2pt, align=left, font=\scriptsize, fill=green!5, text width=3.2cm, inner sep=4pt},
        warn/.style={rectangle, draw=red!55, rounded corners=2pt, align=left, font=\scriptsize, fill=red!4, text width=3.2cm, inner sep=4pt},
        arrow/.style={-{Stealth[length=2mm]}, thick},
      ]
      \node[font=\scriptsize\bfseries] at (0,1.1) {Open-loop evaluation (rollouts)};
      \node[block] (ctx) at (0,0.3) {burn-in context\\$q_{0:K-1}$};

      \node[block] (ar) at (-1.8,-0.9) {autoregressive rollout\\(observer-in-loop)};
      \node[block] (tk) at (1.8,-0.9) {takeover rollout\\(infer once, integrate)};

      \node[block] (obs2) at (-1.8,-2.1) {$o_{\phi}, C_{\psi}$ each step};
      \node[block] (core2) at (1.8,-2.1) {$\Phi_{\dt}$ only};

      \node[block, fill=white] (m1) at (-1.8,-3.3) {rollout metrics};
      \node[block, fill=white] (m2) at (1.8,-3.3) {rollout metrics};

      \draw[arrow] (ctx) -- (ar);
      \draw[arrow] (ctx) -- (tk);
      \draw[arrow] (ar) -- (obs2);
      \draw[arrow] (obs2) -- (m1);
      \draw[arrow] (tk) -- (core2);
      \draw[arrow] (core2) -- (m2);

      \node[note] (n3) at (5.6,-0.6) {\textbf{Observer-in-loop gap}\\
      autoregressive vs.\ takeover rollouts\\
      (large gap $\Rightarrow$ observer)};
      \node[warn] (n5) at (5.6,-1.95) {\textbf{Discrete-time stability}\\
      passivity-violation rate and\\
      stiffness proxy for damping half-step};
      \node[note] (n4) at (5.6,-3.3) {\textbf{Core/physics diagnostics}\\
      rollout error and energy-budget residual};

      \draw[arrow, black!50] (tk.east) -- (n3.west);
      \draw[arrow, red!55] (core2.east) -- (n5.west);
      \draw[arrow, black!50] (m2.east) -- (n4.west);
    \end{tikzpicture}
  \end{minipage}
  \caption{\textbf{Q-only computation graph and diagnostic probes.} Training is teacher-forced (ground-truth $q$ history), while evaluation is open-loop and therefore sensitive to compounding error. Autoregressive rollouts keep the observer in the loop; takeover rollouts isolate the PHAST core transition given a single inferred boundary state.}
  \label{fig:qonly_debug_graph}
\end{figure}
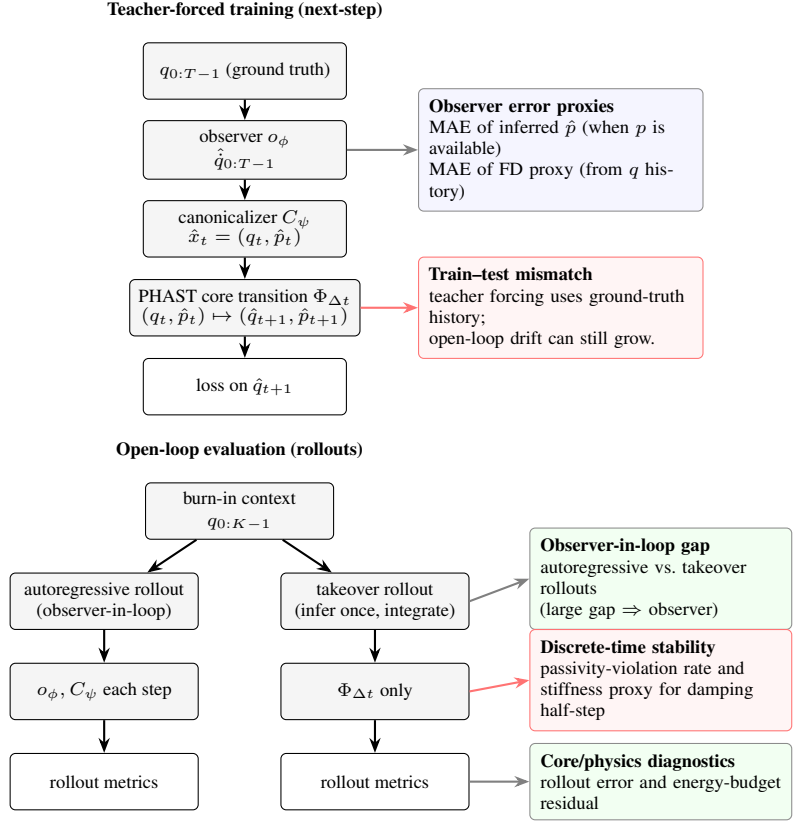

\begin{center}
  \fcolorbox{black!50}{black!3}{%
    \begin{minipage}{0.965\linewidth}
      \small
      \textbf{Error sources and diagnostic comparisons.}
      \vspace{0.25em}
      \begin{itemize}
        \item \textbf{Observer amplification (q-only):} if autoregressive rollouts are much worse than takeover rollouts, repeated velocity re-inference is injecting error (noise mismatch / insufficient context). Compare inferred $\hat p$ (when available) to the true $p$, and compare to a finite-difference proxy.
        \item \textbf{Discrete-time stiffness:} energy increases or numerical instabilities can occur when the explicit damping half-step is too stiff for the chosen $\dt$. Remedies include smaller $\dt$, more substeps $L$, or bounding the damping strength (Sec.~\ref{sec:methods:householder}).
        \item \textbf{Damping as an error-sink (identifiability breakdown):} accurate rollouts with poor damping recovery ($R^2_D\ll 0$) suggests dissipation is compensating for errors in $V$/$M$. Bounded Householder damping reduces this degeneracy; when the isotropic base term $d_0$ is known, fixing $d_0$ further improves identifiability.
        \item \textbf{Time-rescaling ambiguity:} a learned internal timestep $\dt_{\mathrm{model}}\neq \dt$ can improve forecasting by reparameterizing time, but it confounds parameter recovery; for identifiability studies, we recommend fixed $\dt$.
        \item \textbf{Error compounding / phase sensitivity:} small one-step errors can yield large long-horizon drift. Horizon sweeps (e.g., $H\in\{10,50,100\}$) and the full-state ablation help localize whether the dominant source is the observer pipeline or the PHAST core transition.
      \end{itemize}
      \vspace{0.1em}
      \scriptsize
      \textit{Implementation note:} Table~\ref{tab:metric_keys} lists the metric identifiers used in our evaluation pipeline; additional PHAST-specific stability diagnostics (e.g., stiffness proxies and timestep statistics) are reported when applicable.
    \end{minipage}
  }
\end{center}

% ===========================================================================
% F. Additional Benchmark Tables
% ===========================================================================
% appendix_tables_qonly.tex — Q-only forecasting tables (shared between ICML appendix and arXiv main)

% ===========================================================================
% F. Additional Benchmark Tables
% ===========================================================================
\section{Additional Benchmark Tables}
\label{app:tables:qonly_rollout}

We report the full set of q-only forecasting tables across all thirteen benchmarks.
All results follow the protocol in Sec.~\ref{sec:experiments:setup} (mean $\pm$ std over 5 seeds; dataset seed fixed to 42).
Tables~\ref{tab:qonly_rollout_pendulum_h100}--\ref{tab:qonly_rollout_oscillator_h100} cover the eight mechanical benchmarks (rollout MSE at horizon $H{=}100$).
Tables~\ref{tab:qonly_rollout_rlc}--\ref{tab:qonly_rollout_predprey} cover the five non-mechanical systems from Table~\ref{tab:structured_hamiltonians} (next-step MSE).

\begin{table}[t]
    \centering
    \caption{\textbf{Q-only open-loop forecasting (1-DOF pendulum environments).} Mean $\pm$ std of wrapped-angle rollout MSE at horizon $H{=}100$ over 5 seeds. Parameter counts correspond to the Windy setting (largest; position-dependent damping introduces additional parameters). Lower is better.}
    \label{tab:qonly_rollout_pendulum_h100}
    \small
    \begin{tabular}{@{}lrrrr@{}}
        \toprule
        \textbf{Model} & \textbf{Params} & \textbf{Conservative} & \textbf{Damped} & \textbf{Windy} \\
        \midrule
        \multicolumn{5}{l}{\textit{PHAST (ours)}} \\
        PHAST (KNOWN) & 3{,}364 & $0.738 \pm 0.143$ & \best{0.017 \pm 0.005} & $0.106 \pm 0.020$ \\
        PHAST (PARTIAL) & 13{,}736 & \best{0.680 \pm 0.043} & $0.018 \pm 0.001$ & \best{0.092 \pm 0.014} \\
        PHAST (UNKNOWN) & 13{,}738 & $0.993 \pm 0.128$ & $0.275 \pm 0.083$ & $0.298 \pm 0.048$ \\
        \midrule
        \multicolumn{5}{l}{\textit{Baselines}} \\
        GRU & 37{,}889 & $2.939 \pm 0.302$ & $1.310 \pm 0.743$ & $1.796 \pm 0.625$ \\
        S5 & 17{,}089 & $2.941 \pm 0.042$ & $0.657 \pm 0.065$ & $0.600 \pm 0.047$ \\
        LinOSS & 17{,}089 & $2.849 \pm 0.475$ & $2.286 \pm 0.324$ & $1.458 \pm 0.324$ \\
        D-LinOSS & 33{,}793 & $2.738 \pm 0.630$ & $0.450 \pm 0.241$ & $0.435 \pm 0.239$ \\
        Transformer & 100{,}161 & $2.320 \pm 0.224$ & $0.493 \pm 0.105$ & $0.824 \pm 0.134$ \\
        VPT & 16{,}833 & $2.875 \pm 0.182$ & $2.111 \pm 0.222$ & $2.218 \pm 0.135$ \\
        \bottomrule
    \end{tabular}
\end{table}

\ifdefined\isarxiv
% Included in the arXiv main body for a tighter narrative (see Sec.~\ref{sec:experiments:results}).
\else
% tables_windy_identifiability.tex — Windy Pendulum identifiability + energy consistency table

\begin{table}[t]
    \centering
    \caption{\textbf{Windy Pendulum (q-only) identifiability and energy consistency.} Mean $\pm$ std over 5 seeds (same setup as Table~\ref{tab:qonly_rollout_pendulum_h100}). Baselines do not expose a damping field, so damping recovery metrics are not applicable.}
    \label{tab:windy_identifiability}
    \scriptsize
    \begin{tabular}{@{}lrrrrr@{}}
        \toprule
        \textbf{Model} & \textbf{$\mathrm{WrapMSE}_\theta^{\mathrm{roll}}(H{=}100)$} $\downarrow$ & \textbf{$R^2_D$} $\uparrow$ & \textbf{$\mathrm{MAE}_D$} $\downarrow$ & \textbf{$\mathrm{EbudRes}^{\mathrm{roll}}(H{=}100)$} $\downarrow$ & \textbf{Params} \\
        \midrule
        PHAST (KNOWN) & $0.106 \pm 0.020$ & \best{0.996 \pm 0.004} & \best{0.007 \pm 0.004} & $1.56 \pm 0.02$ & 3{,}364 \\
        PHAST (PARTIAL) & \best{0.092 \pm 0.014} & $0.654 \pm 0.077$ & $0.063 \pm 0.011$ & \best{1.50 \pm 0.04} & 13{,}736 \\
        PHAST (UNKNOWN) & $0.298 \pm 0.048$ & $-96.546 \pm 15.897$ & $1.343 \pm 0.115$ & $2.58 \pm 0.19$ & 13{,}738 \\
        \midrule
        GRU & $1.796 \pm 0.625$ & --- & --- & $12.34 \pm 6.58$ & 37{,}889 \\
        S5 & $0.600 \pm 0.047$ & --- & --- & $11.71 \pm 2.91$ & 17{,}089 \\
        LinOSS & $1.458 \pm 0.324$ & --- & --- & $56.14 \pm 37.41$ & 17{,}089 \\
        D-LinOSS & $0.435 \pm 0.239$ & --- & --- & $6.86 \pm 1.14$ & 33{,}793 \\
        Transformer & $0.824 \pm 0.134$ & --- & --- & $7.71 \pm 0.63$ & 100{,}161 \\
        VPT & $2.218 \pm 0.135$ & --- & --- & $40.48 \pm 8.04$ & 16{,}833 \\
        \bottomrule
    \end{tabular}
\end{table}

\fi

\begin{table}[t]
    \centering
    \caption{\textbf{Q-only open-loop forecasting (2-DOF double pendulum environments).} Mean $\pm$ std of wrapped-angle rollout MSE at horizon $H{=}100$ over 5 seeds.}
    \label{tab:qonly_rollout_double_h100}
    \small
    \begin{tabular}{@{}lrrr@{}}
        \toprule
        \textbf{Model} & \textbf{Params} & \textbf{Double (cons)} & \textbf{Double (damped)} \\
        \midrule
        \multicolumn{4}{l}{\textit{PHAST (ours)}} \\
        PHAST (KNOWN) & 3{,}588 & $0.421 \pm 0.061$ & $0.332 \pm 0.030$ \\
        PHAST (PARTIAL) & 12{,}166 & \best{0.402 \pm 0.047} & \best{0.320 \pm 0.032} \\
        PHAST (UNKNOWN) & 13{,}067 & $0.629 \pm 0.060$ & $0.529 \pm 0.047$ \\
        \midrule
        \multicolumn{4}{l}{\textit{Baselines}} \\
        GRU & 38{,}146 & $1.300 \pm 0.143$ & $1.346 \pm 0.127$ \\
        S5 & 17{,}218 & $0.618 \pm 0.028$ & $0.630 \pm 0.031$ \\
        LinOSS & 17{,}218 & $1.640 \pm 0.214$ & $1.573 \pm 0.129$ \\
        D-LinOSS & 33{,}922 & $1.501 \pm 0.124$ & $1.298 \pm 0.105$ \\
        Transformer & 100{,}290 & $1.033 \pm 0.248$ & $0.846 \pm 0.132$ \\
        VPT & 16{,}962 & $2.848 \pm 0.213$ & $2.721 \pm 0.156$ \\
        \bottomrule
    \end{tabular}
\end{table}

\begin{table}[t]
    \centering
    \caption{\textbf{Q-only open-loop forecasting (Oscillator environments).} Mean $\pm$ std of rollout MSE at horizon $H{=}100$ over 5 seeds.}
    \label{tab:qonly_rollout_oscillator_h100}
    \small
    \begin{tabular}{@{}lrrrr@{}}
        \toprule
        \textbf{Model} & \textbf{Params} & \textbf{Oscillator (cons)} & \textbf{Oscillator (damped)} \\
        \midrule
        \multicolumn{4}{l}{\textit{PHAST (ours)}} \\
        PHAST (KNOWN) & 3{,}587 & \best{0.0010 \pm 0.0003} & $0.0012 \pm 0.0004$ \\
        PHAST (PARTIAL) & 12{,}165 & \best{0.0010 \pm 0.0002} & \best{0.0011 \pm 0.0003} \\
        PHAST (UNKNOWN) & 12{,}171 & $0.0072 \pm 0.0031$ & $0.0101 \pm 0.0046$ \\
        \midrule
        \multicolumn{4}{l}{\textit{Baselines}} \\
        GRU & 38{,}146 & $1.8470 \pm 0.0754$ & $1.4565 \pm 0.0811$ \\
        S5 & 17{,}218 & $2.0200 \pm 0.0899$ & $1.7079 \pm 0.0570$ \\
        LinOSS & 17{,}218 & $1.7241 \pm 0.1787$ & $1.3927 \pm 0.1189$ \\
        D-LinOSS & 33{,}922 & $1.8603 \pm 0.2750$ & $1.6275 \pm 0.1809$ \\
        Transformer & 100{,}290 & $1.0873 \pm 0.2987$ & $0.9259 \pm 0.2537$ \\
        VPT & 16{,}962 & $3.0941 \pm 1.1140$ & $2.5710 \pm 0.7990$ \\
        \bottomrule
    \end{tabular}
\end{table}

% ===========================================================================
% Tables for additional Table 1 systems (non-mechanical benchmarks)
% ===========================================================================

\begin{table}[t]
    \centering
    \caption{\textbf{Q-only open-loop forecasting (RLC circuit, KNOWN regime).} Mean $\pm$ std of next-step MSE over 5 seeds. The damped series RLC circuit ($L{=}1$, $C{=}1$, $R{=}0.5$) is the electrical analog of a damped spring--mass system. Lower is better.}
    \label{tab:qonly_rollout_rlc}
    \small
    \begin{tabular}{@{}lrr@{}}
        \toprule
        \textbf{Model} & \textbf{Params} & \textbf{RLC (damped)} \\
        \midrule
        \multicolumn{3}{l}{\textit{PHAST (ours)}} \\
        PHAST (KNOWN) & 3{,}364 & \best{2.64 \times 10^{-5} \pm 1.5 \times 10^{-8}} \\
        PHAST (PARTIAL) & 11{,}878 & $2.64 \times 10^{-5} \pm 1.1 \times 10^{-8}$ \\
        PHAST (UNKNOWN) & 11{,}879 & $2.63 \times 10^{-5} \pm 3.1 \times 10^{-8}$ \\
        \midrule
        \multicolumn{3}{l}{\textit{Baselines}} \\
        S5 & 17{,}089 & $1.84 \times 10^{-3} \pm 6.69 \times 10^{-4}$ \\
        LinOSS & 17{,}089 & $8.57 \times 10^{-4} \pm 1.26 \times 10^{-4}$ \\
        GRU & 37{,}889 & $3.22 \times 10^{-3} \pm 3.00 \times 10^{-4}$ \\
        LSTM & 50{,}497 & $7.53 \times 10^{-3} \pm 2.78 \times 10^{-4}$ \\
        VPT & 16{,}833 & $8.29 \times 10^{-4} \pm 4.38 \times 10^{-5}$ \\
        Transformer & 100{,}161 & $4.81 \times 10^{-4} \pm 1.32 \times 10^{-4}$ \\
        \bottomrule
    \end{tabular}
\end{table}

\begin{table}[t]
    \centering
    \caption{\textbf{Q-only open-loop forecasting (Lennard--Jones 3-particle cluster, KNOWN regime).} Mean $\pm$ std of next-step MSE over 5 seeds. Three particles in 2D with pairwise LJ potential ($\epsilon{=}1$, $\sigma{=}1$) and Langevin friction ($\gamma{=}0.1$). Lower is better.}
    \label{tab:qonly_rollout_lj3}
    \small
    \begin{tabular}{@{}lrr@{}}
        \toprule
        \textbf{Model} & \textbf{Params} & \textbf{LJ-3 (damped)} \\
        \midrule
        \multicolumn{3}{l}{\textit{PHAST (ours)}} \\
        PHAST (KNOWN) & 4{,}490 & \best{4.61 \times 10^{-10} \pm 4.20 \times 10^{-12}} \\
        PHAST (PARTIAL) & 13{,}324 & $4.59 \times 10^{-10} \pm 2.04 \times 10^{-12}$ \\
        PHAST (UNKNOWN) & 13{,}355 & $6.93 \times 10^{-10} \pm 1.38 \times 10^{-11}$ \\
        \midrule
        \multicolumn{3}{l}{\textit{Baselines}} \\
        S5 & 17{,}734 & $2.05 \times 10^{-4} \pm 4.33 \times 10^{-5}$ \\
        LinOSS & 17{,}734 & $3.07 \times 10^{-4} \pm 6.64 \times 10^{-5}$ \\
        GRU & 39{,}174 & $1.96 \times 10^{-3} \pm 1.80 \times 10^{-4}$ \\
        LSTM & 52{,}102 & $3.39 \times 10^{-3} \pm 3.15 \times 10^{-4}$ \\
        VPT & 17{,}478 & $1.49 \times 10^{-3} \pm 8.07 \times 10^{-4}$ \\
        Transformer & 100{,}806 & $1.20 \times 10^{-3} \pm 1.62 \times 10^{-4}$ \\
        \bottomrule
    \end{tabular}
\end{table}

\begin{table}[t]
    \centering
    \caption{\textbf{Q-only open-loop forecasting (Heat exchange, KNOWN regime).} Mean $\pm$ std of next-step MSE over 5 seeds. Two coupled thermal masses ($c_1{=}c_2{=}1$, coupling $\kappa{=}0.5$, loss $\kappa_{\mathrm{loss}}{=}0.1$). Lower is better.}
    \label{tab:qonly_rollout_heat}
    \small
    \begin{tabular}{@{}lrr@{}}
        \toprule
        \textbf{Model} & \textbf{Params} & \textbf{Heat Exch.\ (damped)} \\
        \midrule
        \multicolumn{3}{l}{\textit{PHAST (ours)}} \\
        PHAST (KNOWN) & 3{,}592 & \best{2.42 \times 10^{-6} \pm 9.74 \times 10^{-10}} \\
        PHAST (PARTIAL) & 12{,}170 & $2.42 \times 10^{-6} \pm 2.63 \times 10^{-9}$ \\
        PHAST (UNKNOWN) & 12{,}173 & $2.48 \times 10^{-6} \pm 1.85 \times 10^{-8}$ \\
        \midrule
        \multicolumn{3}{l}{\textit{Baselines}} \\
        S5 & 17{,}218 & $1.00 \times 10^{-3} \pm 1.62 \times 10^{-4}$ \\
        LinOSS & 17{,}218 & $4.46 \times 10^{-4} \pm 3.11 \times 10^{-5}$ \\
        GRU & 38{,}146 & $3.63 \times 10^{-3} \pm 3.31 \times 10^{-4}$ \\
        LSTM & 50{,}818 & $8.03 \times 10^{-3} \pm 1.39 \times 10^{-4}$ \\
        VPT & 16{,}962 & $1.71 \times 10^{-3} \pm 1.53 \times 10^{-4}$ \\
        Transformer & 100{,}290 & $7.74 \times 10^{-4} \pm 2.14 \times 10^{-4}$ \\
        \bottomrule
    \end{tabular}
\end{table}

\begin{table}[t]
    \centering
    \caption{\textbf{Q-only open-loop forecasting (N-body gravity, PARTIAL regime).} Mean $\pm$ std of next-step MSE over 5 seeds. Three bodies in 2D with softened gravity ($G{=}1$, $\epsilon_{\mathrm{soft}}{=}0.1$) and drag ($\gamma{=}0.05$). Gravitational potential template is given; masses and damping are learned. Lower is better.}
    \label{tab:qonly_rollout_nbody}
    \small
    \begin{tabular}{@{}lrr@{}}
        \toprule
        \textbf{Model} & \textbf{Params} & \textbf{N-body 3 (damped)} \\
        \midrule
        \multicolumn{3}{l}{\textit{PHAST (ours)}} \\
        PHAST (KNOWN) & 4{,}489 & \best{4.28 \times 10^{-8} \pm 1.28 \times 10^{-11}} \\
        PHAST (PARTIAL) & 13{,}323 & $4.27 \times 10^{-8} \pm 2.25 \times 10^{-11}$ \\
        PHAST (UNKNOWN) & 13{,}355 & $4.28 \times 10^{-8} \pm 6.17 \times 10^{-11}$ \\
        \midrule
        \multicolumn{3}{l}{\textit{Baselines}} \\
        S5 & 17{,}734 & $3.03 \times 10^{-3} \pm 5.98 \times 10^{-4}$ \\
        LinOSS & 17{,}734 & $2.00 \times 10^{-3} \pm 7.74 \times 10^{-5}$ \\
        GRU & 39{,}174 & $7.18 \times 10^{-3} \pm 1.79 \times 10^{-3}$ \\
        LSTM & 52{,}102 & $4.64 \times 10^{-2} \pm 8.17 \times 10^{-3}$ \\
        VPT & 17{,}478 & $1.47 \times 10^{-2} \pm 1.30 \times 10^{-2}$ \\
        Transformer & 100{,}806 & $1.83 \times 10^{-3} \pm 6.88 \times 10^{-4}$ \\
        \bottomrule
    \end{tabular}
\end{table}

\begin{table}[t]
    \centering
    \caption{\textbf{Q-only open-loop forecasting (Predator--prey, UNKNOWN regime).} Mean $\pm$ std of next-step MSE over 5 seeds. Dissipative Lotka--Volterra ($\alpha{=}1$, $\beta{=}0.1$, $\gamma{=}0.4$, $\delta{=}0.1$, $K{=}100$, $\mu{=}0.01$). All PHAST components are neural (non-canonical Hamiltonian structure). Lower is better.}
    \label{tab:qonly_rollout_predprey}
    \small
    \begin{tabular}{@{}lrr@{}}
        \toprule
        \textbf{Model} & \textbf{Params} & \textbf{Predator--Prey (damped)} \\
        \midrule
        \multicolumn{3}{l}{\textit{PHAST (ours)}} \\
        PHAST (KNOWN) & 3{,}364 & $0.0203 \pm 0.0002$ \\
        PHAST (PARTIAL) & 11{,}878 & $0.0203 \pm 0.0002$ \\
        PHAST (UNKNOWN) & 11{,}879 & \best{0.0199 \pm 0.0004} \\
        \midrule
        \multicolumn{3}{l}{\textit{Baselines}} \\
        S5 & 17{,}089 & $6.12 \pm 0.58$ \\
        LinOSS & 17{,}089 & $3.53 \pm 0.70$ \\
        GRU & 37{,}889 & $4.72 \pm 0.23$ \\
        LSTM & 50{,}497 & $5.85 \pm 0.14$ \\
        VPT & 16{,}833 & $0.226 \pm 0.012$ \\
        Transformer & 100{,}161 & $0.179 \pm 0.014$ \\
        \bottomrule
    \end{tabular}
\end{table}

% ===========================================================================
% F. Additional Ablations
% ===========================================================================
\section{Additional Ablations}
\label{app:ablations}

\subsection{Full-State Baseline (Isolating PHAST Core Transition)}
\label{app:ablations:fullstate}

\paragraph{Motivation.}
Our main benchmarks are q-only and therefore include an observer+canonicalizer pipeline (Appendix~\ref{app:arch:qonly}).
To isolate the contribution of the port-Hamiltonian transition itself (PHAST core transition), we include a full-state ablation where the state $(q,p)$ is observed and no observer/canonicalizer is used.

\paragraph{Setup.}
We evaluate on the full-state Windy Pendulum benchmark and compare PHAST variants against a dissipative Hamiltonian baseline (DHNN).
All methods use the same training protocol as Sec.~\ref{sec:experiments:setup} (AdamW with a cosine schedule; 50 epochs), and we report mean $\pm$ std over 5 model seeds with a fixed dataset seed ($42$).

\begin{table}[t]
    \centering
    \setlength{\tabcolsep}{3pt}
    \caption{\textbf{Full-state Windy Pendulum (q,p observed): PHAST vs DHNN.} Mean $\pm$ std over 5 seeds (dataset seed fixed to 42). Lower is better. DHNN does not expose an explicit damping field $D(q)$, so $R^2_D$/$\mathrm{MAE}_D$ are not reported.}
    \label{tab:fullstate_windy_dhnn}
    \scriptsize
    \begin{tabular}{@{}lrrrrr@{}}
        \toprule
        \textbf{Model} & \textbf{Params} &
        \textbf{$\mathrm{WrapMSE}_\theta$} $\downarrow$ &
        \textbf{$\mathrm{WrapMSE}_\theta^{\mathrm{roll}}(100)$} $\downarrow$ &
        \textbf{$R^2_D$} $\uparrow$ &
        \textbf{$\mathrm{MAE}_D$} $\downarrow$ \\
        \midrule
        PHAST (PARTIAL) & 10{,}376 & $1.49{\times}10^{-5} \pm 2.86{\times}10^{-6}$ & $0.089 \pm 0.005$ & $0.798 \pm 0.034$ & $0.045 \pm 0.005$ \\
        PHAST (UNKNOWN) & 10{,}378 & $6.35{\times}10^{-5} \pm 1.47{\times}10^{-5}$ & $0.128 \pm 0.003$ & $0.583 \pm 0.082$ & $0.081 \pm 0.010$ \\
        DHNN & 107{,}124 & $0.186 \pm 0.061$ & $2.264 \pm 0.160$ & --- & --- \\
        \bottomrule
    \end{tabular}
\end{table}

\subsection{Configuration-Dependent Mass (Nonseparable Hamiltonian)}
\label{app:ablations:nonseparable_mass}

\paragraph{Motivation.}
Several of our benchmarks (e.g., Cart-Pole and Double Pendulum; Appendix~\ref{app:environments}) have a configuration-dependent inertia $M(q)$, yielding a nonseparable Hamiltonian.
Our main results intentionally use a separable constant-mass approximation (Sec.~\ref{sec:intro:regimes}) to enable a lightweight leapfrog core.
Here we provide a targeted ablation showing that using the true $M(q)$ together with a nonseparable Hamiltonian integrator can materially improve long-horizon q-only rollouts.

\paragraph{Setup.}
We use Windy Cart-Pole (q-only) in the KNOWN regime and compare: (i) a constant-mass approximation vs.\ (ii) the true configuration-dependent $M(q)$ (Eq.~\eqref{eq:cartpole_mass}).
Both variants use Strang splitting with a Hamiltonian implicit midpoint conservative core and a fixed timestep, and follow the training protocol of Sec.~\ref{sec:experiments:setup}.
We report mean $\pm$ std over 3 model seeds (dataset seed fixed to 42).

\begin{table}[t]
    \centering
    \setlength{\tabcolsep}{2pt}
    \caption{\textbf{Nonseparable mass ablation (Windy Cart-Pole, q-only).} Mean $\pm$ std over 3 seeds (dataset seed fixed to 42). Both rows use the same implicit-midpoint conservative core; only the mass model differs. Lower is better.}
    \label{tab:cartpole_nonseparable_mass}
    \scriptsize
    \begin{tabular}{@{}lrrrrrr@{}}
        \toprule
        \textbf{Model} & \textbf{Params} &
        \textbf{$\mathrm{MixedMSE}^{\mathrm{roll}}(H{=}100)$} $\downarrow$ &
        \textbf{$\mathrm{EbudRes}^{\mathrm{roll}}(H{=}100)$} $\downarrow$ &
        \textbf{$\mathrm{PassViol}^{\mathrm{roll}}(H{=}100)$} $\downarrow$ &
        \textbf{$R^2_D$} $\uparrow$ &
        \textbf{$\mathrm{MAE}_D$} $\downarrow$ \\
        \midrule
        PHAST (KNOWN, constant $M$) & 3{,}589 & $0.096 \pm 0.016$ & $3.31 \pm 0.39$ & $0.124 \pm 0.007$ & $0.956 \pm 0.004$ & $0.028 \pm 0.001$ \\
        PHAST (KNOWN, $M(q)$) & 3{,}589 & \best{0.031 \pm 0.006} & \best{1.30 \pm 0.06} & \best{0.016 \pm 0.007} & \best{0.986 \pm 0.002} & \best{0.013 \pm 0.001} \\
        \bottomrule
    \end{tabular}
\end{table}

\begin{table}[t]
    \centering
    \setlength{\tabcolsep}{2pt}
    \caption{\textbf{Nonseparable mass ablation (Damped Double Pendulum, q-only).} Mean $\pm$ std over 3 seeds (dataset seed fixed to 42). Both rows use the same implicit-midpoint conservative core; only the mass model differs. Lower is better.}
    \label{tab:double_pendulum_nonseparable_mass}
    \scriptsize
    \begin{tabular}{@{}lrrrrr@{}}
        \toprule
        \textbf{Model} & \textbf{Params} &
        \textbf{$\mathrm{WrapMSE}_\theta^{\mathrm{roll}}(H{=}100)$} $\downarrow$ &
        \textbf{$\mathrm{EbudRes}^{\mathrm{roll}}(H{=}100)$} $\downarrow$ &
        \textbf{$\mathrm{PassViol}^{\mathrm{roll}}(H{=}100)$} $\downarrow$ \\
        \midrule
        PHAST (KNOWN, constant $M$) & 3{,}589 & $0.367 \pm 0.083$ & $32.71 \pm 4.40$ & $0.408 \pm 0.030$ \\
        PHAST (KNOWN, $M(q)$) & 3{,}589 & \best{0.284 \pm 0.029} & \best{17.16 \pm 1.00} & \best{0.343 \pm 0.026} \\
        \bottomrule
    \end{tabular}
\end{table}

\subsection{PHAST Core Transition Substeps and Learnable Timestep}
\label{app:ablations:substeps_timestep}

\paragraph{What we vary.}
The PHAST core transition can compose $L$ substeps per environment step (Eq.~\eqref{eq:phast_substeps}) and can optionally learn per-substep timesteps $\{\delta t_s\}_{s=1}^L$.
In the main experiments (Sec.~\ref{sec:experiments:setup}) we use $L{=}1$ and initialize the internal timestep to the dataset sampling interval $\dt$; unless otherwise noted, this timestep is learnable.

\paragraph{Why it may help.}
Increasing $L$ primarily reduces discretization error in stiff or coarsely sampled regimes; it does not introduce additional $(V,\Mmat,\Dmat)$ parameters (Table~\ref{tab:qonly_sharing}).
Learning $\delta t_s$ can improve forecasting by time-rescaling, but it can also confound identifiability since time scaling is partially interchangeable with force/damping magnitudes.

\paragraph{Results (coarse-$\dt$; fixed timestep).}
Table~\ref{tab:core_substeps_coarse_dt} summarizes a small multi-seed sweep for PHAST (PARTIAL, q-only) in two coarse-sampling settings.
Increasing substeps from $L{=}1$ to $L{=}4$ yields modest but consistent improvements on $H{=}100$ rollouts, at roughly $4\times$ per-step compute.

\begin{table}[t]
  \centering
  \caption{\textbf{Effect of PHAST core transition substeps at coarse sampling.} Mean $\pm$ std over 5 seeds for PHAST (PARTIAL, q-only) with fixed timesteps $\delta t_s=\dt/L$. Metrics are rollout error at horizon $H{=}100$.}
  \label{tab:core_substeps_coarse_dt}
  \small
  \begin{tabular}{@{}lccrr@{}}
    \toprule
    \textbf{Environment} & \textbf{Metric} & $\dt$ & $L{=}1$ & $L{=}4$ \\
    \midrule
    Windy Cart-Pole & $\mathrm{MixedMSE}^{\mathrm{roll}}(H{=}100)$ & 0.05 & $0.134 \pm 0.020$ & \best{0.126 \pm 0.016} \\
    Double Pendulum & $\mathrm{WrapMSE}_\theta^{\mathrm{roll}}(H{=}100)$ & 0.04 & $2.412 \pm 0.071$ & \best{2.366 \pm 0.088} \\
    \bottomrule
  \end{tabular}
\end{table}

\paragraph{Recommended reporting.}
When sweeping $L$ and learnable $\delta t_s$, we recommend reporting both rollout error and wall-clock per step, and including at least one coarse-$\dt$ setting to stress discretization.

\subsection{Runtime Microbenchmark (PHAST Primitives)}
\label{app:ablations:timing}

To support the efficiency claims in Sec.~\ref{sec:intro}, we microbenchmark the two dominant structured primitives on CPU:
(i) applying Householder damping to a vector ($v \mapsto \Dmat(q)v$, $O(nr)$) and
(ii) computing a velocity via Woodbury ($p \mapsto \Mmat^{-1}p$, $O(nr^2{+}r^3)$).
Figure~\ref{fig:phast_primitives_timing} shows near-linear scaling in $n$ for fixed rank $r{=}2$.

\begin{figure}[t]
    \centering
    \includegraphics[width=0.85\linewidth]{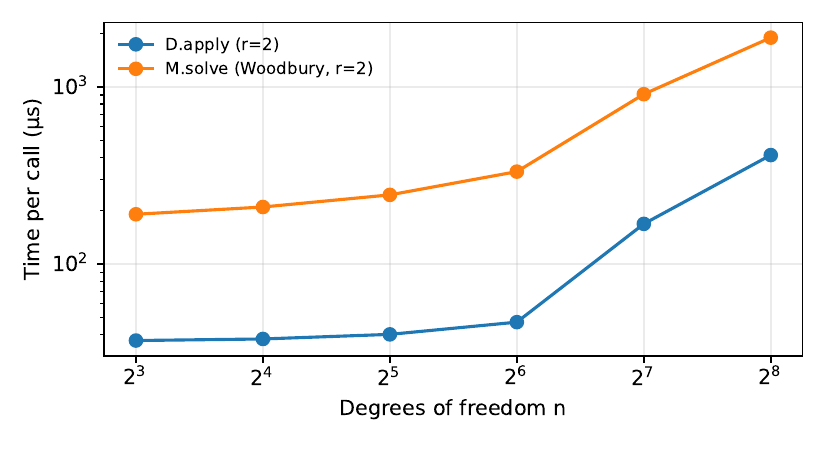}
    \caption{\textbf{CPU microbenchmark of PHAST structured primitives.} Average wall-clock time per call for Householder damping application and Woodbury-based mass solve, sweeping degrees of freedom $n$ with fixed rank $r{=}2$.}
    \label{fig:phast_primitives_timing}
\end{figure}

\subsection{Base Damping Policy (\texorpdfstring{$d_0$}{d0})}

\paragraph{Choice in our experiments.}
In grey-box settings where the base damping scale is known (e.g., Windy Pendulum and Windy Cart-Pole), we fix the isotropic base term $d_0$ in the PARTIAL regime to avoid an identifiability loophole where $d_0$ becomes an error sink when $\sum_{i=1}^{r} \beta_i(q)$ is bounded (Sec.~\ref{sec:methods:householder}).
We also considered learning or bounding $d_0$, but we keep it fixed in the reported PARTIAL windy experiments for interpretability.

\paragraph{Ablation: learnable timestep and $d_0$ policy (Windy Pendulum, q-only).}
Table~\ref{tab:windy_dt_d0_ablation} reports a small sweep that varies whether the internal timestep is fixed ($\dt_{\mathrm{model}}=\dt$) or learnable, and whether the base damping term $d_0$ is fixed or learned, holding the damping-strength cap at $\sum_i\beta_i(q)\le \bbar$ with $\bbar=0.5$.
Learning $\dt_{\mathrm{model}}$ improves long-horizon rollouts and reduces discrete-time passivity violations, while learning $d_0$ has only a minor effect under this bound; we therefore keep $d_0$ fixed for interpretability in the reported PARTIAL windy experiments.

\begin{table}[t]
    \centering
    \caption{\textbf{Learnable timestep and base damping policy (Windy Pendulum, q-only; PARTIAL).} Mean $\pm$ std over 3 seeds on CPU (20 epochs; dataset seed fixed to 42). All runs use Householder damping with a fixed cap $\sum_i\beta_i(q)\le \bbar$ ($\bbar=0.5$) and report horizon-$H{=}100$ diagnostics under the standard q-only rollout protocol (Sec.~\ref{sec:experiments:setup}).}
    \label{tab:windy_dt_d0_ablation}
    \scriptsize
    \setlength{\tabcolsep}{3pt}
    \begin{tabular}{@{}lrrrrr@{}}
        \toprule
        \textbf{Setting} &
        \textbf{$\mathrm{WrapMSE}_\theta^{\mathrm{roll}}(H{=}100)$} $\downarrow$ &
        \textbf{$R^2_D$} $\uparrow$ &
        \textbf{$\mathrm{MAE}_D$} $\downarrow$ &
        \textbf{$\mathrm{EbudRes}(H{=}100)$} $\downarrow$ &
        \textbf{$\mathrm{PassViol}^{\mathrm{roll}}(H{=}100)$} $\downarrow$ \\
        \midrule
        fixed $\dt_{\mathrm{model}}$, fixed $d_0$ &
        $0.175 \pm 0.009$ &
        $0.465 \pm 0.038$ &
        $0.084 \pm 0.003$ &
        $1.34 \pm 0.04$ &
        $0.020 \pm 0.004$ \\
        learnable $\dt_{\mathrm{model}}$, fixed $d_0$ &
        \best{0.153 \pm 0.013} &
        $0.472 \pm 0.040$ &
        $0.085 \pm 0.003$ &
        \best{1.28 \pm 0.04} &
        \best{0.015 \pm 0.004} \\
        fixed $\dt_{\mathrm{model}}$, learned $d_0$ &
        $0.175 \pm 0.009$ &
        $0.476 \pm 0.047$ &
        \best{0.082 \pm 0.004} &
        $1.34 \pm 0.04$ &
        $0.020 \pm 0.004$ \\
        learnable $\dt_{\mathrm{model}}$, learned $d_0$ &
        $0.153 \pm 0.013$ &
        \best{0.478 \pm 0.047} &
        $0.084 \pm 0.004$ &
        $1.28 \pm 0.04$ &
        $0.015 \pm 0.004$ \\
        \bottomrule
    \end{tabular}
\end{table}

\subsection{Baseline Capacity Scaling (Q-Only)}

\paragraph{What we vary.}
To address concerns that large margins might be driven by baseline under-capacity, we run a small capacity scaling check on Windy Pendulum by doubling the baseline hidden dimension (from 64 to 128) while keeping depth fixed (2 layers) and using the same training protocol as Sec.~\ref{sec:experiments:setup}.

\begin{table}[t]
    \centering
    \caption{\textbf{Baseline capacity scaling (Windy Pendulum, q-only).} Mean $\pm$ std of wrapped-angle rollout MSE at horizon $H{=}100$ over 3 seeds (dataset seed fixed to 42) on CPU. Increasing baseline hidden dimension does not close the gap to PHAST (Table~\ref{tab:qonly_rollout_pendulum_h100}).}
    \label{tab:baseline_scaling_windy}
    \scriptsize
    \begin{tabular}{@{}lrr@{}}
        \toprule
        \textbf{Model (scaled)} & \textbf{Params} & \textbf{$\mathrm{WrapMSE}_\theta^{\mathrm{roll}}(H{=}100)$} $\downarrow$ \\
        \midrule
        GRU ($d{=}128$) & 149{,}505 & $2.130 \pm 0.566$ \\
        S5 ($d{=}128$) & 34{,}049 & $0.654 \pm 0.015$ \\
        Transformer ($d{=}128$) & 396{,}929 & $0.868 \pm 0.237$ \\
        \bottomrule
    \end{tabular}
\end{table}

\subsection{Bounding Damping Strength}

\paragraph{What we vary.}
To study the impact of spectral control on identifiability (Sec.~\ref{sec:methods:householder}), we vary a single scalar bound on the total damping strength, $\sum_{i=1}^{r}\beta_i(q)\le \bbar$ (Eq.~\ref{eq:damping_bound}). We also consider per-term bounds $\beta_i(q)\le \beta_{\max}$.

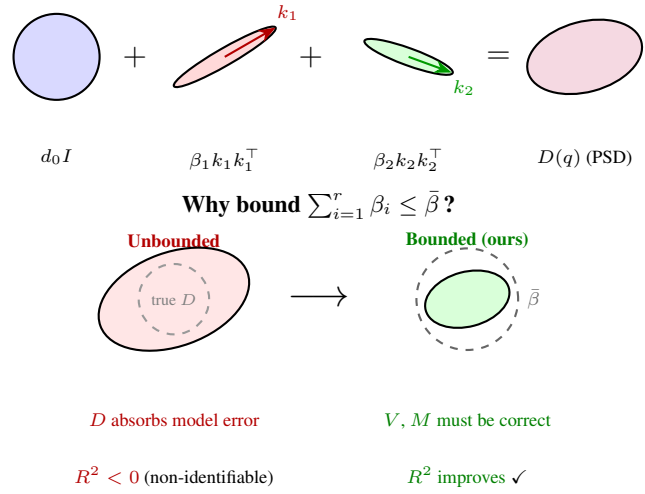
\begin{figure}[t]
  \centering
  \begin{tikzpicture}[scale=0.52]
    % === TOP ROW: Construction ===
    \begin{scope}[shift={(-3.5,0)}]
      \draw[thick, fill=blue!15] (0,0) circle (1.1);
      \node[below=1.1cm, font=\scriptsize] at (0,0) {$d_0 I$};
    \end{scope}
    \node[font=\large] at (-1.5,0) {$+$};
    \begin{scope}[shift={(0.8,0)}]
      \draw[thick, fill=red!15, rotate=30] (0,0) ellipse (1.5 and 0.22);
      \draw[-{Stealth}, thick, red!70!black] (0,0) -- (1.3,0.75);
      \node[red!70!black, above right, font=\scriptsize] at (1.1,0.7) {$k_1$};
      \node[below=1.1cm, font=\scriptsize] at (0,0) {$\beta_1 k_1k_1^\T$};
    \end{scope}
    \node[font=\large] at (3.0,0) {$+$};
    \begin{scope}[shift={(5.5,0)}]
      \draw[thick, fill=green!15, rotate=-20] (0,0) ellipse (1.2 and 0.22);
      \draw[-{Stealth}, thick, green!60!black] (0,0) -- (1.1,-0.4);
      \node[green!60!black, below right, font=\scriptsize] at (0.9,-0.35) {$k_2$};
      \node[below=1.1cm, font=\scriptsize] at (0,0) {$\beta_2 k_2k_2^\T$};
    \end{scope}
    \node[font=\large] at (7.8,0) {$=$};
    \begin{scope}[shift={(10.0,0)}]
      \draw[thick, fill=purple!15, rotate=15] (0,0) ellipse (1.5 and 0.9);
      \node[below=1.1cm, font=\scriptsize] at (0,0) {$\Dmat(q)$ (PSD)};
    \end{scope}

    % === BOTTOM ROW: Bounded vs Unbounded (increased vertical space) ===
    \node[font=\small\bfseries] at (3.2,-3.8) {Why bound $\textstyle\sum_{i=1}^{r} \beta_i \le \bar\beta$\,?};

    % Unbounded (left)
    \begin{scope}[shift={(-0.5,-6.2)}]
      \node[font=\scriptsize\bfseries, text=red!70!black] at (0,1.5) {Unbounded};
      % Large ellipse - D absorbs error
      \draw[thick, fill=red!10, rotate=20] (0,0) ellipse (2.0 and 1.2);
      \draw[dashed, thick, black!40] (0,0) circle (0.9);
      \node[font=\tiny, text=black!50] at (0,0) {true $D$};
      \node[below=1.4cm, font=\scriptsize, align=center, text=red!70!black] at (0,0) {$D$ absorbs model error};
    \node[below=2.1cm, font=\scriptsize, align=center] at (0,0) {\textcolor{red!70!black}{$R^2 < 0$} (non-identifiable)};
    \end{scope}

    % Arrow
    \node[font=\Large] at (3.2,-6.2) {$\longrightarrow$};

    % Bounded (right)
    \begin{scope}[shift={(7.0,-6.2)}]
      \node[font=\scriptsize\bfseries, text=green!50!black] at (0,1.5) {Bounded (ours)};
      % Constrained ellipse
      \draw[thick, fill=green!15, rotate=15] (0,0) ellipse (1.1 and 0.7);
      % Bound circle
      \draw[dashed, thick, black!60] (0,0) circle (1.3);
      \node[font=\scriptsize, text=black!60] at (1.7,0) {$\bar\beta$};
      \node[below=1.4cm, font=\scriptsize, align=center, text=green!50!black] at (0,0) {$V$, $M$ must be correct};
      \node[below=2.1cm, font=\scriptsize, align=center] at (0,0) {\textcolor{green!50!black}{$R^2$ improves} \checkmark};
    \end{scope}
  \end{tikzpicture}
  \caption{\textbf{Low-rank PSD damping with spectral bound.}
  \emph{Top:} $D(q)=d_0 I+\sum_{i=1}^{r} \beta_i(q)\,k_i(q)k_i(q)^\T$ is PSD by construction.
  \emph{Bottom:} Bounding $\sum_{i=1}^{r}\beta_i(q)\le \bar\beta$ certifies $\lambda_{\max}(D(q))\le d_0+\bar\beta$ (Eq.~\eqref{eq:damping_bound}).}
  \label{fig:householder_geom}
\end{figure}

\ifdefined\isarxiv
% Included in the arXiv main body for a tighter narrative (see Sec.~\ref{sec:experiments:results}).
\else
% tables_partial_damping_cap_windy.tex — PARTIAL damping cap ablation (Windy Pendulum, q-only)

\begin{table}[t]
    \centering
    \caption{\textbf{PARTIAL regime: effect of bounding total damping strength (Windy Pendulum, q-only).} Mean $\pm$ std over 5 seeds on CPU (same data protocol as Sec.~\ref{sec:experiments:setup}). Without the bound, damping becomes an ``error sink'' and is non-identifiable ($R^2_D\ll 0$) even when rollouts are accurate; bounding $\sum_{i=1}^{r}\beta_i(q)$ recovers a physically meaningful damping field.}
    \label{tab:partial_damping_cap_windy}
    \scriptsize
    \setlength{\tabcolsep}{3pt}
    \begin{tabular}{@{}lrrrr@{}}
        \toprule
        \textbf{Damping cap} &
        \textbf{WrapMSE$_\theta$(100)} $\downarrow$ &
        \textbf{$R^2_D$} $\uparrow$ &
        \textbf{MAE$_D$} $\downarrow$ &
        \textbf{EbudRes(100)} $\downarrow$ \\
        \midrule
        Unbounded &
        \best{0.059 \pm 0.007} &
        $-90.15 \pm 6.14$ &
        $1.326 \pm 0.044$ &
        \best{1.50 \pm 0.03} \\
        $\sum_i\beta_i(q)\le 0.5$ &
        $0.071 \pm 0.004$ &
        \best{0.703 \pm 0.072} &
        \best{0.061 \pm 0.010} &
        $1.58 \pm 0.04$ \\
        \bottomrule
    \end{tabular}
\end{table}

\fi

\paragraph{Ablation: cap value sensitivity (Windy Pendulum, q-only; PARTIAL).}
Table~\ref{tab:windy_cap_sensitivity_quick} sweeps the total-strength cap $\sum_i\beta_i(q)\le \bbar$ at a reduced training budget with fixed $\dt_{\mathrm{model}}=\dt$ and fixed $d_0$.
We observe a clear identifiability sweet spot around $\bbar=0.5$: loosening the cap (unbounded or $\bbar=1.0$) makes damping recovery fail ($R^2_D \ll 0$), while tightening to $\bbar=0.25$ slightly increases rollout error but improves energy-residual and passivity diagnostics.

\begin{table}[t]
    \centering
    \caption{\textbf{Cap sensitivity sweep (Windy Pendulum, q-only; PARTIAL).} Mean $\pm$ std over 3 seeds on CPU (20 epochs; dataset seed fixed to 42). All runs use Householder damping with fixed $\dt=0.05$, fixed $\dt_{\mathrm{model}}=\dt$, fixed $d_0$, and vary only the total-strength cap $\sum_i\beta_i(q)\le \bbar$. We report horizon-$H{=}100$ diagnostics under the standard q-only rollout protocol (Sec.~\ref{sec:experiments:setup}).}
    \label{tab:windy_cap_sensitivity_quick}
    \scriptsize
    \setlength{\tabcolsep}{3pt}
    \begin{tabular}{@{}lrrrrr@{}}
        \toprule
        \textbf{Cap on $\sum_{i=1}^{r}\beta_i(q)$} &
        \textbf{$\mathrm{WrapMSE}_\theta^{\mathrm{roll}}(H{=}100)$} $\downarrow$ &
        \textbf{$R^2_D$} $\uparrow$ &
        \textbf{$\mathrm{MAE}_D$} $\downarrow$ &
        \textbf{$\mathrm{EbudRes}(H{=}100)$} $\downarrow$ &
        \textbf{$\mathrm{PassViol}^{\mathrm{roll}}(H{=}100)$} $\downarrow$ \\
        \midrule
        Unbounded &
        \best{0.162 \pm 0.015} &
        $-40.940 \pm 8.630$ &
        $0.897 \pm 0.089$ &
        $1.43 \pm 0.04$ &
        $0.022 \pm 0.006$ \\
        $\sum_{i=1}^{r}\beta_i(q)\le 0.25$ &
        $0.180 \pm 0.012$ &
        $-0.080 \pm 0.081$ &
        $0.105 \pm 0.004$ &
        \best{1.32 \pm 0.03} &
        \best{0.019 \pm 0.004} \\
        $\sum_{i=1}^{r}\beta_i(q)\le 0.5$ &
        $0.175 \pm 0.009$ &
        \best{0.465 \pm 0.038} &
        \best{0.084 \pm 0.003} &
        $1.34 \pm 0.04$ &
        $0.020 \pm 0.004$ \\
        $\sum_{i=1}^{r}\beta_i(q)\le 1.0$ &
        $0.175 \pm 0.014$ &
        $-1.506 \pm 0.582$ &
        $0.211 \pm 0.026$ &
        $1.38 \pm 0.05$ &
        $0.022 \pm 0.004$ \\
        \bottomrule
    \end{tabular}
\end{table}

\begin{table}[t]
    \centering
    \caption{\textbf{UNKNOWN regime: effect of bounding total damping strength (Windy Pendulum, q-only).} Mean $\pm$ std over 5 seeds on CPU (same data protocol as Sec.~\ref{sec:experiments:setup}). Bounding $\sum_{i=1}^{r}\beta_i(q)$ substantially improves damping recovery (identifiability) but can degrade long-horizon open-loop forecasting, illustrating a forecasting--identifiability trade-off in the UNKNOWN regime.}
    \label{tab:unknown_damping_cap_windy}
    \scriptsize
    \setlength{\tabcolsep}{3pt}
    \begin{tabular}{@{}lrrrr@{}}
        \toprule
        \textbf{Cap on $\sum_{i=1}^{r}\beta_i(q)$} &
        \textbf{$\mathrm{WrapMSE}_\theta^{\mathrm{roll}}(H{=}100)$} $\downarrow$ &
        \textbf{$R^2_D$} $\uparrow$ &
        \textbf{$\mathrm{MAE}_D$} $\downarrow$ &
        \textbf{$\mathrm{EbudRes}(H{=}100)$} $\downarrow$ \\
        \midrule
        Unbounded &
        \best{0.298 \pm 0.054} &
        $-96.546 \pm 17.774$ &
        $1.343 \pm 0.129$ &
        \best{2.581 \pm 0.214} \\
        $\sum_{i=1}^{r}\beta_i(q)\le 0.5$ &
        $0.416 \pm 0.071$ &
        \best{-0.234 \pm 0.268} &
        \best{0.134 \pm 0.019} &
        $2.713 \pm 0.207$ \\
        $\sum_{i=1}^{r}\beta_i(q)\le 1.0$ &
        $0.396 \pm 0.076$ &
        $-6.457 \pm 1.376$ &
        $0.363 \pm 0.034$ &
        $2.666 \pm 0.214$ \\
        \bottomrule
    \end{tabular}
\end{table}

\subsection{Gauge Freedom and Parameter Identification}
\label{app:ablations:gauge_freedom}

\paragraph{Why identifiability can fail.}
Even in noise-free settings, recovering physical parameters from trajectories can be ill-posed: the map from a mechanical model $(M(\cdot),V(\cdot))$ to observed trajectories need not be injective.
Two common sources of non-uniqueness are:
(i) \emph{Lagrangian gauge}---adding a total time derivative to the Lagrangian does not change the Euler--Lagrange equations,
\begin{equation}
  L'(q,\dot q,t)=L(q,\dot q,t)+\tfrac{d}{dt}F(q,t),
\end{equation}
and (ii) \emph{coordinate/canonical freedom}---changes of variables can alter the representation of $(M,V)$ (and the meaning of $p$) without changing the underlying configuration-space trajectories.
In data-driven learning, this appears as \emph{parameter trade-offs}: distinct $(M,V)$ (and, in dissipative settings, also $D$) can yield similarly accurate rollouts.

\paragraph{Connection to PHAST regimes.}
PHAST's knowledge regimes can be viewed as \emph{gauge-fixing choices} that trade flexibility for identifiability:
\textbf{KNOWN} anchors $(V,M)$ and learns only $D(q)$; \textbf{PARTIAL} anchors the \emph{form} of $V$ and uses calibrated damping bounds; \textbf{UNKNOWN} is intentionally flexible but can be non-identifiable (Sec.~\ref{sec:experiments:eval}).
The damping-strength bound in Eq.~\eqref{eq:damping_bound} plays an analogous role to gauge fixing: it prevents dissipation from absorbing model mismatch and improves physical recovery (Tables~\ref{tab:partial_damping_cap_windy}--\ref{tab:unknown_damping_cap_windy}).

\paragraph{Illustrative controlled study (double pendulum).}
To make the gauge-freedom issue concrete, we include a small controlled experiment on a conservative double pendulum.
We compare learning $M(q)$ and/or $V(q)$ under different anchors using an Euler--Lagrange residual loss:
\begin{equation}
  r(q,\dot q,\ddot q) := M(q)\ddot q + C_M(q,\dot q)\dot q + \nabla V(q),
  \qquad
  \mathcal{L}_{\mathrm{res}}=\|r(q,\dot q,\ddot q)\|_2^2,
\end{equation}
where $C_M(q,\dot q)\dot q$ is the Coriolis/centrifugal term induced by $M(q)$.
In components,
\begin{equation}
  \big[C_M(q,\dot q)\dot q\big]_i
  =
  \sum_{j,k=1}^{n}\Gamma_{ijk}(q)\dot q_j\dot q_k,
  \qquad
  \Gamma_{ijk}(q)=\tfrac{1}{2}\Big(\partial_{q_k}M_{ij}+\partial_{q_j}M_{ik}-\partial_{q_i}M_{jk}\Big).
\end{equation}
The residual loss directly constrains $M$ when $V$ is anchored.
Empirically, when both $M$ and $V$ are fully free (neural $V$), parameter recovery fails despite low trajectory MSE, consistent with gauge freedom.
Anchoring $V$ (exactly, or via a structured form with learnable parameters) makes $M$ identifiable and yields accurate recovery.
We report results for two mass parameterizations: (i) a full SPD (Cholesky) model used in earlier canonicalization studies, and (ii) the current PHAST UNKNOWN-style mass \emph{NeuralMass} (diagonal $+$ low-rank outer products with Woodbury solves).

\paragraph{Double pendulum physics.}
The true configuration-dependent mass and potential for the double pendulum (unit masses and lengths) are:
\begin{equation}
  M(q) = \begin{pmatrix} 2 & \cos(\theta_1{-}\theta_2) \\ \cos(\theta_1{-}\theta_2) & 1 \end{pmatrix}, \qquad
  V(q) = -2g\cos\theta_1 - g\cos\theta_2,
\end{equation}
where $g=9.81$.
The Euler--Lagrange residual loss enforces $M(q)\ddot q + C_M(q,\dot q)\dot q + \nabla V(q) = 0$ via the Coriolis term $C_M$.

\paragraph{Structured potential (PARTIAL analogue).}
``Structured $V$'' means the functional \emph{form} is fixed to the true physics, but coefficients are learned:
\begin{equation}
  V(q; g_1, g_2) = -g_1\cos\theta_1 - g_2\cos\theta_2,
\end{equation}
where $g_1, g_2$ are trainable parameters (true values: $g_1{=}2g{=}19.62$, $g_2{=}g{=}9.81$).
This is analogous to PHAST's PARTIAL regime where the potential template is known but parameters are learned.

\paragraph{NeuralMass parameterization.}
The NeuralMass module uses the same Householder-style low-rank form as PHAST damping (Sec.~\ref{sec:methods:householder}):
\begin{equation}
  \Mmat(q) = \mathrm{diag}(d(q)) + \sum_{i=1}^{r} \alpha_i(q)\,k_i(q)k_i(q)^\T,
  \quad d_j > 0,\ \alpha_i \ge 0,\ \|k_i\|=1,
\end{equation}
where positivity is enforced via softplus and directions are normalized.
Inversion uses the Woodbury identity in $O(nr^2{+}r^3)$.
The SPD (Cholesky) baseline parameterizes $M = LL^\T$ with lower-triangular $L$ having positive diagonal.

\begin{table}[t]
  \centering
  \caption{\textbf{Gauge-freedom ablation (conservative double pendulum).} Mean over 2 runs (50 epochs). Lower is better. ``Structured $V$'' means the functional form is fixed but coefficients are learned (analogous to PHAST PARTIAL). Errors are averaged over held-out configurations: $\mathrm{FrobErr}_M:=\mathbb{E}_q\|M(q)-\hat M(q)\|_F$ and $\mathrm{MAE}_{\nabla V}:=\frac{1}{n}\mathbb{E}_q\|\nabla V(q)-\nabla\hat V(q)\|_1$.}
  \label{tab:gauge_freedom_double_pendulum}
  \scriptsize
  \setlength{\tabcolsep}{3pt}
  \begin{tabular}{@{}p{0.46\linewidth}p{0.14\linewidth}ccc@{}}
    \toprule
    \textbf{Setup} & \textbf{Mass} & \textbf{$\mathrm{FrobErr}_M$} $\downarrow$ & \textbf{$\mathrm{MAE}_{\nabla V}$} $\downarrow$ & \textbf{Conclusion} \\
    \midrule
    Learn $M$ (known $V$) + $\mathcal{L}_{\mathrm{res}}$ & SPD & $0.011$ & --- & Identifiable \\
    Learn $M$ (known $V$) + $\mathcal{L}_{\mathrm{res}}$ & NeuralMass & $0.306$ & --- & Identifiable (high var.) \\
    Learn $V$ (known $M$) + $\mathcal{L}_{\mathrm{res}}$ & Known $M$ & --- & $0.094$ & Identifiable \\
    Learn both ($M$ + neural $V$) + $\mathcal{L}_{\mathrm{res}}$ & SPD & $2.37$ & $6.93$ & Non-identifiable \\
    Learn both ($M$ + neural $V$) + $\mathcal{L}_{\mathrm{res}}$ & NeuralMass & $2.23$ & $6.73$ & Non-identifiable \\
    Learn $M$ + structured $V$ + $\mathcal{L}_{\mathrm{res}}$ & SPD & $0.035$ & $0.082$ & Identifiable \\
    Learn $M$ + structured $V$ + $\mathcal{L}_{\mathrm{res}}$ & NeuralMass & $0.089$ & $0.066$ & Identifiable \\
    Learn both ($M$ + sparse-basis $V$) + $\mathcal{L}_{\mathrm{res}}$ & NeuralMass & $1.79$ & $12.05$ & Non-identifiable \\
    \bottomrule
  \end{tabular}
\end{table}

% ===========================================================================
% G. Notation Reference
% ===========================================================================
\section{Notation Reference}
\label{app:notation}

\paragraph{Configuration space.}
We write $q\in\mathcal{Q}$ for the configuration, where $\mathcal{Q}$ is typically a product manifold
$\mathcal{Q}=\mathbb{R}^{n_e}\times(\mathbb{S}^1)^{n_a}$ with $n_e+n_a=n$.
In implementation, $q$ is stored in $\mathbb{R}^n$ and angular errors use $\mathrm{wrap}(\cdot)$ to respect periodicity (Appendix~\ref{app:metrics}).
We use $x$ to denote the phase-space state $x=(q,p)$, but in the Cart-Pole environment $q=(x,\theta)$ also uses $x$ for the cart position coordinate; context disambiguates.

\begin{table}[t]
    \centering
    \caption{Index ranges.}
    \label{tab:index_ranges}
    \small
    \begin{tabular}{@{}lll@{}}
        \toprule
        \textbf{Index} & \textbf{Range} & \textbf{Meaning} \\
        \midrule
        $b$ & $1,\ldots,B$ & Trajectory index (batching) \\
        $t$ & $0,\ldots,T-1$ & Time index (transitions use $t{=}0,\ldots,T{-}2$) \\
        $h$ & $0,\ldots,H-1$ & Rollout step index (horizon $H$) \\
        $s$ & $1,\ldots,L$ & PHAST substep index \\
        $i$ & $1,\ldots,r$ & Rank-1 term index in low-rank updates \\
        $j$ & $1,\ldots,n$ & Coordinate index \\
        \bottomrule
    \end{tabular}
\end{table}

\begin{table}[h]
    \centering
    \caption{\textbf{Symbol reference.}
    Symbols are grouped by the section in which they are introduced.
    Gray rows mark group headers indicating the originating section.}
    \label{tab:notation}
    \scriptsize
    \setlength{\tabcolsep}{3pt}
    \begin{tabular}{@{}p{0.16\linewidth}p{0.20\linewidth}p{0.48\linewidth}@{}}
        \toprule
        \textbf{Symbol} & \textbf{Space / shape} & \textbf{Meaning} \\

        % ── Problem setting & knowledge regimes (Sec. 3.1) ──────────────
        \midrule
        \multicolumn{3}{@{}l}{\cellcolor{black!8}\textbf{Problem setting \& knowledge regimes}
          \hfill\textit{Sec.~\ref{sec:intro:regimes}}} \\
        \addlinespace[2pt]
        $n$ & $\mathbb{N}$ & Degrees of freedom \\
        $\mathcal{Q}$ & manifold & Configuration space (e.g., $\mathbb{R}^{n_e}\times(\mathbb{S}^1)^{n_a}$) \\
        $q$ & $\mathcal{Q}$ & Generalized configuration (forecast target) \\
        $p$ & $\mathbb{R}^n$ & Conjugate momentum (latent in q-only setting) \\
        $x=(q,p)$ & $\mathcal{Q}\times\mathbb{R}^n$ & Phase-space state \\
        $y_t$ & $\mathcal{Q}$ & Observation (q-only: $y_t=q_t$) \\

        % ── Port-Hamiltonian structure (Sec. 3.2) ──────────────────────
        \midrule
        \multicolumn{3}{@{}l}{\cellcolor{black!8}\textbf{Port-Hamiltonian dynamics}
          \hfill\textit{Sec.~\ref{sec:methods:ph}}} \\
        \addlinespace[2pt]
        $\Ham(q,p)$ & $\mathcal{Q}\times\allowbreak\mathbb{R}^n\to\allowbreak\mathbb{R}$ & Hamiltonian (total energy) \\
        $V(q)$ & $\mathcal{Q}\to\mathbb{R}$ & Potential energy \\
        $\Mmat(q)$ & $\mathcal{Q}\to\mathbb{R}^{n\times n}$ & Mass/inertia matrix (SPD; general, configuration-dependent) \\
        $\Mmat$ & $\mathbb{R}^{n\times n}$ & Constant-mass approximation $\Mmat(q)\approx\Mmat$ (used in main experiments) \\
        $v=\Mmat(q)^{-1}p$ & $\mathbb{R}^n$ & Generalized velocity \\
        $\Dmat(q)$ & $\mathcal{Q}\to\mathbb{R}^{n\times n}$ & Damping matrix (PSD) \\
        $\Jmat$ & $\mathbb{R}^{2n\times 2n}$ & Interconnection matrix (skew-symmetric) \\
        $\Rmat$ & $\mathbb{R}^{2n\times 2n}$ & Dissipation matrix (PSD; block-diagonal with $\Dmat$) \\
        $I$ & $\mathbb{R}^{n\times n}$ & Identity matrix (size inferred from context) \\
        $m$ & $\mathbb{N}$ & Input/port dimension (when control is present) \\
        $u$ & $\mathbb{R}^m$ & Port input (control) \\
        $G$ & $\mathbb{R}^{2n\times m}$ & Input matrix (port mapping) \\
        $y^{\mathrm{port}}$ & $\mathbb{R}^m$ & Port output $y^{\mathrm{port}}=G^\T\nabla\Ham(x)$ \\

        % ── Householder-style parameterizations (Sec. 3.x) ─────────────
        \midrule
        \multicolumn{3}{@{}l}{\cellcolor{black!8}\textbf{Low-rank parameterizations}
          \hfill\textit{Sec.~\ref{sec:methods:householder}}} \\
        \addlinespace[2pt]
        $r$ & $\mathbb{N}$ & Rank of low-rank expansion \\
        $d_0$ & $\mathbb{R}_{\ge 0}$ & Isotropic baseline damping \\
        $\beta_i(q)$ & $\mathcal{Q}\to\mathbb{R}_{\ge 0}$ & Strength of rank-$1$ damping term \\
        $k_i(q)$ & $\mathcal{Q}\to\mathbb{S}^{n-1}$ & Direction of rank-$1$ damping term ($\|k_i(q)\|_2=1$) \\
        $\bbar$ & $\mathbb{R}_{\ge 0}$ & Damping strength bound ($\sum_i\beta_i\le\bbar$) \\
        $\Lambda$ & $\mathbb{R}^{n\times n}$ & Diagonal base matrix $\Lambda=\mathrm{diag}(d)$ (Woodbury) \\
        $U$ & $\mathbb{R}^{n\times r}$ & Low-rank factor matrix (Woodbury: $\Mmat=\Lambda+UU^\T$) \\
        $I_r$ & $\mathbb{R}^{r\times r}$ & Identity matrix in rank-$r$ Woodbury updates \\

        % ── Integration (Sec. 3.x) ────────────────────────────────────
        \midrule
        \multicolumn{3}{@{}l}{\cellcolor{black!8}\textbf{Structure-preserving integration}
          \hfill\textit{Sec.~\ref{sec:methods:integration}}} \\
        \addlinespace[2pt]
        $\dt$ & $\mathbb{R}_{>0}$ & Time step \\
        $\Phi_{\dt}$ & $\mathcal{Q}\times\allowbreak\mathbb{R}^n\to\allowbreak\mathcal{Q}\times\allowbreak\mathbb{R}^n$ & Full integrator map (Strang splitting) \\
        $\Phi_H^{\dt}$ & $\mathcal{Q}\times\allowbreak\mathbb{R}^n\to\allowbreak\mathcal{Q}\times\allowbreak\mathbb{R}^n$ & Conservative step (symplectic core) \\
        $\Phi_D^{\dt/2}$ & $\mathcal{Q}\times\allowbreak\mathbb{R}^n\to\allowbreak\mathcal{Q}\times\allowbreak\mathbb{R}^n$ & Dissipation half-step \\
        $L$ & $\mathbb{N}$ & PHAST core transition substeps per step \\
        $\delta t_s$ & $\mathbb{R}_{>0}$ & Per-substep timestep (optional; $s=1,\ldots,L$) \\

        % ── Q-only observer pipeline ──────────────────────────────────
        \midrule
        \multicolumn{3}{@{}l}{\cellcolor{black!8}\textbf{Q-only observer pipeline}
          \hfill\textit{Eqs.~\eqref{eq:qonly_fd}--\eqref{eq:qonly_takeover}}} \\
        \addlinespace[2pt]
        $\dot q_t^{\mathrm{fd}}$ & $\mathbb{R}^n$ & Finite-difference velocity estimate from q-only data \\
        $o_\phi$ & causal map & Velocity observer (FD+TCN): $(q_{0:t},\dot q^{\mathrm{fd}}_{0:t})\mapsto \delta_t$ \\
        $\delta_t$ & $\mathbb{R}^n$ & Observer correction to finite-difference velocity \\
        $\hat{\dot q}_t$ & $\mathbb{R}^n$ & Observer-estimated velocity \\
        $C_\psi$ & $\mathcal{Q}\times\allowbreak\mathbb{R}^n\to\allowbreak\mathcal{Q}\times\allowbreak\mathbb{R}^n$ & Canonicalizer: $(q,\hat{\dot q})\mapsto(q,\hat p)$ \\
        $\hat p_t$ & $\mathbb{R}^n$ & Canonicalizer output (identity: $\hat p_t=\hat{\dot q}_t$) \\
        $\Pi_q(q,p)$ & $\mathcal{Q}\times\allowbreak\mathbb{R}^n\to\allowbreak\mathcal{Q}$ & Position projection $\Pi_q(q,p)=q$ \\

        % ── Experiments & evaluation ──────────────────────────────────
        \midrule
        \multicolumn{3}{@{}l}{\cellcolor{black!8}\textbf{Experiments \& evaluation}
          \hfill\textit{Sec.~\ref{sec:experiments:setup}--\ref{sec:experiments:eval}}} \\
        \addlinespace[2pt]
        $T$ & $\mathbb{N}$ & Trajectory length (sequence length) \\
        $K$ & $\mathbb{N}$ & Burn-in context length (q-only rollouts) \\
        $H$ & $\mathbb{N}$ & Rollout horizon \\
        $B$ & $\mathbb{N}$ & Number of trajectories used in metric averaging \\
        $D_{\mathrm{env}}(q)$ & $\mathcal{Q}\to\mathbb{R}_{\ge 0}$ & Simulator (ground-truth) damping coefficient \\
        $\mathrm{wrap}(\cdot)$ & $\mathbb{R}\to[-\pi,\pi]$ & Wrap angular differences mod $2\pi$ \\

        \bottomrule
    \end{tabular}
\end{table}

% ===========================================================================
% H. Casimir Control (Separate Task)
% ===========================================================================
\ifdefined\isarxiv
% Promoted to the main paper for the arXiv version (see arxiv_extended.tex).
\else

\fi

\end{document}